\documentclass[journal]{IEEEtran}
\ifCLASSINFOpdf
\else
\fi
\usepackage{marvosym}
\usepackage{amssymb}
\usepackage{microtype}
\usepackage{graphicx}
\usepackage{booktabs}
\usepackage{booktabs} 
\usepackage{subcaption}
\usepackage{caption}
\usepackage{amsmath}
\usepackage{amssymb}
\usepackage{amsthm}
\usepackage{varwidth}
\usepackage{multirow}
\usepackage{upgreek}
\usepackage{xcolor}
\usepackage[normalem]{ulem}
\usepackage{enumitem}
\usepackage{wrapfig,lipsum,booktabs}
\usepackage{hyperref}
\usepackage{algorithmic,algorithm}
\usepackage{colortbl}
\usepackage{sidecap}

\newtheorem{definition}{Definition}
\newtheorem{remark}{Remark}

\newtheorem{proposition}{Proposition}
\newtheorem{corollary}{Corollary}[proposition]
\def\revision{\textcolor{black}}
\definecolor{mintbg}{rgb}{.63,.79,.95}
\colorlet{lightmintbg}{mintbg!40}

\begin{document}
%
\title{Differential-Critic GAN: Generating What You Want by a Cue of Preferences}
%
%
%

\author{Yinghua Yao,
        Yuangang Pan\textsuperscript{\Letter},
        Ivor W. Tsang,~\IEEEmembership{Fellow,~IEEE,}
        and~Xin Yao,~\IEEEmembership{Fellow,~IEEE}
\thanks{Yinghua~Yao is with Guangdong Key Laboratory of Brain-inspired Intelligent Computation, Department of Computer Science and Engineering, Southern University of Science and Technology, Shenzhen, China, and also with Australian Artificial Intelligence Institute, University of Technology Sydney, Australia. Yuangang~Pan and Ivor~W.~Tsang are with A*STAR Center for Frontier AI Research, Singapore. Ivor~W.~Tsang is also with Australian Artificial Intelligence Institute, University of Technology Sydney, Australia. Xin~Yao is with the Research Institute of Trustworthy Autonomous Systems (RITAS) and Guangdong Key Laboratory of Brain-inspired Intelligent Computation, Department of Computer Science and Engineering, Southern University of Science and Technology, Shenzhen, China, and also with School of Computer Science, University of Birmingham, UK.  
Email: yinghua.yao@student.uts.edu.au,
yuangang.pan@gmail.com,
ivor.tsang@gmail.com, 
xiny@sustech.edu.cn. (Corresponding author: Yuangang Pan)}
}

%
%

\markboth{Journal of \LaTeX\ Class Files,~Vol.~14, No.~8, August~2015}%
{Shell \MakeLowercase{\textit{et al.}}: Bare Demo of IEEEtran.cls for IEEE Journals}
%



\maketitle

\begin{abstract}
This paper proposes Differential-Critic Generative Adversarial Network (DiCGAN) to learn the distribution of user-desired data when only partial instead of the entire dataset possesses the desired property. DiCGAN generates desired data that meets the user's expectations and can assist in designing biological products with desired properties.
Existing approaches select the desired samples first and train regular GANs on the selected samples to derive the user-desired data distribution. However, the selection of the desired data relies on global knowledge and supervision over the entire dataset.
DiCGAN introduces a differential critic that learns from pairwise preferences, which are local knowledge and can be defined on a part of training data.
The critic is built by defining an additional ranking loss over the Wasserstein GAN’s critic. It endows the difference of critic values between each pair of samples with the user preference and guides the generation of the desired data instead of the whole data. 
For a more efficient solution to ensure data quality, we further reformulate DiCGAN as a constrained optimization problem, based on which we theoretically prove the convergence of our DiCGAN.
Extensive experiments on a diverse set of datasets with various applications demonstrate that our DiCGAN achieves state-of-the-art performance in learning the user-desired data distributions, especially in the cases of insufficient desired data and limited supervision. The code is available in \url{https://github.com/EvaFlower/Differential-Critic-GAN}.
\end{abstract}


%
\IEEEpeerreviewmaketitle

\section{Introduction}
\IEEEPARstart{L}{earning} a good generative model for high-dimensional natural signals, such as images~\cite{zhu2017unpaired}, video~\cite{9210866}, speech~\cite{9219228} and text~\cite{9146686}, has long been one of the key milestones of machine learning. Powered by the learning capabilities of deep neural networks, Generative Adversarial Networks (GANs)~\cite{Goodfellow2014} have brought the field closer to attaining this goal. 
\revision{Currently, GANs are applied in a setting where the whole training dataset is of user interest (Fig.~\ref{fig:tsne_gan}). However, regular GANs no longer meet our requirement when only partial instead of the entire training dataset possesses the desired property~\cite{killoran2017generating}. 
Studying GANs under this setting can be useful in many real-world applications. 
One application can be to optimize generated biological data for desired properties, which can automate the process of designing DNA sequences, proteins and additional macromolecules for usage in medicine and manufacturing~\cite{gupta2019feedback}. Another can be to generate images that meet the user’s interest for image search~\cite{49290}.}

Existing methods~\cite{pmlr-v70-arjovsky17a,mirza2014conditional,gupta2019feedback} derive a user-desired data distribution by labeling the whole training dataset with a universal criterion. Based on the criterion, each training sample is annotated as ``desired'' or ``undesired''. Then GANs only learn to match the distribution of desired data. However, the requirement for the universal criterion is harsh since it requires global knowledge over the dataset, i.e., ``what desired data is" derived from the whole data, which is expensive and is not available in many real-world applications~\cite{christiano2017deep,gupta2019feedback}. For example, when it is asked to train a robot to clean a table. It’s not clear how to construct a suitable reward function (global knowledge, assessing all behaviors with a universal criterion).
Secondly, labeling all of the training data causes high labor costs. 
Suppose that the user is interested in the generation of small digits on MNIST. The global knowledge is about the global ranking list of digits, i.e., $0 \succ 1 \ldots \succ 9$, where the notation $\succ$ denotes the left-hand side is preferred over the right-hand side. 
Traversing the whole dataset, the user annotates zero digits as ``desired'' and other digits as ``undesired''.

\begin{figure}[!tb]
    \centering
    \begin{subfigure}[t]{.32\linewidth}
  \centering
    \includegraphics[width=\linewidth]{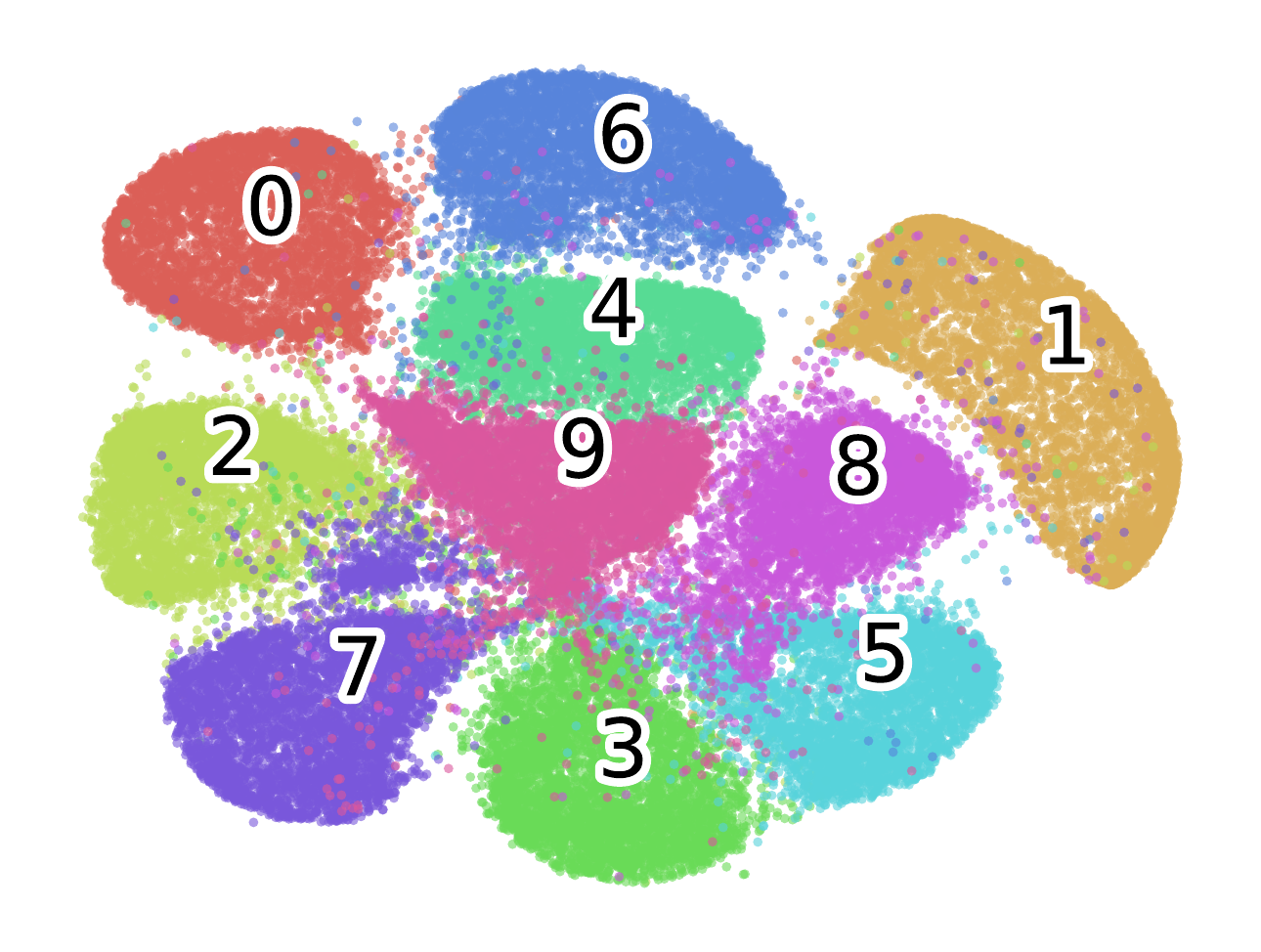}
    \caption{Training data\label{fig:tsne_data}}
  \end{subfigure}
    \begin{subfigure}[t]{.32\linewidth}
  \centering
    \includegraphics[width=\linewidth]{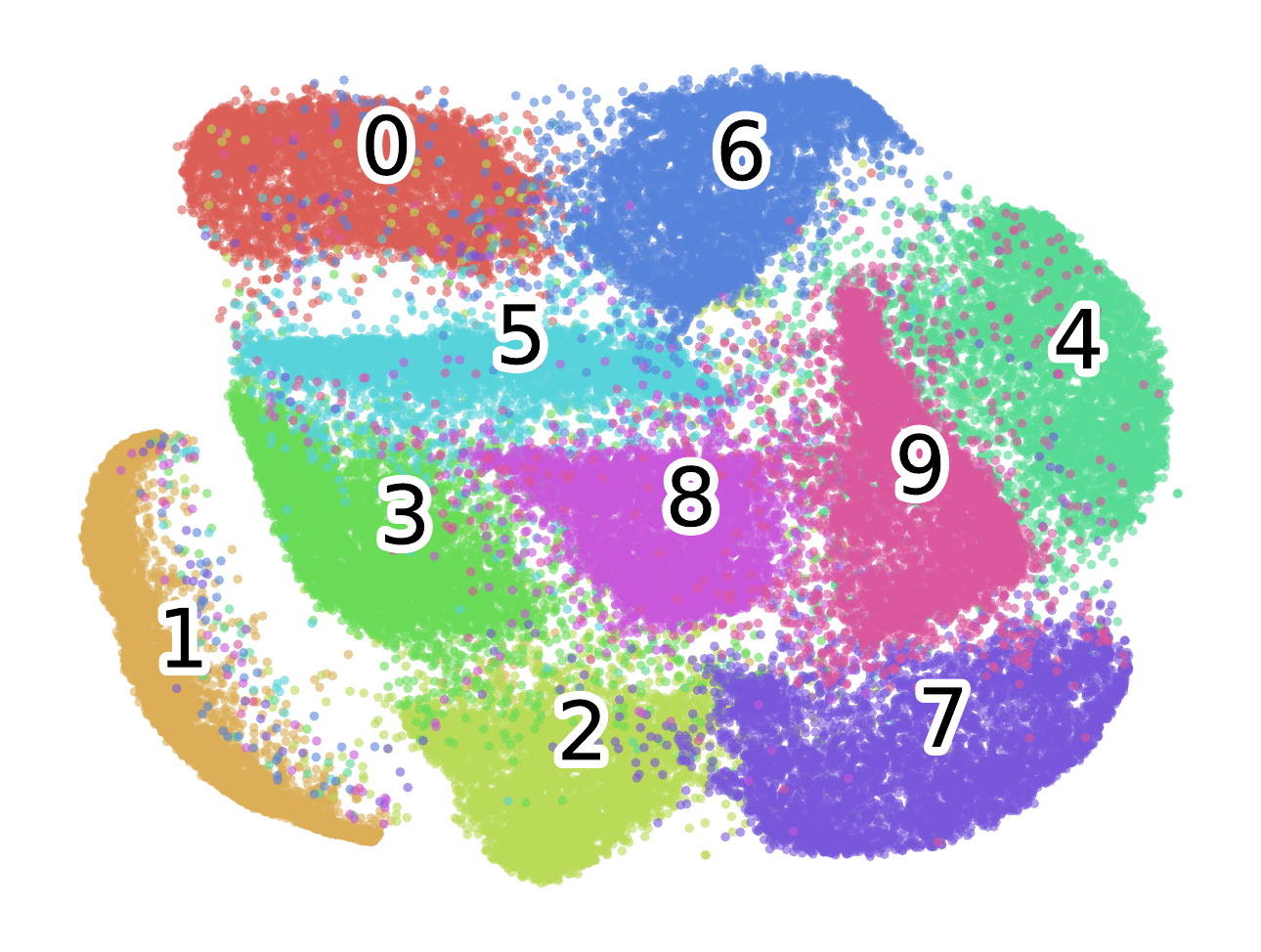}
    \caption{GAN\label{fig:tsne_gan}}
  \end{subfigure}
  \begin{subfigure}[t]{.32\linewidth}
  \centering
    \includegraphics[width=\linewidth]{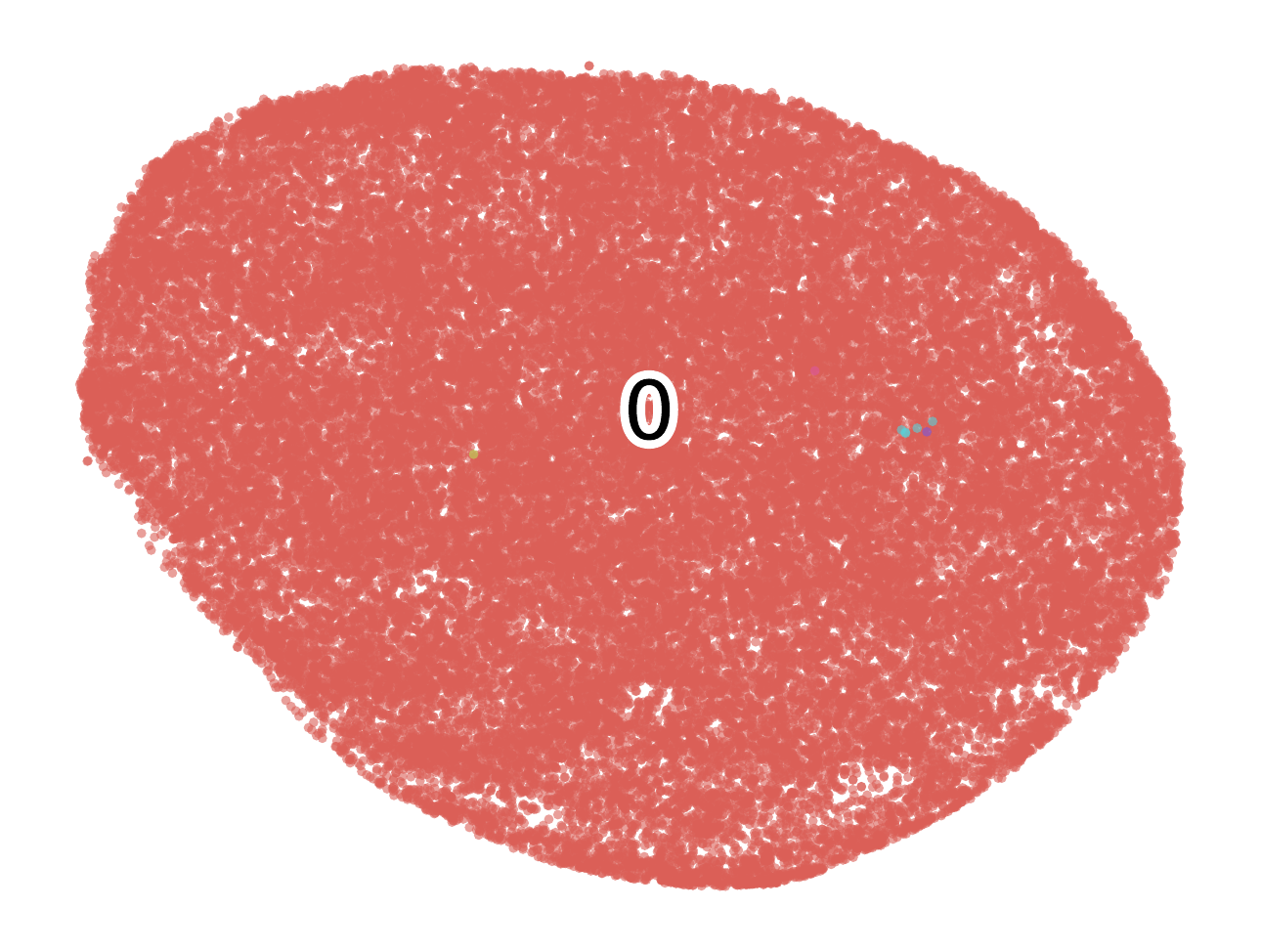}
    \caption{DiCGAN\label{fig:tsne_dicgan}}
  \end{subfigure}
    \caption{\label{fig:ps} t-SNE of $50K$ MNIST samples from (a) training data, (b) GAN and (c) DiCGAN, respectively. Training on MNIST, DiCGAN learns the distribution of small digits, i.e., digit zero, while GAN learns the distribution of the entire dataset.}\vskip-0.1in
\end{figure}

\begin{figure*}[!bt]
    \begin{minipage}[c]{0.62\textwidth}
        \includegraphics[width=\textwidth]{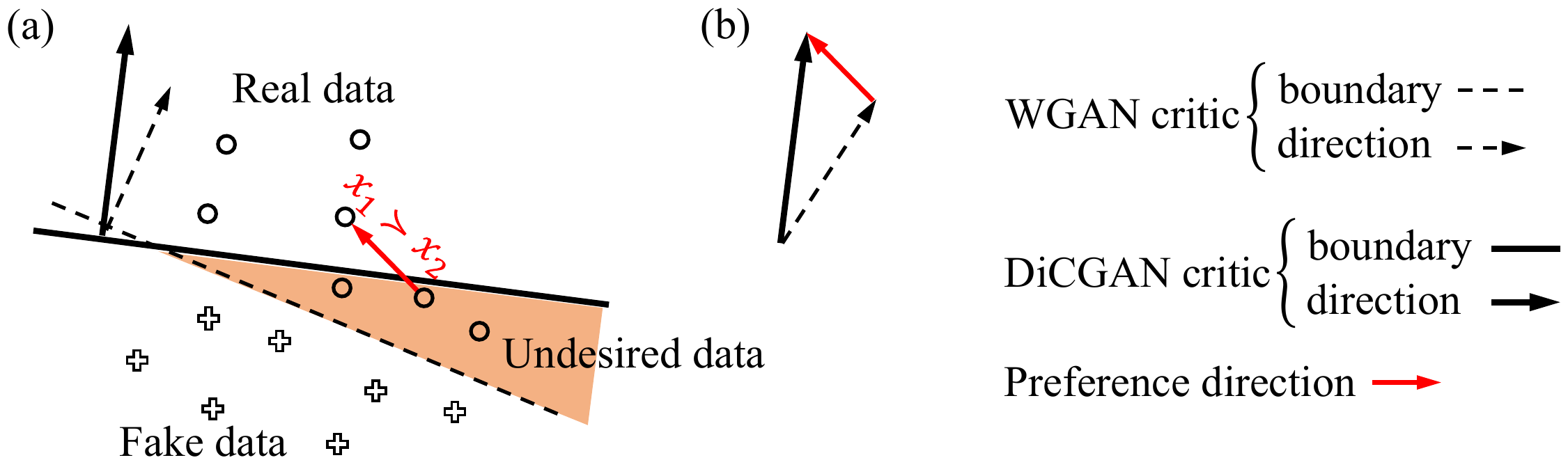}
     \end{minipage}\hfill\hfill
     \begin{minipage}[c]{0.37\textwidth}
        \caption{\label{fig:motivation}Illustration of why DiCGAN can learn the user-desired data distribution. (a) DiCGAN's critic pushes fake data towards the real desired data while WGAN's critic pushes fake data towards all the real data. (b) The change of DiCGAN's critic direction is driven by the preference direction. Note that the preference direction is learned from all pairwise preferences.}
      \end{minipage}\vskip-0.15in
\end{figure*}
Instead of soliciting global knowledge, we consider an easier setting where GAN can be guided towards the distribution of user-desired data by user preferences. 
In particular, pairwise preferences are the most popular form of user preferences due to their simplicity and easy accessibility~\cite{lu2011learning}. Such supervision only requires local knowledge, which can be easily collected with local ranking information. For training a robot to clean a table, it is feasible to compare two behaviors of the robot and determine which one is preferred w.r.t. the goal of cleaning the table~\cite{christiano2017deep}. In addition, resorting to pairwise preferences, it is not necessary to label all training data.
Therefore, our target is to incorporate pairwise preferences into the learning of GAN, \mbox{so as to guide the generation of the desired data.}

Relativistic GAN (RGAN)~\cite{jolicoeur-martineau2018} is a variant of regular GAN and is proposed to learn the whole data distribution. It considers the critic values as the indicators of sample quality, which are similar to the ranking scores. Motivated by this, we consider taking the critic values as the ranking scores and define the ranking loss for pairwise preferences based on the critic values directly. In particular, the difference in critic values for each pair of samples reflects the user's preference over the samples. This is why we call our critic the \textbf{differential critic}, and we propose Differential-Critic GAN (DiCGAN) for learning the user-desired data distribution. As shown in Fig.~\ref{fig:motivation}, the differential critic incorporates the direction of user preferences, which pushes the original critic direction towards the real desired data region instead of the entire real data region. 

The main contributions are summarized as follows:
\begin{itemize}[leftmargin=.2in]
    \item We propose DiCGAN to learn the distribution of the desired data from the entire dataset (Fig.~\ref{fig:tsne_dicgan}) using pairwise preferences. To the best of our knowledge, this is the first work to use local knowledge, i.e.,  local ranking information about user preferences to learn the desired data distribution. Such knowledge is wildly accessible and can be defined on part of the training data.
    \item We are the first one to endow the difference in the critic values between each pair of samples with user preferences. In particular, we incorporate user preferences into GAN’s learning to build Differential-Critic GAN (DiCGAN) via introducing an additional pairwise ranking loss over the WGAN’s critic. Further, we propose an equivalent form of DiCGAN with a hard constraint to ensure data quality.
    \item  We theoretically prove that the distribution of generated samples in DiCGAN can converge to the distribution of desired samples. The relationship between the user preferences and the distribution distance is the first time to be rigorously shown.
    \item We empirically study that our DiCGAN can generate images that meet the user’s interest on MNIST and CelebA-HQ and help design biological products with desired properties on the gene sequence dataset. Our DiCGAN outperforms the baselines in learning the distribution of desired \mbox{data especially when labels of desired data are limited.}
\end{itemize}

\section{Background}
\label{sect:backgroud}
\subsection{Related Work}
Vanilla GANs, like original GAN~\cite{Goodfellow2014}, Wasserstein GAN (WGAN)~\cite{pmlr-v70-arjovsky17a}, can be adapted to learn a user-desired data distribution. A naive way is to first select the samples possessing the desired property based on a universal criterion and then perform regular GAN training only on the selected samples to derive the desired data distribution. 
However, the criterion needs to give a global ranking over the whole data in terms of the interested property so as to pick up desired data, which is expensive. It may not be accessible in real applications. Even, these GANs will fail when the desired samples are insufficient.

The conditional variants of GAN~\cite{mirza2014conditional,odena2017conditional} can be applied in this setting by modeling ``desired/undressed’’ labels as condition variables to learn the conditional desired data distribution. However, the splitting of desired data and undesired data also requires a universal criterion. On the other hand, the generation performance of condition-based GAN is governed by the respective conditions with sufficient training observations. When the desired data is limited, the conditional modeling is dominated by the major classes, i.e., undesired data, resulting in a failure to capture the desired data distribution.

Feedback GAN (FBGAN)~\cite{gupta2019feedback} successfully derives the user-desired data distribution with limited desired data by iteratively introducing desired samples into the training data. Specifically, FBGAN is pre-trained with all training data using the vanilla GAN. At each training epoch, the generator first generates certain amounts of samples. The generated samples possessing the desired property are selected by a universal criterion and used to replace the old training data. Then, regular GAN is trained with the updated training data. Since the ratio of the desired samples gradually increases in the training data, all training data will be replaced with the desired samples. Finally, FBGAN would obtain the desired data distribution. 
Instead of explicitly selecting desired samples, an intuitive way is to first pre-train GAN on all data and then fine-tune it with a classification loss of classifying the generation as ``desired'' to derive the desired data distribution. However, these methods are still restrictive due to the requirement of the universal criterion.

All literature methods resort to a universal criterion to select the desired data in order to learn the desired data distribution, but the criterion requires expensive global knowledge and may even not exist in real applications. In addition, all methods need to label the entire training data, which incurs huge costs. Our work derives the user-desired data distribution using pairwise preferences, which only requires local knowledge and can reduce the burden of labeling the whole training data~\cite{DBLP:journals/ml/PanHT18,pan2022fast,10.1561/1500000016}.
Therefore, our DiCGAN has the advantage of requiring less and more accessible supervision than existing approaches.

\begin{table}[!b]
\centering \vskip-0.1in
\caption{\label{tb:method_comp} Comparison of DiCGAN with WGAN and RGAN in terms of the target data distribution and the critic value.}
	\renewcommand{\arraystretch}{1.}
	\setlength{\tabcolsep}{1.mm}{	
		\scalebox{1.}{
\begin{tabular}{c|c|c|c}\toprule[1.3pt]
\multirow{2}{*}{Method} & \multirow{2}{*}{target distribution} & \multicolumn{2}{c}{critic value}  \\ \cline{3-4}
& & scope & physical meaning    \\\midrule[1pt]
WGAN                    & whole                                     & distribution level  & distribution distance \\
RGAN                    & whole                                     & sample level        & data quality          \\
DiCGAN                  & partial (desired)                         & sample level        & user preference    \\
\bottomrule[1.3pt]
\end{tabular}}}\vskip-0.1in
\end{table}
\subsection{Preliminaries: Generative Adversarial Networks}
GAN~\cite{Goodfellow2014} performs generative modeling by learning a map from low-dimensional latent space $\mathcal{Z}$ to data space $\mathcal{X}$,  i.e., $G: \mathcal{Z} \rightarrow \mathcal{X}$, 
given samples from the training data distribution, namely, $x \sim p_{\mathrm{r}}(x)$. The goal is to find $G$ that achieves $p_{\uptheta}(x)=p_{\mathrm{r}}(x)$, where $p_{{\uptheta}}(x)$ is the distribution of fake data $x=G(z)$.

In order to train the generator, GAN introduces another network, i.e., discriminator, to discriminate real data from fake data. The generator is trained to produce images that are conceived to be realistic by the discriminator. Two networks are trained alternately until the generator successfully fools the discriminator. GAN~\cite{Goodfellow2014}'s objective is defined as follows:
\begin{equation}
    \min_{G}\max_{D}\mathbb{E}_{p_{\mathrm{r}}(x)}\left[\log \sigma\left(D(x)\right)\right]+\mathbb{E}_{p_{{\uptheta}}(x)}\left[\log\left(1-\sigma\left(D(x)\right)\right)\right],
\end{equation}
where $\sigma\left(D(x)\right)$ is the probability that the input data is real and $\sigma$ is the sigmoid function. $D(x)$ is the non-transformed discriminator output, which is called \textit{critic value} in WGAN~\cite{pmlr-v70-arjovsky17a}. 

WGAN~\cite{pmlr-v70-arjovsky17a,gulrajani2017improved} and RGAN~\cite{jolicoeur2019relativistic} are stable variants of GANs defining the loss functions in terms of the critic $D$, i.e., the non-transformed discriminator. Specifically, WGAN measures the quality of fake data in terms of the Wasserstein distance (W-distance) between the real data distribution and the fake data distribution. The W-distance is approximated by the difference in the average critic values between the real data and the fake data. WGAN's objective is defined as follows:
\begin{equation}
    \min_{G}\max_{D}\mathbb{E}_{p_{\mathrm{r}}(x)}\left[{D}(x)\right]-\mathbb{E}_{p_{{\uptheta}}(x)}\left[{D}\left(x\right)\right],
\end{equation}
where $D$ is the critic enforced with a $1$-Lipschitz constraint.

RGAN estimates the probability that the given real data is more realistic than randomly sampled fake data by using the difference in the critic values. \mbox{Its objective is defined as follows:}
\begin{equation}
\begin{aligned}
    &\max_{D}\mathbb{E}_{x_{\mathrm{r}}\sim p_{\mathrm{r}}(x), x_{\uptheta} \sim p_{{\uptheta}}(x) }\left[\log\left(\sigma\left({D}(x_{\mathrm{r}})-{D}\left(x_{\uptheta}\right)\right)\right)\right],\\
    &\max_{G}\mathbb{E}_{x_{\mathrm{r}}\sim p_{\mathrm{r}}(x), x_{\uptheta} \sim p_{{\uptheta}}(x) }\left[\log\left(\sigma\left({D}(x_{\uptheta})-{D}\left(x_{\mathrm{r}}\right)\right)\right)\right].
\end{aligned}
\end{equation}
It has a similar form as the pairwise ranking loss~\cite{burges2005learning}, but interpreting the critic values as \mbox{ranking scores for sample quality.}

Our DiCGAN considers the critic values as the ranking scores, which has a similar viewpoint to RGAN. But very differently, 1) our DiCGAN aims to learn the distribution of user-desired data when only part of the dataset possesses the desired property while RGAN targets to learn the whole data distribution;
2) our DiCGAN uses critic values to represent user preferences while RGAN uses critic values to describe data quality. We summarize the comparison of our DiCGAN with WGAN and RGAN in Table~\ref{tb:method_comp}.
Our DiCGAN can be applied to GAN variants based on the critic, like WGAN and RGAN. In this work, we develop our DiCGAN on WGAN.

\begin{table}[!t]
\centering
\caption{\label{tb:notations} Main mathematical notations in this paper.}
	\renewcommand{\arraystretch}{1.}
	\setlength{\tabcolsep}{1.mm}{	
		\scalebox{1.}{
\begin{tabular}{l|l}
\toprule[1.3pt]
Notation   & Explanation   \\ \hline
$G$     & generator     \\
$D$     & critic    \\
$\mathrm{X}$    & training samples  \\
$x_1 \succ x_2$     & $x_1$ is preferred over $x_2$ \\
$\mathrm{S}$    & pairwise preferences, $\mathrm{S}=\big\{s=(x_1, x_2)|x_1 \succ x_2, x_1, x_2 \in \mathrm{X}\big\}$ \\
$p_{\mathrm{r}}(x)$     & the distribution of the whole data    \\
$p_{\mathrm{d}}(x)$     & the distribution of the user-desired data    \\
$p_{\mathrm{u}}(x)$     & the distribution of the undesired data    \\
$p_{\uptheta}(x)$     & the target generative model    \\
$T$     & threshold to discriminate the desired data from the undesired data  \\
$f(x)$     & score function that maps sample $x$ to the score that reflects \\
           & the user's preference for $x$  \\
$d(,)$     & distribution distance \\
$\varepsilon$  & distance constraint that guarantees good generation quality \\
\toprule[1.3pt]
\end{tabular}}}\vskip-0.1in
\end{table}

\section{DiCGAN for User-desired Distribution}
\label{sect:approach}
No longer learning the distribution of the whole dataset, GAN is applied in a new scenario, where the distribution of the partial dataset is what we desire.
User-desired data may refer to some certain class of data among multiple class datasets, or observations with/without some particular attributes or properties.  Such data can be induced from user preference, which can be represented as an ordering relation between two or more samples in terms of the desired property. We propose differential-critic GAN (DiCGAN) to learn the desired data distribution from \mbox{the user preferences along with the whole dataset.}

\subsection{Learning the Distribution of User-desired Data}
A universal criterion to help derive a user-desired data distribution can be constructed based on a score function. Following the score-based ranking literature~\cite{DBLP:conf/icml/CaoQLTL07}, we suppose that there exists a numeric score associated with each sample,  reflecting the user's preference for the sample. A higher score indicates that its corresponding sample is preferred by the user. 
In detail, let $f()$ denote a score function that maps sample $x$ to score $f(x)$. Let $T$ denote the threshold to discriminate the desired data from the undesired data. That is, if a sample's score $f(x)$ exceeds a predefined threshold $T$, namely, $I(f(x)>T)=1$, the sample $x$ is desired by the user. $I()$ is a sign function, which equals $1$ if its condition is true and $0$ otherwise. 
For the sake of explanation, we use $p_{\mathrm{r}}(x), p_{\mathrm{d}}(x), p_{\mathrm{u}}(x)$ to denote the distribution of the whole data,  the user-desired data and the undesired data, respectively.

Current literatures~\cite{gupta2019feedback,pmlr-v70-arjovsky17a,mirza2014conditional} needs to explicitly label desired/undesired data in order to learn the distribution of the desired data~$p_{\mathrm{d}}(x)$. Namely, the desired data $\mathrm{X}_{\mathrm{d}}=\{x|I(f(x)>T)=1, x\sim p_{\mathrm{r}}(x)\}$. The undesired data $\mathrm{X}_{\mathrm{u}}=\{x|I(f(x) \leq T)=1, x\sim p_{\mathrm{r}}(x)\}$.
However, the assumption that the score function $f()$ is predefined may be too restrictive for real applications, where no universal and explicit criteria exist. Second, the definitions of the desired/undesired samples are highly
dependent on the choice of the threshold~$T$.  Third, labeling over the entire dataset incurs high costs.

Instead of relying on a predefined score function (global knowledge), we propose to learn the desired data distribution in a straightforward manner from the user preferences.
Here, we consider general auxiliary information, i.e., the pairwise preferences, to represent the user preferences, due to its simplicity and easy accessibility. For any two samples $x_1, x_2 \sim p_{\mathrm{r}}(x)$, let $x_1 \succ x_2$ denote that $x_1$ is preferred over $x_2$ according to the user's preference over the samples. Let $\mathrm{X}$ be the training samples, i.e., $\mathrm{X}=\{x \sim p_{\mathrm{r}}(x)\}$. A collection of pairwise preferences $\mathrm{S}$ is obtained by:
\begin{equation}
    \mathrm{S}=\big\{s=(x_1, x_2)|x_1 \succ x_2, x_1, x_2 \in \mathrm{X}\big\}.
    \label{eq:s}
\end{equation}
$S$ can be defined on part of the dataset.
\begin{remark}
We can construct $\mathrm{S}$ by first randomly drawing sample pairs from part of the training samples and then asking the user to select the preferred one from each pair.
\end{remark}
\begin{definition}[Problem Setting]
Given the training samples $\mathrm{X}$ and the pairwise preferences $\mathrm{S}$, the target is to learn a generative model $p_{{\uptheta}}(x)$ that is identical to the distribution of the desired data $p_{\mathrm{d}}(x)$, i.e., $p_{{\uptheta}}(x)=p_{\mathrm{d}}(x)$.
\end{definition}

\subsection{Differential Critic GAN}
\label{sect:dicgan}
Instead of adopting WGAN's critic for quality assessment, we present the differential critic for modeling pairwise preferences.
The differential critic can guide the generation of the user-desired data.

\subsubsection{Pairwise Preference}
We consider incorporating pairwise preferences into the training of GAN.

The score-based ranking model~\cite{10.1145/1390334.1390382} is used to model the pairwise preferences. It learns the score function~$f()$, of which the score value, called ranking score in the model, is the indicator of the user preferences. Further, the difference in ranking scores can indicate the pairwise preference relation. That is, for any pair of samples $x_1, x_2$, if $x_1 \succ x_2$ then $f(x_1)-f(x_2)>0$ and vice versa. For any pairwise preference $s: x_1 \succ x_2$, the ranking loss we consider is as follows:
\begin{equation}
    h(s) = \max\left(0, -\left(f\left(x_{1}\right)-f\left(x_{2}\right)\right)+m\right),
 \label{eq:margin}
\end{equation}
where $m$ is the ranking margin. For other forms of ranking losses, the reader can refer to~\cite{10.1145/1390334.1390382}.

Instead of learning the score function independently of GAN’s training, we consider incorporating it into GAN’s training, guiding GAN towards the generation of the desired data. 
The critic in RGAN~\cite{jolicoeur-martineau2018} is similar to the score function, where the critic values are used to describe the quality of samples. We are motivated to take the critic values as the ranking scores and define the ranking loss on the critic directly. In particular, the difference in the critic values for each pair of samples reflects the user's preference over the samples.

\subsubsection{Loss Function}
We build DiCGAN based on WGAN and the pairwise ranking loss is defined over the WGAN's critic. The loss function for DiCGAN is defined as:
\begin{equation}
\begin{aligned}
    \min_{G}\max_{D} \mathbb{E}_{p_{\mathrm{r}}(x)}\left[{D}(x)\right]
    -\mathbb{E}_{p_{\mathrm{\uptheta}}(x)}\left[{D}\left(x\right)\right]
    -\lambda\frac{1}{|\mathrm{S}|}\sum_{s\in \mathrm{S}}\left[h\left(s\right)\right],
\end{aligned}
\label{eq:objective}
\end{equation}
where $h(s)$ is the pairwise ranking loss (Eq.~\eqref{eq:margin}). $f()$ is approximated by the critic $D$. Namely, $h(s) \approx \max\left(0, -\left(D\left(x_{1}\right)-D\left(x_{2}\right)\right)+m\right)$. $\lambda$ is a balance factor, which will be discussed further in section~\ref{sec:3.3}.
Similar to WGAN, we formulate the objective for the differential critic $L_D$ and the generator $L_G$ as:
\begin{equation}
\begin{aligned}\label{eq:loss}
    L_D &= \frac{1}{b}\sum_{i=1}^{b}\left(D(x^{i})-D(G(z^{i}))\right)-\lambda\frac{1}{n_{\mathrm{s}}}\sum_{j=1}^{n_\mathrm{s}}h(s^{j}),\\
    L_G &= \frac{1}{b}\sum_{i=1}^{b}-D(G(z^{i})).
\end{aligned}
\end{equation}
where $b$ is the batch size. $n_{\mathrm{s}}$ is the number of preferences sampling from $\mathrm{S}$.

The advantages of DiCGAN are twofold.
$(1)$ The introduced ranking loss in DiCGAN is defined on the critic directly. Apart from WGAN, it can be easily applied to other GAN variants developed based on the critic, e.g., RGAN.
$(2)$ The construction of pairwise preferences involves the undesired data. Thus, the undesired samples are also utilized during the training and they, together with desired samples, provide the generation direction of the desired data for the generator. 

We argue that the differential critic in DiCGAN can guide the generator to learn the user-desired data distribution. As shown in Fig.~\ref{fig:motivation}, the differential critic in DiCGAN provides the direction towards the real desired data. We denote the critic direction as the moving direction of the fake data, which is orthogonal to the decision boundary of the critic. Referring to Eq.~\eqref{eq:objective}, DiCGAN's critic loss consists of two terms: the vanilla WGAN loss and the ranking loss.  The vanilla WGAN loss imposes the critic direction from the fake data to the real data. Meanwhile, the ranking loss induces a user preference direction, which points from the undesired data to the desired data. Combining these two effects, the critic direction of DiCGAN targets the region of the real desired data only. 

The above proposed DiCGAN (Eq.~\eqref{eq:objective}) however requires sensitive hyperparameter tuning during the training. Revisiting the objective (Eq.~\eqref{eq:objective}), the first two terms (WGAN loss) can be considered as the WGAN regularization, which ensures the generated data distribution is close to the whole real data distribution, i.e., $p_{\mathrm{\uptheta}} \approx p_{\mathrm{r}}$. The third term (ranking loss) serves as a correction for WGAN, which makes WGAN slightly biased to our target of learning the desired data distribution, i.e., $p_{{\uptheta}}=p_{\mathrm{d}}$.
Therefore, the WGAN regularization serves as the cornerstone of our DiCGAN. Particularly, if the desired data distribution is close to the whole data distribution, the ranking loss easily corrects the WGAN to achieve the desired data distribution. Otherwise, satisfactory performance of DiCGAN may require the online hyperparameter tuning of $\lambda$ during the training process. \mbox{Thus, it is hard to train with Eq.~\eqref{eq:objective} in this case.} 

\subsection{Reformulating DiCGAN to Ensure Data Quality}
\label{sec:3.3}
In this section, we reformulate DiCGAN as a form with a hard constraint. This form indicates that the tuning for Eq.~\eqref{eq:objective} relies largely on the distance between the distributions of the desired data and the undesired data. Further, it inspires us to derive a more efficient solution -- minor correction and major correction.

According to the above analysis, the WGAN
loss serves as the cornerstone of our DiCGAN and the pairwise ranking loss serves as a correction for WGAN. Thus, we consider reformulating the objective of DiCGAN, i.e., Eq.~\eqref{eq:objective} into an equivalent objective with a hard WGAN constraint:
\begin{equation}
\begin{aligned}\label{eq:objective_3}
    &\min_{G} \max_{D}-\sum_{s\in \mathrm{S}}\left[h\left(s\right)\right], \\
    \text{s.t.} \; d(p_{\mathrm{r}}, p_{\uptheta})= &\left|\mathbb{E}_{p_{\mathrm{r}}(x)}\left[{D}(x)\right]-\mathbb{E}_{p_{{\uptheta}}(x)}\left[{D}\left(x\right)\right] \right|\leq \varepsilon.
\end{aligned}
\end{equation}
where $\varepsilon>0$. Note that we impose an explicit non-negative constraint on $d(p_{\mathrm{r}}, p_{\uptheta})$, to highlight that it is a distance metric. 
It is still equivalent to WGAN loss from its definition. Eq.~\eqref{eq:objective} is the Lagrangian function. Since Eq.~\eqref{eq:objective_3} imposes a hard constraint on the WGAN loss, it is more difficult to optimize compared to Eq.~\eqref{eq:objective}. 
However, more efficient solutions of DiCGAN can be explored by analyzing Eq.~\eqref{eq:objective_3} regarding the hard constraint on $d(p_{\mathrm{r}}, p_{\uptheta})$.

\begin{figure}[!t]
    \centering
    \includegraphics[width=0.88\linewidth]{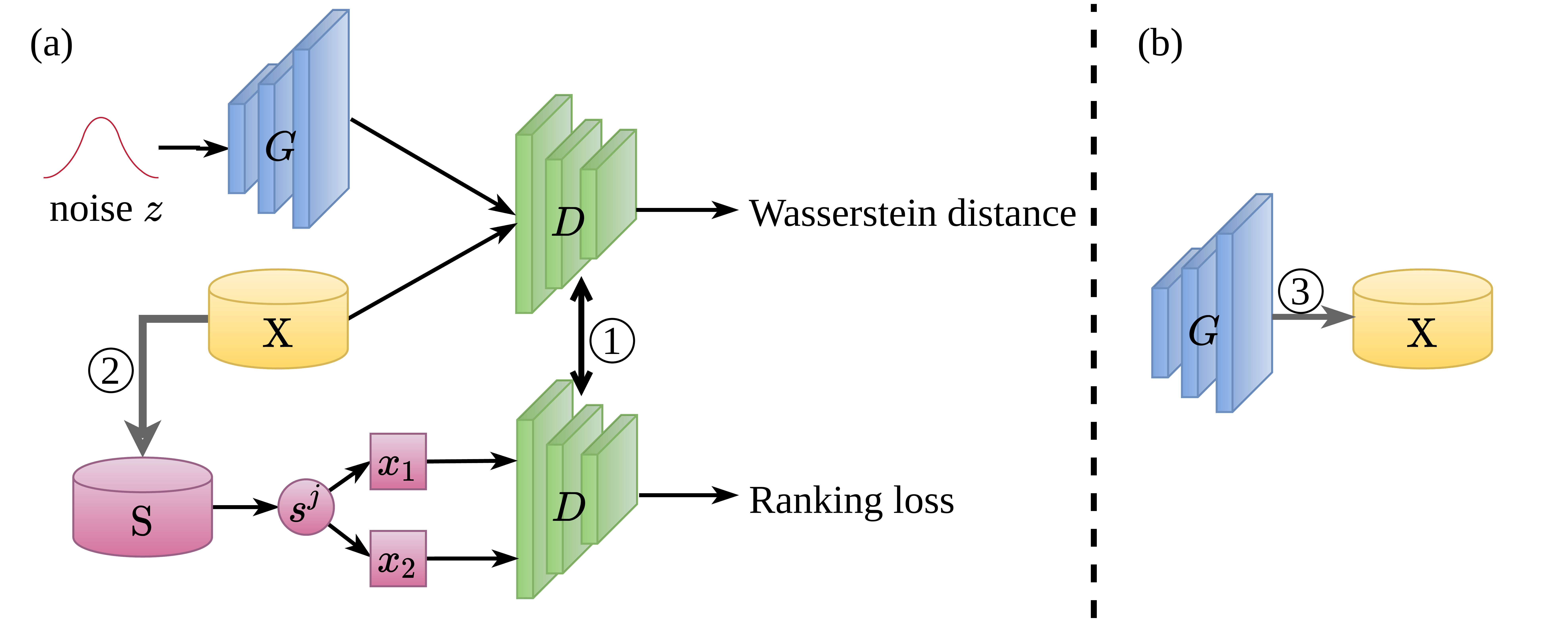}
    \caption{\label{fig:framework} DiCGAN architecture and training. DiCGAN is alternately trained with step (a) and (b). (a) Training DiCGAN at one minor correction. (b) Replacing data after one minor correction. \textcircled{1} denotes the shared differential critic $D$. \textcircled{2} denotes that $\mathrm{S}$ is constructed from $\mathrm{X}$ using Eq.~\eqref{eq:s}. \textcircled{3} denotes data replacement using Eq.~\eqref{eq:x}.}\vskip-0.1in
\end{figure}
In terms of a minor correction situation, this means the desired data distribution $p_{\mathrm{d}}$ is close to the real data distribution $p_{\mathrm{r}}$. Therefore, the hard constraint dominates the training goal of DiCGAN. By assigning a proper $\lambda$ to ensure the constraint is satisfied, Eq.~\eqref{eq:objective} can learn the distribution of the user-desired data while ensuring data quality. 

In terms of a major correction situation, this means the desired data distribution $p_{\mathrm{d}}$ is quite diverse from the real data distribution $p_{\mathrm{r}}$. Therefore, DiCGAN needs to achieve an equilibrium between the correction, imposed by the ranking loss, and the hard constraint, imposed by the WGAN loss. However, a large correction may not ensure the quality of the generated data, since the WGAN loss, used to guarantee the image quality, is defined between the generated data and the whole real data. 
To avoid the major correction, we propose to break the major correction into a sequence of minor corrections to ensure data quality. Namely, at each minor correction, we first use the generator $G$ to generate $n_{\mathrm{g}}$ samples, denoted as~$\mathrm{X_g}$:
\begin{equation}
    \mathrm{X}_\mathrm{g}^e \leftarrow \{G^e(z^{1}), \ldots, G^e(z^{n_{\mathrm{g}}})\},\quad \{z^{i} \sim p(z)\}_{i=1}^{n_{\mathrm{g}}},
    \label{eq:xg}
\end{equation}
where $e$ is $e$-th minor correction. Then we replace partial old training samples with the generated samples:
\begin{equation}
    \mathrm{X}^{e+1} \leftarrow \mathrm{X}^{e} \setminus \mathrm{X}_\mathrm{o}^{e} \cup \mathrm{X}_\mathrm{g}^e,
    \label{eq:x}
\end{equation}
\mbox{where $\mathrm{X}_\mathrm{o}^{e}$ are the old (least-recently added) $n_{\mathrm{g}}$ samples in $\mathrm{X}^{e}$.} 

Due to the ranking loss, the generated data distribution~$p_{\uptheta}^e$ is closer to the desired data distribution~$p_{\mathrm{d}}$, compared to the constructed $p_{\mathrm{r}}^{e}$ at each minor correction. Therefore, the iterative replacement (Eq.~\eqref{eq:x}) can gradually shift the real data distribution $p_{\mathrm{r}}$ towards the desired data distribution $p_{\mathrm{d}}$. Namely, $d(p_{\mathrm{r}}, p_{\mathrm{d}}) >\cdots > d(p_{\mathrm{r}}^{e}, p_{\mathrm{d}}) > d(p_{\mathrm{r}}^{e+1}, p_{\mathrm{d}}) >\cdots$. According to the monotone convergence theorem, $d(p_{\mathrm{r}}^{e}, p_{\mathrm{d}})$ will converge to zero when $e \rightarrow +\infty$.
So only a minor correction needs to be imposed on $p_{\uptheta}^{e}$ by optimizing Eq.~\eqref{eq:objective} at each minor correction. Iteratively, the generated distribution $p_{\uptheta}$ shifts towards $p_{\mathrm{d}}$. The training algorithm is summarized in Algorithm~\ref{alg:dicgan_whole}. The architecture and training of DiCGAN can be seen in Fig.~\ref{fig:framework}. For the sake of easy optimization, we pretrain the differential critic $D$ and the generator $G$ using vanilla WGAN.
\begin{algorithm}[!t]
  \caption{Training algorithm of DiCGAN}
  \label{alg:dicgan_whole}
  \begin{algorithmic}[1]
  \small
    \STATE \textbf{Input:} training data $\mathrm{X}$, pairwise preferences $\mathrm{S}$
    \STATE \textbf{Initilization:} balance factor $\lambda$, \#generated samples $n_{\mathrm{g}}$, \#pairs $n_{\mathrm{s}}$, batch size $b$, \#iterations per minor correction $n_{\mathrm{i}}$, \#critic iterations per generator iteration $n_\mathrm{critic}$    
    \STATE Pretrain $D$ and $G$
    \REPEAT
    \STATE $\%$ Shift to the user-preferred distribution\\
    \STATE \quad Generate samples using Eq.~\eqref{eq:xg} \\
    \STATE \quad Replace partial old samples in $\mathrm{X}$ with $\mathrm{X_g}$ using Eq.~\eqref{eq:x}
      \STATE Obtain pairwise preferences $\mathrm{S}$ using Eq.~\eqref{eq:s}
    \STATE $\%$ Training of $D$ and $G$ at a minor correction
    \FOR{$l=1,\dots, n_{\mathrm{i}}$}      
      \FOR{$t=1,\ldots,n_\mathrm{critic}$}
        \STATE Sample $\{x^{i}\}_{i=1}^{b}$ from $\mathrm{X}$, $\{z^{i} \sim p(z)\}_{i=1}^{b}$
        \STATE Sample $\{s^j\}_{j=1}^{n_\mathrm{s}}$ from $\mathrm{S}$.
        \STATE Train the differential critic $D$ using $L_D$ in Eq.~\eqref{eq:loss}
      \ENDFOR
      \STATE Train the generator $G$ using $L_G$ in Eq.~\eqref{eq:loss}
    \ENDFOR  
    \UNTIL{converge}
    \STATE \textbf{Output:} \mbox{generator G for desired data distribution}
  \end{algorithmic}
\end{algorithm}

\subsection{Convergence Analysis}
In this section, we analyze the convergence of our DiCGAN under the minor correction and the major correction, respectively. In the case of the minor correction, we prove that the distribution of generated data $p_{\uptheta}(x)$ converges to the distribution of user-desired data $p_{\mathrm{d}}(x)$ via adversarial training of GAN given that the differential critic of DiCGAN converges to the score function whose score describes the user's preference for the sample. In the case of the major correction, $p_{\uptheta}(x)$ is proven gradually moving towards $p_{\mathrm{d}}(x)$ with a sequence of minor correction as one minor correction shifts $p_{\uptheta}(x)$ towards $p_{\mathrm{d}}(x)$ with a certain small distance. 

Suppose the training data $\mathrm{X}=\{x_1, x_2, \ldots, x_n\}$, where $n$ is the number of training samples. Their corresponding scores $o=f(x)$ are $\{o_1, o_2, \ldots, o_n\}$, which describes the user's preference for the sample. The maximum score among the samples is denoted as $o_{max}$ while the minimum one is $o_{min}$.

\begin{proposition}\label{prop:minor}
In the case of the minor correction, i.e., $d\left(p_{\mathrm{r}}(x), p_{\mathrm{d}}(x)\right) \leq \varepsilon$, $p_{\mathrm{\uptheta}}(x)$ converges to $p_{\mathrm{d}}(x)$.
\end{proposition}
\begin{proof}
According to the theory of learning to rank~\cite{10.1561/1500000016}, by setting an appropriate $\lambda$, we have $D$ converge to the score function $f()$ iff $S$ is sufficient. Then the real desired data will be assigned higher scores than the real undesired data.

With Eq.~\eqref{eq:loss}, the generator is optimized to generate samples with scores as high as possible while the critic is optimized to assign the generated samples lower than the training samples. As only training samples preferred by the user are assigned with high scores, when the adversarial training converges, the generated samples are alike samples with high scores, i.e., the desired data, which shares the same principle in~\cite{Goodfellow2014,pmlr-v70-arjovsky17a}. Therefore, $p_{\uptheta}(x)$ converges to $p_{\mathrm{d}}(x)$.
\end{proof}

\begin{proposition}\label{prop:minor_2}
In the case of the minor correction where $d\left(p_{\mathrm{r}}(x), p_{\mathrm{d}}(x)\right) \leq \varepsilon$ and $p_{\uptheta}(x) = p_{\mathrm{d}}(x)$, we can prove that $\mathbb{E}_{p_{{\uptheta}}(x)}\left[{D}\left(x\right)\right] = \mathbb{E}_{p_{\mathrm{r}}(x)}\left[{D}(x)\right] + \delta $, for some $\delta>0$.
\end{proposition}
\begin{proof}
Without loss of generality, we represent $p_{\mathrm{r}}(x)$ and $p_{\uptheta}(x)$ in a fine-grain formulation. Namely,
\begin{equation}\label{eq_slpit}
   p_{\mathrm{r}}(x) =  (1-\alpha) p_{\mathrm{d}}(x) + \alpha p_{\mathrm{u}}(x), \quad p_{\uptheta}(x) =  p_{\mathrm{d}}(x),
\end{equation}
where $\alpha \in [0,1]$ is a very small value such as $d\left(p_{\mathrm{r}}(x), p_{\mathrm{d}}(x)\right) \leq \varepsilon$ is satisfied.

Furthermore, according to our definition of the ranking model, the score for the desired data should be higher than that of the undesired data, namely 
\begin{equation}\label{eq:c}
D(x)
\begin{cases}
>T, & \text{if }x \sim p_{\mathrm{d}}(x);\\
\leq {T}, & \text{if }x \sim p_{\mathrm{u}}(x),
\end{cases}
\end{equation}
where $o_{min} < {T} < o_{max}$. (1) $(o_{min}, o_{max})$ is introduced since the critic score is always bounded; (2) {$T$} denotes some value to discriminate the desired data from undesired data.   

Taking into consideration of both Eq.~\eqref{eq_slpit},~\eqref{eq:c}, we have 
\begin{align}\label{eq_main}
&\mathbb{E}_{p_{\mathrm{r}}(x)}\left[{D}\left(x\right)\right] 
= \int \left[(1-\alpha) p_{\mathrm{d}}(x) + \alpha p_{\mathrm{u}}(x)\right]  D(x) dx \\
&= \int \left[(1-\alpha) p_{\mathrm{d}}(x)\right] D(x)  dx  + \int \left[\alpha p_{\mathrm{u}}(x)\right]  D(x) dx \nonumber\\
&= \int p_{\mathrm{d}}(x) D(x) dx -\alpha ( \int p_{\mathrm{d}}(x) D(x) dx - \int p_{\mathrm{u}}(x) D(x) dx ).\nonumber\
\end{align}
Considering that (1) $p_{\mathrm{d}}(x)$ and $p_{\mathrm{u}}(x)$ are always positive; (2) $D(x)$ is continuous and bounded on the domain of $x$ with respect to $p_{\mathrm{d}}(x)$ and $p_{\mathrm{u}}(x)$, respectively, we have the following derivations according to the mean value theorem for integrals: 
\begin{subequations}
\begin{align*}
& \exists \ \xi_{\mathrm{d}} \in ({T}, o_{max}], \quad \int p_{\mathrm{d}}(x) D(x) dx = \xi_{\mathrm{d}} \int p_{\mathrm{d}}(x) dx = \xi_{\mathrm{d}}; \\
& \exists \ \xi_{\mathrm{u}} \in [o_{min}, {T}], \quad \int p_{\mathrm{u}}(x) D(x) dx = \xi_{\mathrm{u}} \int p_{\mathrm{u}}(x) dx = \xi_{\mathrm{u}}. 
\end{align*}
\end{subequations}
Since $\xi_{\mathrm{d}} > \xi_{\mathrm{u}}$, we have
\begin{align*}
& \alpha \left( \int p_{\mathrm{d}}(x) D(x) dx - \int p_{\mathrm{u}}(x) D(x) dx \right) = \alpha(\xi_{\mathrm{d}} - \xi_{\mathrm{u}}) >0.
\end{align*}
$\Longrightarrow \textrm{Eq.~\eqref{eq_main}} = \mathbb{E}_{p_{\mathrm{d}}(x)}\left[{D}\left(x\right)\right] - \delta,  \textrm{where }\delta = \alpha(\xi_{\mathrm{d}} - \xi_{\mathrm{u}}) >0.$
Furthermore, by replacing $p_{\mathrm{d}}(x)$ with $p_{\uptheta}(x)$, we have 
\begin{equation*}
    \mathbb{E}_{p_{{\uptheta}}(x)}\left[{D}\left(x\right)\right] = \mathbb{E}_{p_{\mathrm{r}}(x)}\left[{D}(x)\right] + \delta, 
\end{equation*}
for some $\delta>0$.
\qedhere
\end{proof}

\begin{corollary}\label{cor:minor}
The minor correction moves $p_{\uptheta}$ towards $p_{\mathrm{d}}$ with distance $\delta$ compared to $p_{\mathrm{r}}$, i.e., $d(p_{\mathrm{r}}, p_{\mathrm{d}})-d(p_{\uptheta}, p_{\mathrm{d}})=\delta$. 
\end{corollary}

\begin{proposition}\label{prop:major}
In the case of the major correction, i.e., $d\left(p_{\mathrm{r}}, p_{\mathrm{d}}\right)=T_0$, the distance between $p_{\uptheta}(x)$ and $p_{\mathrm{d}}(x)$ converges to $d(p_{\uptheta}^k, p_{\mathrm{d}})=T_0-k\delta$ after $k$ minor corrections.
\end{proposition}
\begin{proof}
The major correction is divided into a sequence of minor corrections. 
In the first minor correction, the training data distribution is $p_{\mathrm{r}}(x)$. Derived by Corollary~\ref{cor:minor}, after this minor correction, we have 
\begin{equation*}
    d(p_{\mathrm{r}}, p_{\mathrm{d}})-d(p_{\uptheta}^1, p_{\mathrm{d}})=\delta.
\end{equation*} 
With $d\left(p_{\mathrm{r}}, p_{\mathrm{d}}\right)=T_0$, we can get
\begin{equation*}
    d(p_{\uptheta}^1, p_{\mathrm{d}})=T_0-\delta.
\end{equation*}

In the second minor correction, we can replace all training samples with the generated samples obtained in the first correction. Thus, the training data distribution becomes $p_{\uptheta}^1(x)$. After two minor corrections, similarly, we can get
\begin{equation*}
    d(p_{\uptheta}^2, p_{\mathrm{d}})=T_0-2\delta.
\end{equation*}
So on and so forth. After $k$ minor corrections, we can get 
\begin{equation*}
    d(p_{\uptheta}^k, p_{\mathrm{d}})=T_0-k\delta.
\qedhere
\end{equation*}
\end{proof}

\begin{corollary}\label{cor:major}
In the case of major correction, $p_{\uptheta}(x)$ converges to $p_{\mathrm{d}}(x)$ after $K$ minor corrections, where $K=\lceil \frac{T_0}{\delta} \rceil$.
\end{corollary}
\begin{proof}
Since $d(p_{\uptheta}^k, p_{\mathrm{d}})>=0$ and $d(p_{\uptheta}^k, p_{\mathrm{d}})=T_0-k\delta$ decreases as $k$ increases, $d(p_{\uptheta}^k, p_{\mathrm{d}})$ converges to zero when $k \rightarrow +\infty$ according to the monotone convergence theorem.
Specifically, when $k=\lceil \frac{T_0}{\delta} \rceil$, $d(p_{\uptheta}^k, p_{\mathrm{d}})=0$.
\end{proof}

From Proposition~\ref{prop:minor} and Corollary~\ref{cor:major}, we conclude that in DiCGAN, the distribution of generated samples $p_{\mathrm{\uptheta}}(x)$ converges to $p_{\mathrm{d}}(x)$.

\subsection{Technical Novelty of DiCGAN}

We elaborate our DiCGAN’s technical novelty and its significance in terms of the following four aspects:
\begin{itemize}[leftmargin=.15in]
    \item \textbf{The first one to apply user preferences for desired data generation.} Current approaches for desired data generation require expensive global knowledge, which is usually not available. Our DiCGAN uses local knowledge only -- local ranking information about user preferences.  
    \item \textbf{New insight for critic value.} Our DiCGAN considers the critic values as the ranking scores that represent user preferences. Based on this insight, we can incorporate user preferences into GAN’s learning instead of learning the score function \mbox{for user preferences independently of GAN’s training:}
    \begin{itemize}[leftmargin=.2in]
        \item[1.] \textit{Naive combination between user preferences and GAN does not work.} Introducing an additional critic that learns from user preferences onto WGAN would lead to the conflict between WGAN's original critic for good quality generation and the extra critic for desired data generation.
        \item[2.] As critic values can represent data quality and user preferences, we define a differential critic by defining an additional pairwise ranking loss on the WGAN’s critic and build DiCGAN (Eq.~\eqref{eq:objective}). Then the original WGAN's critic loss encourages:
        \begin{equation*}
            x_1>x_2 \text{ for } x_1\sim p_{\mathrm{d}}(x) \text{ and } x_2 \sim p_{\mathrm{u}}(x);
        \end{equation*} 
        and the ranking loss encourages:
        \begin{equation*}
            x_1>x_2 \text{ for } x_1\sim p_{\mathrm{r}}(x) \text{ and } x_2 \sim p_{\uptheta}(x).
        \end{equation*} 
        \textit{The critic would guide the generation with high critic values, encouraging the generation of user-desired data with good quality.}
    \end{itemize}
    \item \textbf{Efficient solution by an equivalent form with a hard constraint.} The naive form of DiCGAN (Eq.~\eqref{eq:objective}) requires heavy hyper-parameter tuning when there is a large distance gap between the distributions of the desired data and the whole data. Thus, we propose an equivalent form of DiCGAN (Eq.~\eqref{eq:objective_3}). Based on it, we derive a more efficient solution in terms of minor correction and major correction, \mbox{which can always ensure good data quality.}
    \item \textbf{The first rigorous model for desired data generation.} To the best of our knowledge, no previous work theoretically studies this problem. The above three points pave the way for the theoretical convergence proof of desired data generation:
    \begin{itemize}[leftmargin=.15in]
        \item[1.] Because of DiCGAN’s form with a hard distance constraint, we can analyze the convergence of DiCGAN under the minor correction and the major correction.
        \item[2.] Since we interpret the critic values in DiCGAN as the ranking scores, the relationship between the user preferences (reflected by ranking scores) and the distribution distance (represented by critic values) [Proposition~\ref{prop:minor_2} and Corollary~\ref{cor:minor}] can be derived.  This is the first time that such a relationship is rigorously shown.
    \end{itemize}
\end{itemize}

\subsection{\revision{Discussions about Pairwise Regularization to Generator}}
\label{sec:prg}
\revision{
In this section, we claim that adding the pairwise regularization to the generator requires heavy supervision and is invalid. Our DiCGAN thus does not consider such regularization.
}

\revision{
As the target is to learn the desired data distribution, the regularization on the generator can be used to make the critic values of the generated samples larger than those of the undesired samples. Specifically, a selector is first applied to give a full ranking for the training data, and then the bottom $K_0$ samples are picked up as the undesired samples. The pairwise preferences are then defined over the generated samples and the undesired samples. Note that the undesired subset of the training data requires labeling all training data.
}

We consider two cases of adding the regularization to the generator. 
First, we only add the pairwise regularization to the generator (PRG-1). Second, we add the regularization to the generator together with the regularization on the critic (PRG-2). 

The objective for PRG-1 is as follows:
\begin{equation}\label{eq:loss_pgr1}
\begin{aligned}
    L_D &= \mathbb{E}_{p_{\mathrm{r}}(x)}\left[{D}(x)\right]
    -\mathbb{E}_{p_{\mathrm{\uptheta}}(x)}\left[{D}\left(x\right)\right], \\
    L_G &= \mathbb{E}_{p_{\mathrm{\uptheta}}(x)}\left[{D}\left(x\right)\right]
    -\lambda_g\frac{1}{|\mathrm{S}^{'}|}\sum_{s\in \mathrm{S}^{'}}\left[h\left(s\right)\right],
\end{aligned}
\end{equation}
where $h(s)$ is Eq.~\eqref{eq:margin}. $\mathrm{S}^{'}$ is the pairwise preferences constructed between the generated data and the undesired data, i.e., $\mathrm{S}^{'}=\big\{s=(x_1, x_2)|x_1 \succ x_2, x_1 \sim p_{\mathrm{\uptheta}}(x), x_2 \sim p_{\mathrm{u}}(x)\big\}$. Now the generator consists of two terms, the original WGAN loss on the generator aims to achieve $\mathbb{E}_{p_{\mathrm{\uptheta}}(x)}\left[{D}\left(x\right)\right]>\mathbb{E}_{p_{\mathrm{r}}(x)}\left[{D}(x)\right]$, while the regularization aims to achieve $\mathbb{E}_{p_{\mathrm{\uptheta}}(x)}\left[{D}\left(x\right)\right]>\mathbb{E}_{p_{\mathrm{u}}(x)}\left[{D}(x)\right]$. Since the undesired data is a subset of the real data, i.e., $\{x|x \sim p_{\mathrm{u}}(x)\} \subseteq	\{x|x \sim p_{\mathrm{r}}(x)\}$, the WGAN loss always dominates the training of the generator. Therefore, \mbox{PRG-1} degenerates to WGAN.

The objective for PRG-2 is as follows:
\begin{align}\label{eq:loss_pgr2}
    L_D &= \mathbb{E}_{p_{\mathrm{r}}(x)}\left[{D}(x)\right]
    -\mathbb{E}_{p_{\mathrm{\uptheta}}(x)}\left[{D}\left(x\right)\right]-\lambda\frac{1}{|\mathrm{S}|}\sum_{s\in \mathrm{S}}\left[h\left(s\right)\right], \nonumber \\
    L_G &= \mathbb{E}_{p_{\mathrm{\uptheta}}(x)}\left[{D}\left(x\right)\right]
    -\lambda_{g}\frac{1}{|\mathrm{S}^{'}|}\sum_{s\in \mathrm{S}^{'}}\left[h\left(s\right)\right],
\end{align}
where $\mathrm{S}$ is constructed based on~\eqref{eq:s}.
Although the generator consists of two terms, the same as our analysis about PRG-1, the extra pairwise regularization on the generator is invalid. Meanwhile, the extra pairwise regularization on the critic works like that in DiCGAN. Therefore, the whole framework degenerates to DiCGAN.

\begin{figure*}[tb]
    \centering
    \begin{subfigure}[b]{.29\linewidth}
    \centering
    \includegraphics[width=.82\linewidth]{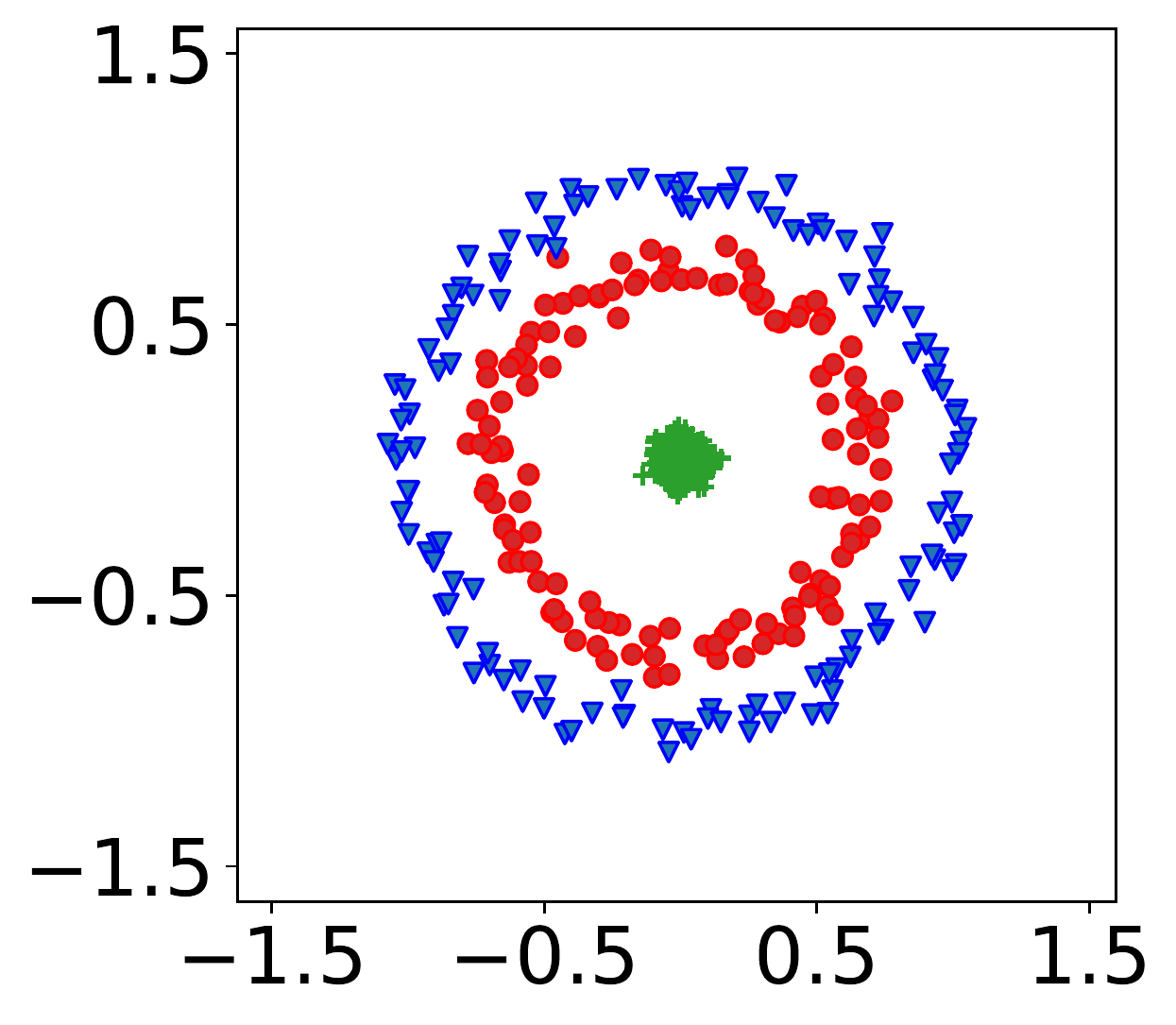}
    \caption{Data}
    \label{fig:circles}
  \end{subfigure}
  \begin{subfigure}[b]{.325\linewidth}
  \centering
    \includegraphics[width=.8\linewidth]{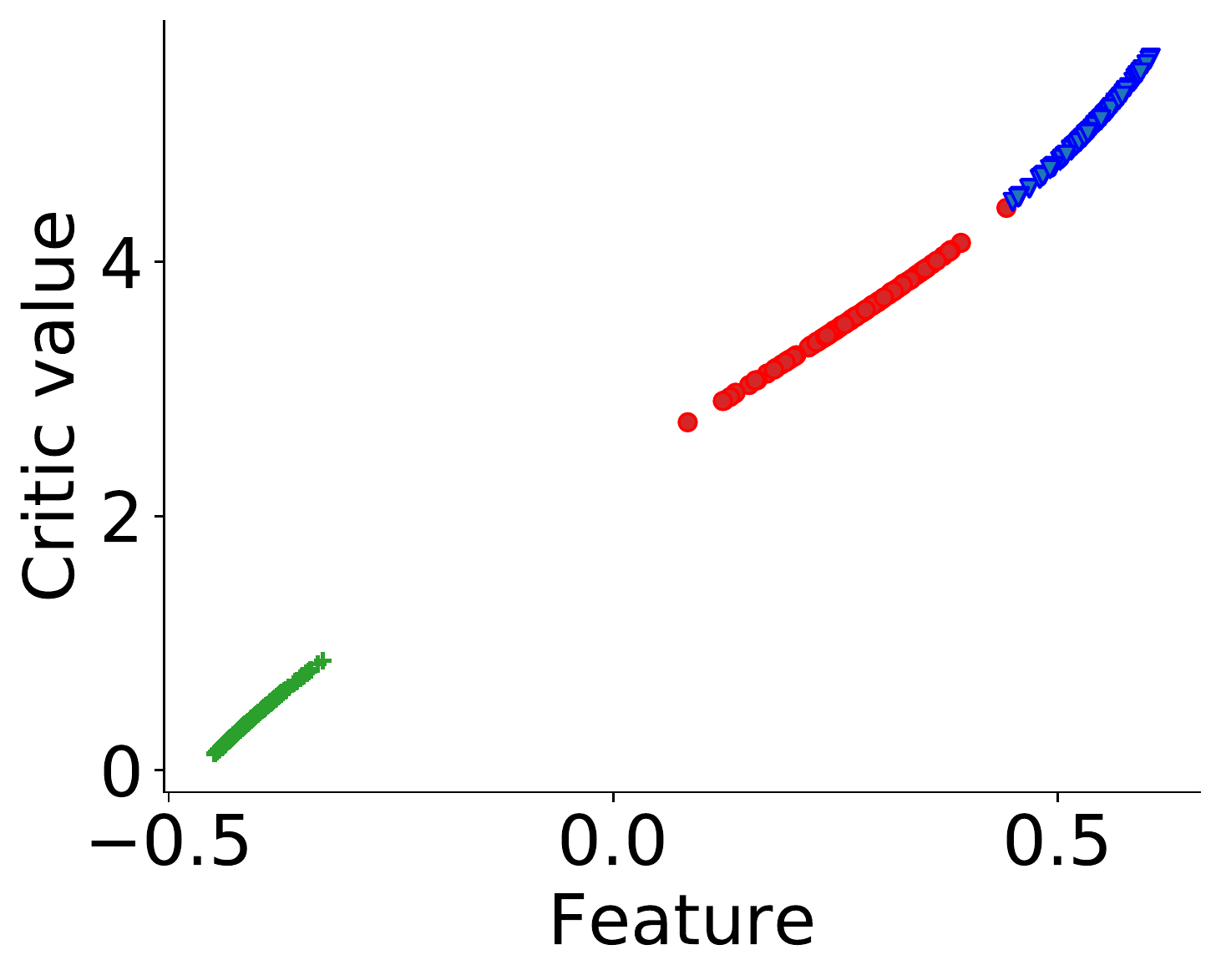}
    \caption{WGAN}
    \label{fig:wgan_circles_fixed_g}
  \end{subfigure}
  \begin{subfigure}[b]{.345\linewidth}
  \centering
    \includegraphics[width=.8\linewidth]{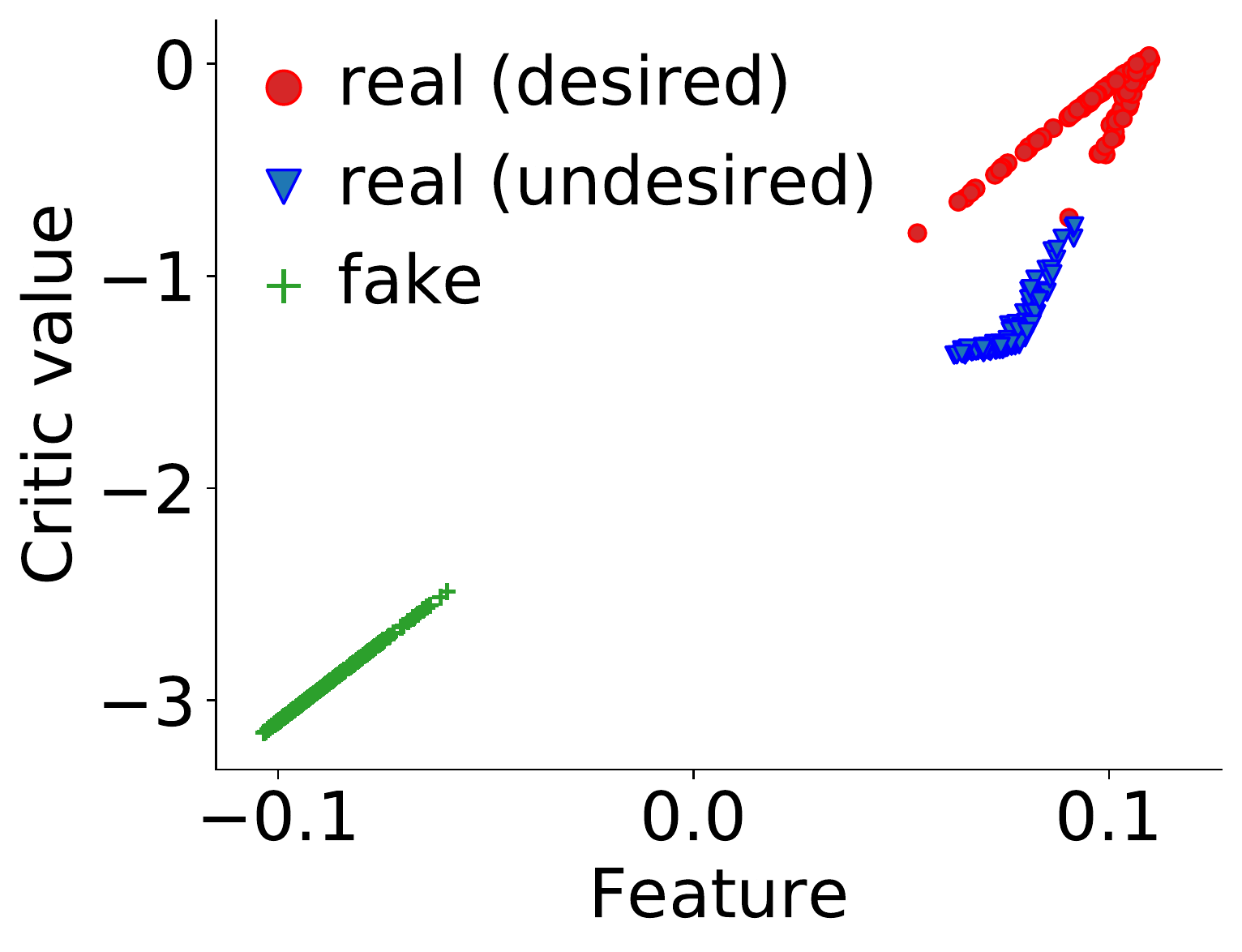}
    \caption{DiCGAN}
    \label{fig:dicgan_circles_fixed_g}
  \end{subfigure} \vskip-0.05in
  \caption{ \label{fig:compare_D} Comparison of the critic in (b) WGAN and (c) DiCGAN. DiCGAN's critic can assign higher critic values for real desired data than real undesired data while WGAN's critic cannot. ``Feature'' is obtained by using kernel PCA to project the output on the second last layer of the critic into 1D space.}\vskip-0.1in
\end{figure*}

\begin{figure*}[!htb]
    \centering
    \begin{subfigure}[b]{\linewidth} \hskip0.8in
    \includegraphics[width=.4\linewidth]{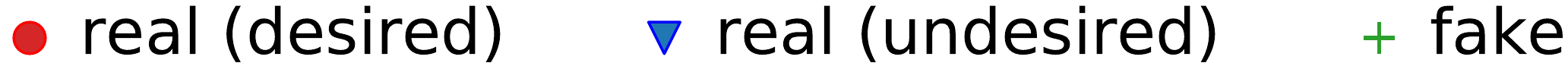}
  \end{subfigure}    \vskip-0.1in
    \begin{subfigure}[b]{.29\linewidth}
  \centering
    \includegraphics[width=.82\linewidth]{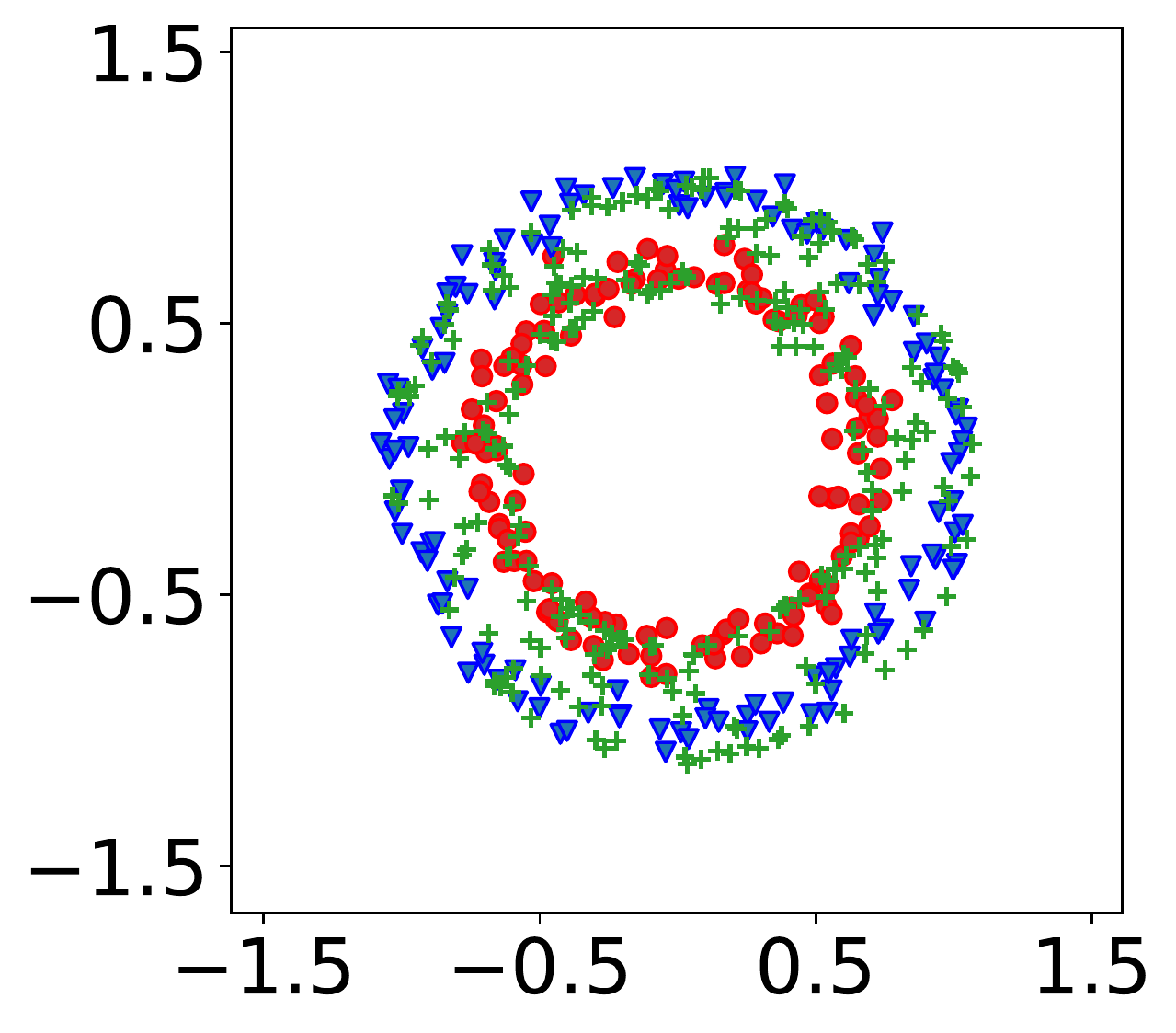}
    \caption{WGAN\label{fig:wgan_gen}}
    \label{fig:wgan_circle_samples}
  \end{subfigure}
  \begin{subfigure}[b]{.29\linewidth}
  \centering
    \includegraphics[width=.82\linewidth]{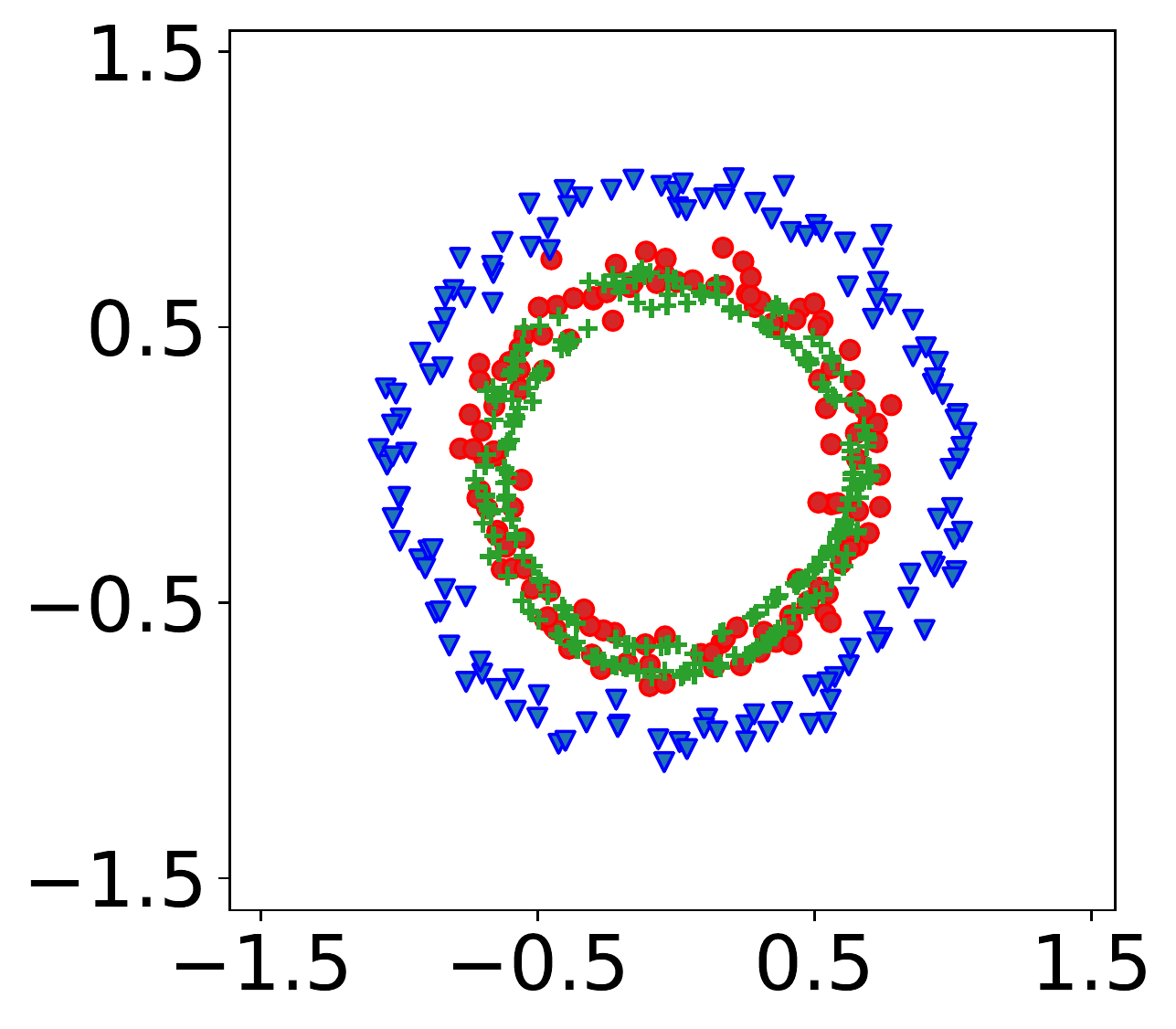}
    \caption{DiCGAN\label{fig:dicgan_gen}}
    \label{fig:dicgan_circle_samples}
  \end{subfigure}
   \begin{subfigure}[b]{.36\linewidth}
  \centering
    \includegraphics[width=.8\linewidth]{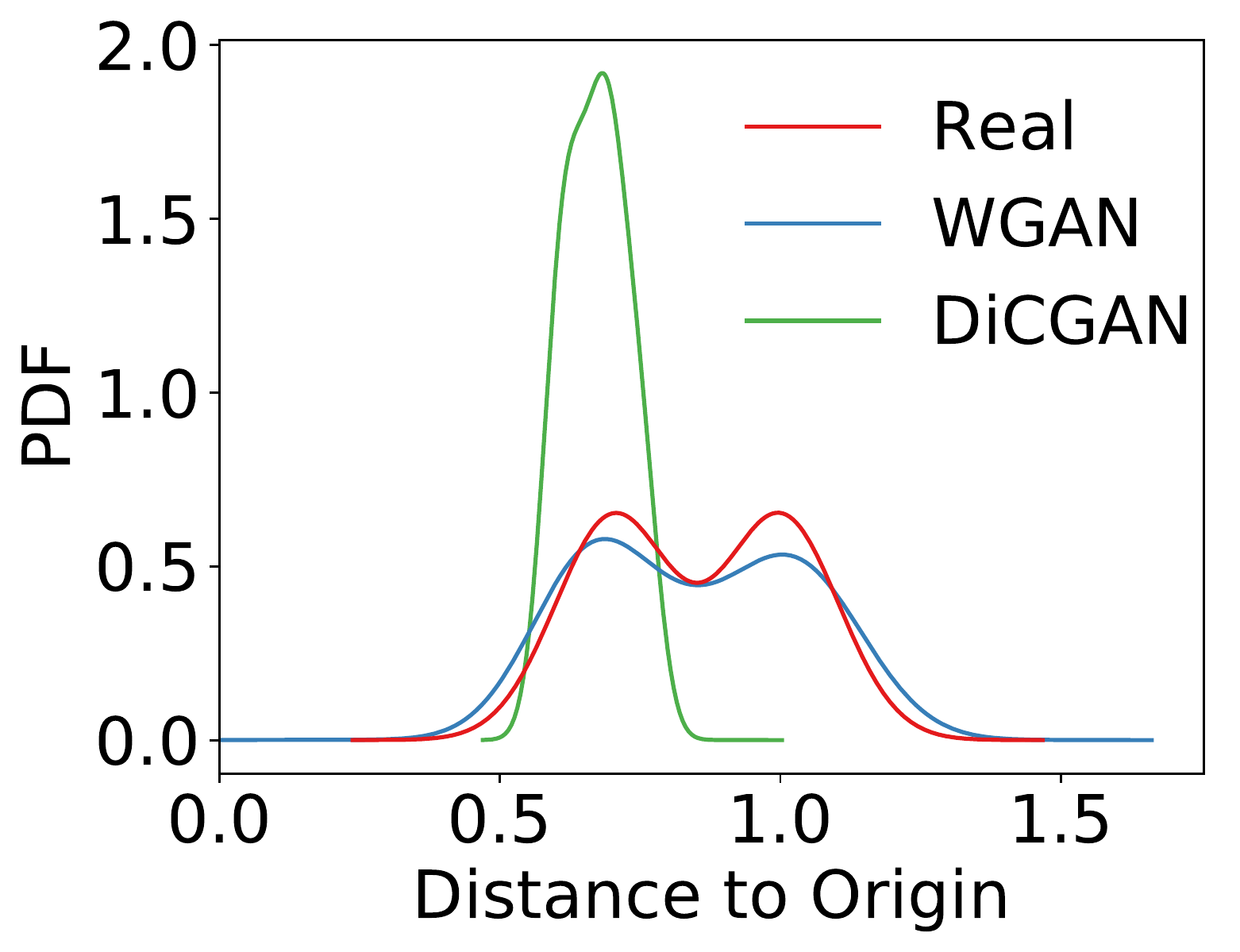}
    \caption{\label{fig:circle_dist_pdf}PDF vs. distance}
  \end{subfigure}\vskip-0.05in
  \caption{\label{fig:compare_G} (a-b) Visualization of the generated samples from WGAN and DiCGAN. The fake data is expected to overlap with the real desired data only. (c) Probability density function (PDF) vs.\mbox{ sample distance to the origin.}}\vskip-0.1in
\end{figure*}

\section{Case Study on Synthetic Data}
\label{sect:case study}
To gain an intuitive understanding of the differences between our DiCGAN and WGAN regarding the critic and the generator, we conduct a case study on a synthetic dataset.

The synthetic dataset consists of two concentric circles by adding Gaussian noise with a standard deviation of 0.05, which is a 2D mixture Gaussian distribution with two modes (See Fig.~\ref{fig:circles}). The samples located on the inner circle are considered to be the desired data, while the samples on the outer circle are defined as the undesired data. 
By labeling the desired data as $y=1$ and the undesired data as $y=0$,  we can construct the pairwise preference for two samples $x_1$ and $x_2$ based on their labels. Namely, $x_1\succ x_2$ if $y_1=1 \wedge y_2=0$, and vice versa. The pairs are constructed within each mini-batch. 
Our target is to learn the distribution of the desired data (i.e., samples on the inner circle), using the whole data along with the constructed pairwise preferences.

\subsection{WGAN vs DiCGAN on Critic}
Experiment setting: we fix the generator and simulate the fake data as the 2D Gaussian blob with a standard deviation of 0.05 (green pluses). We first train the critic until convergence. Then, we project the output on the second last layer of the critic into 1D space using kernel principal components analysis (PCA), to obtain the projected features. 
To explore the difference between the critics of WGAN and DiCGAN, we draw the curve of the critic values versus the projected features for WGAN and DiCGAN, respectively (Fig.~\ref{fig:wgan_circles_fixed_g},~\ref{fig:dicgan_circles_fixed_g}).

From Fig.~\ref{fig:wgan_circles_fixed_g},~\ref{fig:dicgan_circles_fixed_g}, we can see: $(1)$ in terms of the real data and the fake data, the critic of both WGAN and DiCGAN can achieve perfect discrimination. Meanwhile, the projected features of the real data and those of the fake data are also completely separated; $(2)$ in terms of the real desired data and the real undesired data, the critic of DiCGAN assigns higher values to the desired samples, compared to the undesired samples. This is because our ranking loss expects a higher ranking score (i.e., critic value) for the desired sample. $(3)$ In contrast, the critic of WGAN assigns lower values to the desired data since the desired data is closer to the fake data compared to the undesired data.

\subsection{WGAN vs DiCGAN on Generator}
\label{sec:comp_generator}
Experiment setting: we train the critic and the generator following the regular GANs' training procedure. The generation results of WGAN and DiCGAN are shown in Fig.~\ref{fig:wgan_gen}, ~\ref{fig:dicgan_gen}.

DiCGAN (shown in Fig.~\ref{fig:dicgan_gen}) only generates the user-desired data. Namely, generated data covers the inner circle. In contrast, WGAN (shown in Fig.~\ref{fig:wgan_gen}) generates all data. Namely, generated data covers the inner circle and the outer circle. 
As the critic in DiCGAN can guide the fake data towards the real data region and away from the undesired data region, the generator thus produces data that is similar to the real desired data. 
Because the critic in WGAN pushes the fake data to the region of all real data, the generator \mbox{finally produces the whole real-alike data.}

Further, we calculate the distance from the real samples to the origin and plot the probability density function versus the distance in Fig.~\ref{fig:circle_dist_pdf}. We also do this for the generated samples from WGAN and DiCGAN, respectively. It shows that DiCGAN only captures one mode of the real data distribution, consistent with the results that DiCGAN only produces desired samples. In contrast, WGAN captures all modes of the real data distribution, meaning that WGAN generates all real data.

\section{Experimental Study}
\label{sect:experiment}
Our DiCGAN for desired data generation has various applications in the real world. In particular, we apply our DiCGAN to two applications: 1) generating images that meet the user’s interest for a given dataset, which can be used for image search~\cite{49290}.
2) optimizing biological products with desired properties, which can automate the process of designing DNA sequences for usage in medicine and manufacturing~\cite{gupta2019feedback}.
In these applications, we verify that our DiCGAN only using local knowledge (i.e., user preferences) outperforms current methods relying on global knowledge when labels of desired data are limited.
Furthermore, we study the relation between critic values and user preferences as well as the effects of each component in DiCGAN.

\begin{figure*}[!t]
    \centering
  \begin{subfigure}[t]{.18\linewidth}
    \centering
    \includegraphics[width=\linewidth]{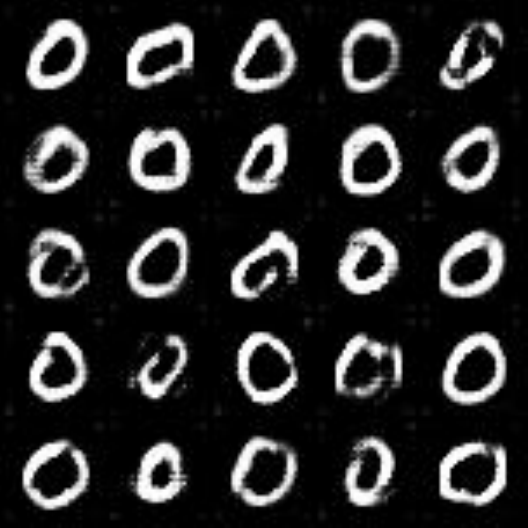}
    \caption{WGAN ($25/25$)}
  \end{subfigure}
  \begin{subfigure}[t]{.18\linewidth}
  \centering
    \includegraphics[width=\linewidth]{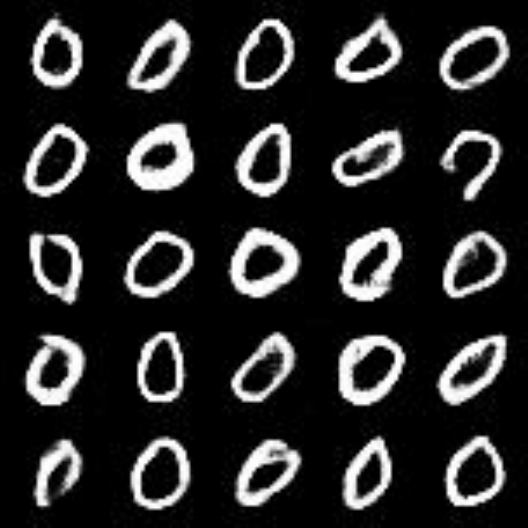}
    \caption{CWGAN ($25/25$)}
  \end{subfigure}
  \begin{subfigure}[t]{.18\linewidth}
  \centering
    \includegraphics[width=\linewidth]{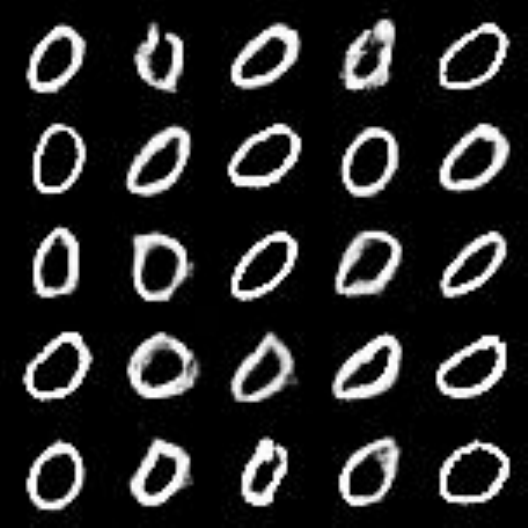}
    \caption{FBGAN ($25/25$)}
  \end{subfigure}
  \begin{subfigure}[t]{.18\linewidth}
  \centering
    \includegraphics[width=\linewidth]{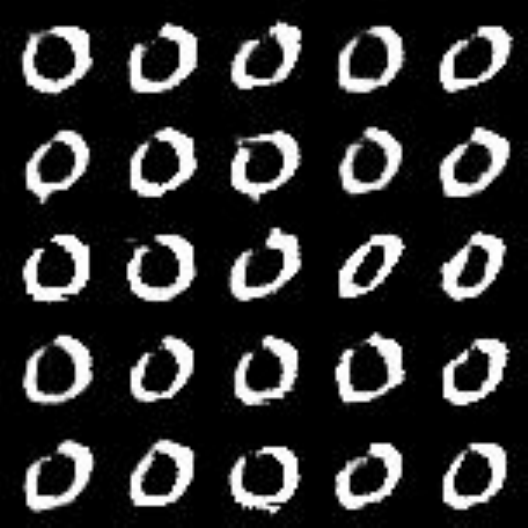}
    \caption{\revision{GAN-FT ($25/25$)}\label{fig:gan-ft_mnist}}
  \end{subfigure}
  \begin{subfigure}[t]{.18\linewidth}
  \centering
    \includegraphics[width=\linewidth]{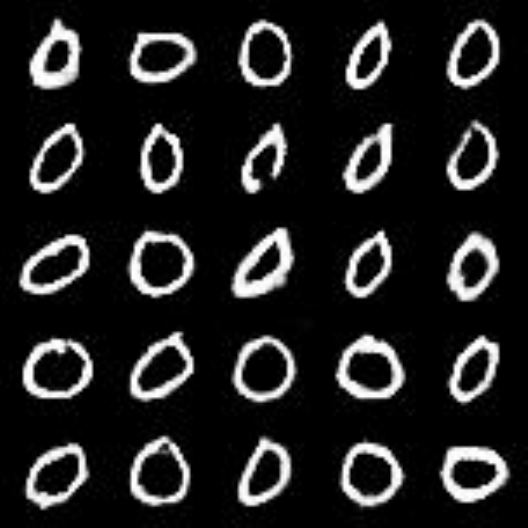}
    \caption{DiCGAN ($25/25$)}
  \end{subfigure} \vskip-0.05in
    \caption{\label{fig:mnist_baselines} Generated images on MNIST by (a) WGAN, (b) CWGAN, (c) FBGAN, \revision{(d) GAN-FT} and (e) DiCGAN}.\vskip-0.1in
\end{figure*}

\begin{figure*}[!t]
    \centering
    \begin{subfigure}[t]{.18\linewidth}
  \centering
    \includegraphics[width=\linewidth]{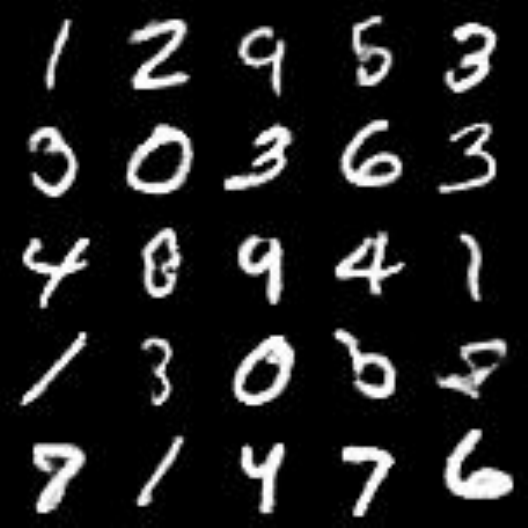}
    \caption*{Iter 0 ($1.0\%$)}
  \end{subfigure}
  \begin{subfigure}[t]{.18\linewidth}
    \includegraphics[width=\linewidth]{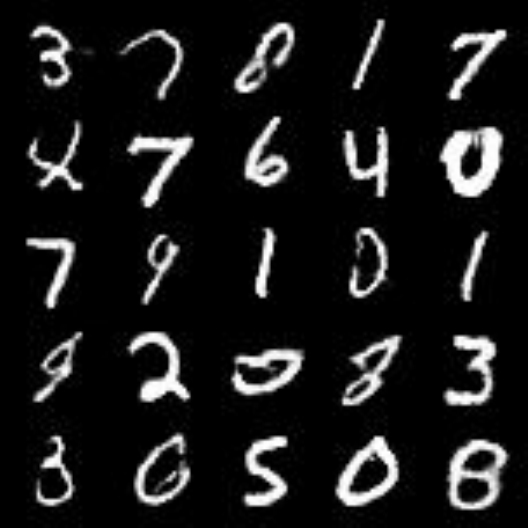}
    \caption*{Iter 200 ($16.1\%$)}
  \end{subfigure}
  \begin{subfigure}[t]{.18\linewidth}
  \centering
    \includegraphics[width=\linewidth]{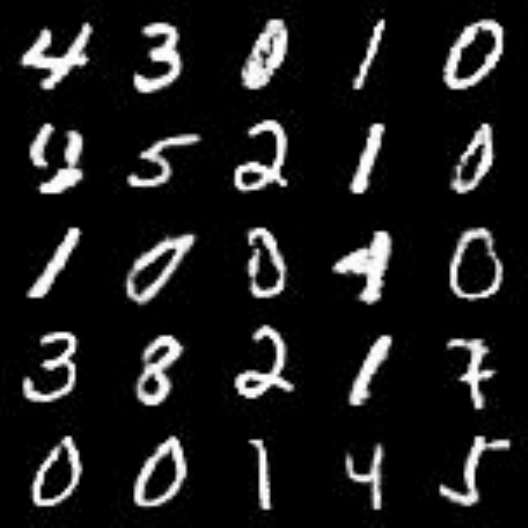}
    \caption*{Iter 400 ($28.2\%$)}
  \end{subfigure}
  \begin{subfigure}[t]{.18\linewidth}
    \includegraphics[width=\linewidth]{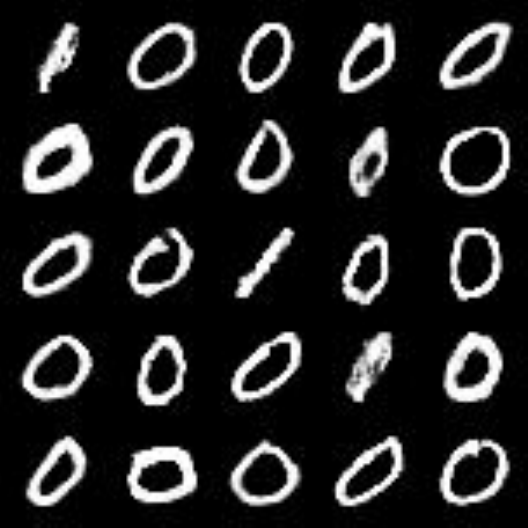}
    \caption*{Iter 1000 ($95.8\%$)}
  \end{subfigure}
  \begin{subfigure}[t]{.18\linewidth}
    \includegraphics[width=\linewidth]{Fig_mnist_dicgan.pdf}
    \caption*{Iter 2000 ($99.9\%$)}
  \end{subfigure} \vskip-0.05in
    \caption{\label{fig:mnist_samples}  Generated images of DiCGAN on MNIST during the training process. DiCGAN learns the distribution of small digits, which gradually generates more small digit images. The \% denotes the percentage of zero digits in $50K$ generated samples.}\vskip-0.1in
\end{figure*}
\textbf{Baselines} We compare DiCGAN with WGAN~\cite{pmlr-v70-arjovsky17a}, CWGAN~\cite{mirza2014conditional}, FBGAN~\cite{gupta2019feedback} and \revision{GAN-FT}. 1) WGAN is trained with only the desired data to derive the desired data distribution. 2) CWGAN is the extension of GAN with a conditional label $c$. To train CWGAN, we split the training data into the desired class ($c=1$) and the undesired class ($c=0$) based on {global knowledge}.  Then $p(x|c=1)$ is the desired data distribution. 3) FBGAN adopts an iterative training paradigm to derive the desired data distribution. First, FBGAN is pre-trained with all training data. At each training epoch, FBGAN resorts to an extra selector to select the desired samples from the generated samples and use them to replace the least-recently added samples in the training dataset. Then FBGAN performs regular GAN training with the updated training data. \revision{4) GAN-FT is to fine-tune a pre-trained GAN with a classification loss on desired data. It is possible to use GAN loss defined between the generated data and the desired data to constrain the quality of desired data during the fine-tuning of GAN-FT. This is actually similar to the baseline WGAN that is trained on the desired subset of training data. Thus it would still suffer from poor data quality issues when there is limited desired data in the training dataset.} 

\textbf{Datasets} MNIST~\cite{726791} consists of $28 \times 28$ images with digit zero to nine. $50K$ training images are regarded as training data. CelebA-HQ~\cite{DBLP:conf/iclr/KarrasALL18} is the high-quality subset of Celeb Faces Attributes Dataset, which has $30K$ face images of celebrities. We use all images as the training data and resize them to $64 \times 64$. The gene sequence dataset~\cite{gupta2019feedback} contains $3,655$ gene sequences with a maximum length of 156 codings for proteins collected from the Uniprot database. All methods applied to the datasets use the same supervision for a fair comparison. On MNIST and CelebA-HQ, we resort to class labels to derive the desired data distribution. On the gene sequence dataset, we resort to an analyzer that can evaluate the desired property for genes to derive the desired data distribution. 
\begin{remark}
Considering pairwise preferences over explicitly labeling what the user considers to be good data or not is beneficial especially given the limited supervision, which will be verified in the following experiments.
\end{remark}

\textbf{Evaluation Metric:} To evaluate the performance of learning the desired data distribution, we calculate the percentage of desired data (PDD) in GAN's generation. $\text{PDD} = \frac{|\{x|x \text{ is desired}, x \in \mathrm{X_g}\}|}{|\mathrm{X_g}|} \times 100\%$, 
where $\mathrm{X_g}$ are generated samples.

\subsection{Capturing Small Digits on MNIST}
\label{sec:mnist}
Suppose the user is interested in learning the distribution of small digits on MNIST. Zero is the smallest digit of MNIST, thus as the desired data.

\textbf{Networks \& Hyperparameters} \revision{By a coarse grid search, the balance factor $\lambda$ is set to 1. The ranking margin $m$ is set to $1$ following~\cite{cao2006adapting}.} The batch size $b$ is set to~$50$. The network architecture of the critic and generator in our DiCGAN are based on WGAN-GP~\cite{gulrajani2017improved}. See Supplementary for details. The baselines share the same architecture for a fair comparison. The optimizer is Adam~\cite{DBLP:journals/corr/KingmaB14} with a learning rate of $1e\textrm{-}4$ and $\beta_1=0.5, \beta_2=0.9$. The number of critic \mbox{iterations per generator iteration $n_{critic}$ is $5$.}

\begin{table}[!b]\vskip-0.1in
\centering
\caption{\label{tb:mnist_percentage} Percentage of desired data in the generation (PDD) of various GANs on MNIST. Best results
are highlighted in bold. Top $1$ means digit zero. Top $5$ means digits zero to four.}
	\renewcommand{\arraystretch}{1.}
	\setlength{\tabcolsep}{0.7mm}{	
		\scalebox{1.}{
\begin{tabular}{ccccccc}
\toprule[1.3pt]
Method & Original & WGAN  & CWGAN & FBGAN & \revision{GAN-FT} & DiCGAN \\ \hline
Top~1  & 9.9      & 97.3  & 95.0  & \textbf{100.0} & \revision{\textbf{100.0}} & \textbf{100.0}  \\
Top~5  & 51.1     & 98.2  & 96.4  & \textbf{100.0} & \revision{\textbf{100.0}} & \textbf{100.0}  \\
\toprule[1.3pt]
\end{tabular}}}
\end{table}

\textbf{Training} As for WGAN and CWGAN, zero digits in the training data are regarded as the desired samples ($c=1$), whose size is $4,950$. The other digits are labeled as the undesired samples, whose size is $45,050$ ($c=0$). WGAN is only trained with the desired data. CWGAN conditions on $c$ to model a conditional data distribution $p(x|c)$ for MNIST. 
\revision{For GAN-FT, we first pre-trained WGAN-GP with all digit images. Then we fine-tuned its generator with a classifier loss that makes the generated samples classified as digit zero.}
FBGAN and our DiCGAN both introduce the generated samples into the training dataset during the training. The labels of the generated samples are obtained by resorting to a classifier, pre-trained for digit classification. At every training epoch, FBGAN generates $50K$ samples and requests the classifier to label them. Then the selector in FBGAN will rank the images using their corresponding labels, where the smaller digits are ranked higher. The selector selects the generated images with digits ranked in the top $50\%$, i.e., small digits, as the desired data to replace old training data. 
As for DiCGAN, the pairwise comparison can be obtained for two images $x_1$ and $x_2$ according to their predicated label $y_1$ and $y_2$, namely $x_1 \succ x_2$ if $y_1<y_2$, and vice versa.  At each iteration, \#pairwise preferences $n_s$ is $25$. \#iteration per minor correction $n_{\mathrm{i}}=200$. \#generated samples for each minor correction $n_{\mathrm{g}}=50K$.

\begin{figure}[!t]
    \centering
    \begin{subfigure}[t]{\linewidth}
  \centering
    \includegraphics[width=.62\linewidth]{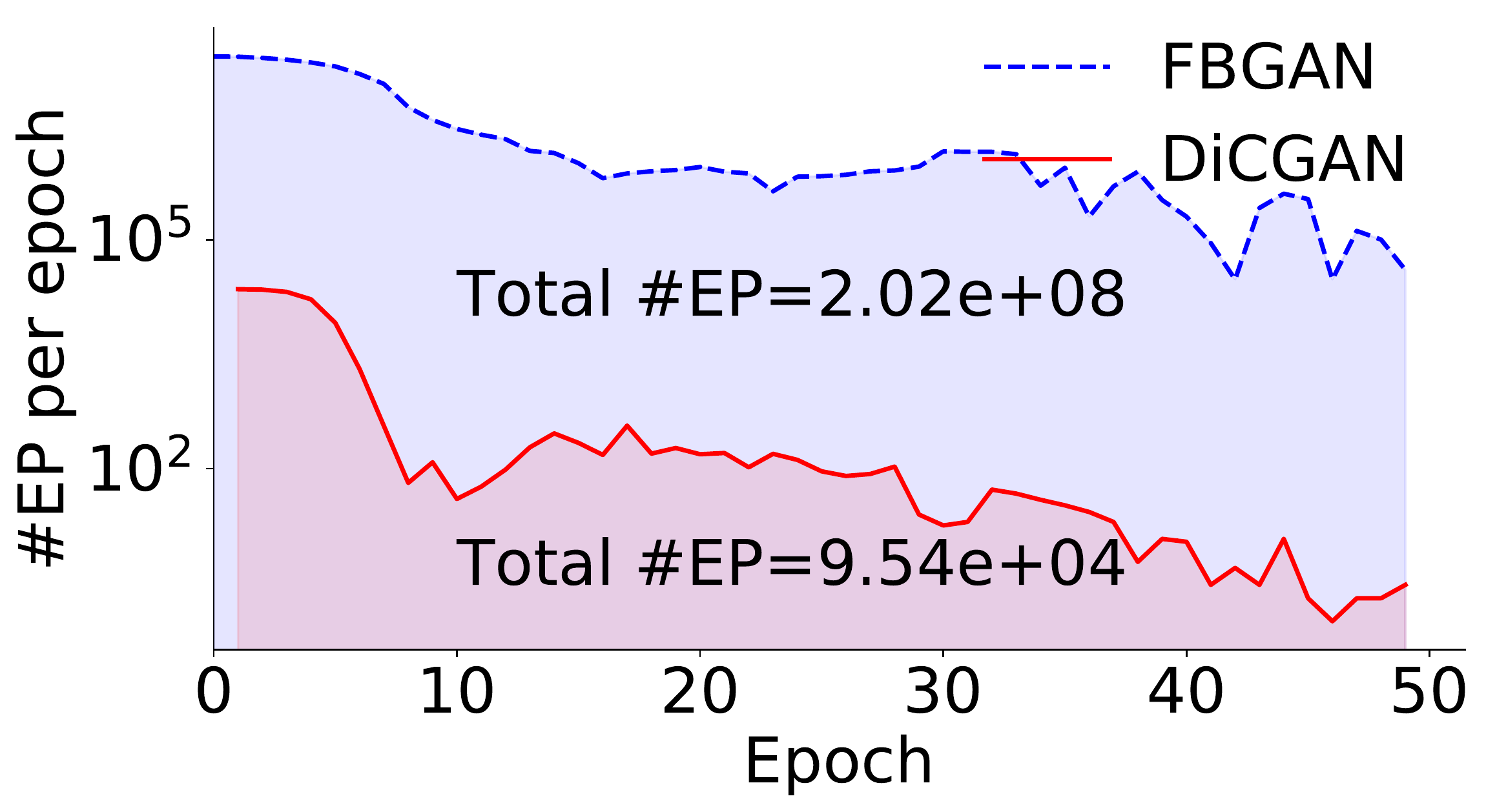}
    \caption{\label{fig:mnist_fbgan_dicgan_nep}}
  \end{subfigure}
    \begin{subfigure}[t]{.43\linewidth}
  \centering
    \includegraphics[width=\linewidth]{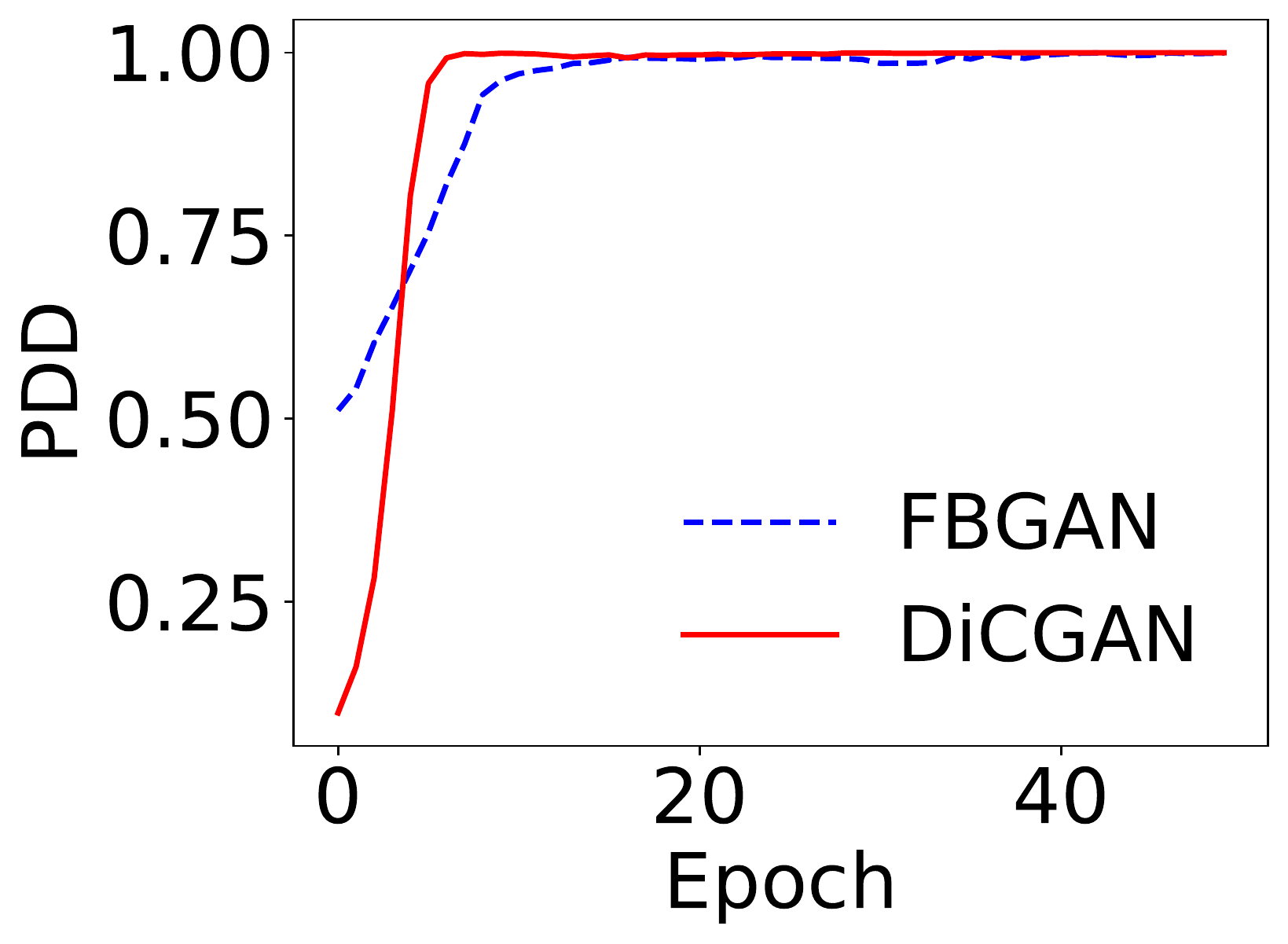}
    \caption{\label{fig:mnist_fbgan_dicgan_pdd}}
  \end{subfigure}
  \begin{subfigure}[t]{.515\linewidth}
  \centering
    \includegraphics[width=\linewidth]{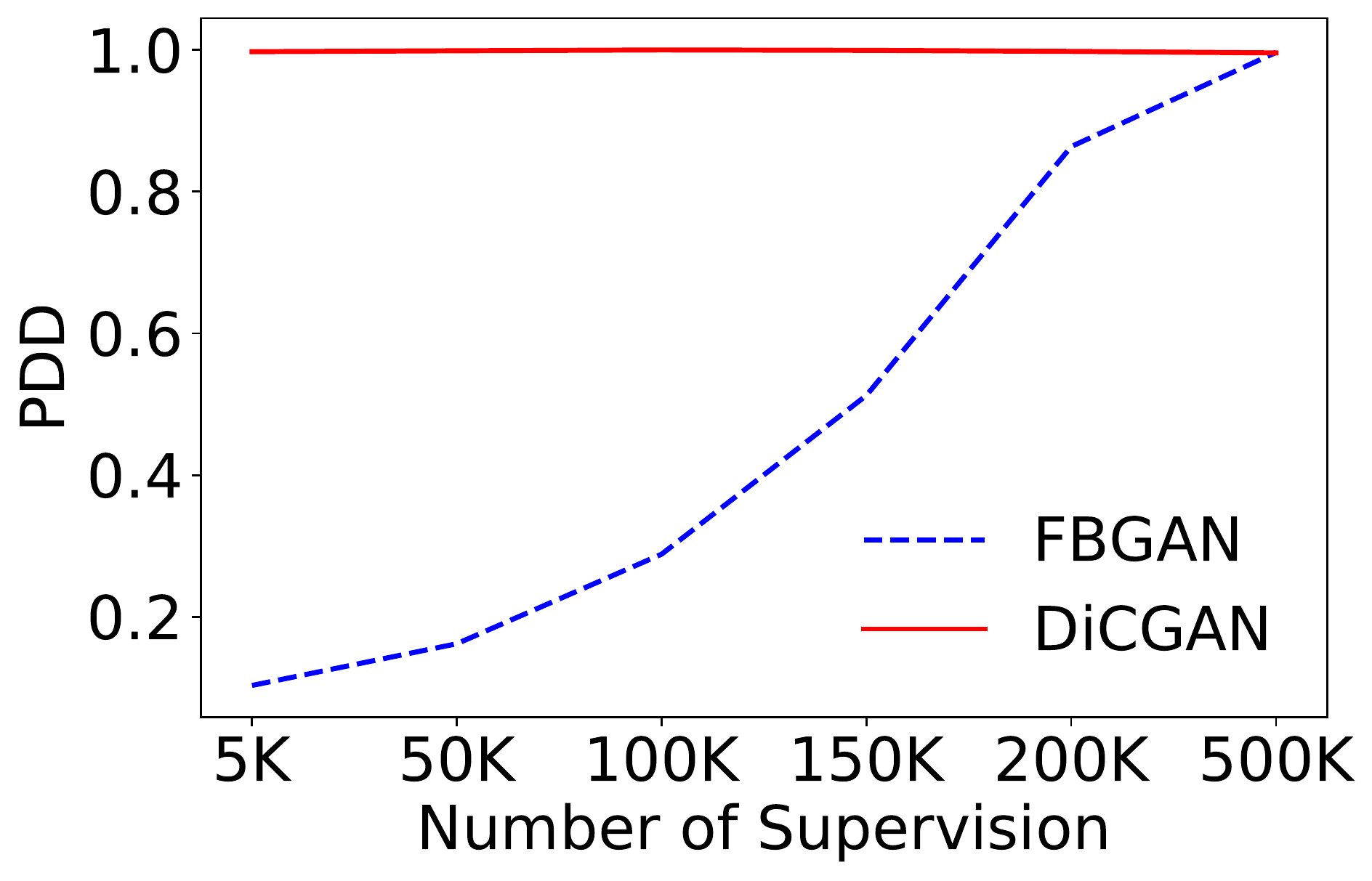}
    \caption{\label{fig:dw_ns}}
  \end{subfigure} \vskip-0.05in
    \caption{\label{fig:more_comp} Comparison of DiCGAN and FBGAN on MNIST. (a) plots used \#EP per epoch. (b) plots PDD versus the training epoch. (c) plots PDD versus the number of supervision.}\vskip-0.1in
\end{figure}

Fig.~\ref{fig:mnist_samples} presents the generated MNIST images randomly sampled from the generator of DiCGAN. It shows that the generated MNIST digits gradually shift to smaller digits during the training, and converge to the digit zero. 
For each method, we sample $50K$ samples from the generator and calculate the percentage of digit zero and digits zero to four among the generated digits for quantitative evaluation. In Table~\ref{tb:mnist_percentage}, only small digits are generated by DiCGAN and FBGAN; WGAN and CWGAN can also learn the distribution of the desired digit since the dataset is simple and has relatively sufficient data for the desired digit. The visual results shown in Fig.~\ref{fig:mnist_baselines} are consistent with the quantitative results. 
However, when the dataset is complex and the desired data is insufficient, WGAN and CWGAN fail, which is described in Sect.~\ref{sec:celebahq}. \revision{GAN-FT also only generates digit zero, but it suffers from mode collapse problem. The generated images have low diversity (Fig.~\ref{fig:gan-ft_mnist}). This is because there lacks data quality guarantee during the later fine-tuning stage.}

\subsubsection{Comparison of DiCGAN and FBGAN}

Though FBGAN achieves good performance in learning the desired data distribution, it requires a lot of supervision information from the selector. We calculate the number of effective pairs (\#EP) used in DiCGAN and FBGAN, respectively. 
\#EP in DiCGAN denotes the total number of explicitly constructed pairs during the training, i.e., $\text{\#EP}=\sum_{i=1}^{n_{\mathrm{e}}}\sum_{j=1}^{n_{\mathrm{i}}}n_{\mathrm{s}}$. As for FBGAN, its selector ranks all generated samples and selects the desired samples from them at each epoch. Therefore, \#EP can be induced by the implicit pairs implied by the desired generated samples versus the undesired generated samples, i.e., $\text{\#EP}=\sum_{i=1}^{n_{\mathrm{e}}}n_{\mathrm{gd}}\times n_{\mathrm{gu}}$, where $n_{\mathrm{e}}$ is the number of training epochs. 
where $n_{\mathrm{gd}}$ and $n_{\mathrm{gu}}$ denote the number of desired samples and undesired samples in the generation, respectively. 

Fig.~\ref{fig:mnist_fbgan_dicgan_nep} plots FBGAN's and DiCGAN's used \#EP at each epoch, respectively. It shows that (1) the \#EP used in DiCGAN is much smaller than that in FBGAN at each training epoch; (2) the total \#EP used in DiCGAN is significantly less than that in FBGAN, which can be reflected from the shadow area. In total, DiCGAN used $9.53e4$ effective pairs while FBGAN used $2.02e8$ effective pairs. Our DiCGAN is scalable to the large training dataset, e.g. MNIST. \#EP in DiCGAN is linearly correlated to the training size. In contrast, \#EP in FBGAN is determined by $n_\mathrm{gd}$ and $n_\mathrm{gu}$, which are both linearly correlated to the training size. Thus, \#EP in FBGAN is quadratically correlated to the training size.

We plot the ratio of digit zero in the generated data (PDD) of DiCGAN and FBGAN during the training process in Fig.~\ref{fig:mnist_fbgan_dicgan_pdd}. It shows that DiCGAN converges faster than FBGAN.

\subsubsection{Comparing DiCGAN and FBGAN given the limited supervision}
\label{sec:limit_sup}
We conduct the experiment on MNIST. Specifically, the query amount of resorting to the pre-trained classifier to obtain the prediction of the generated samples is restricted to $5K$ for both FBGAN and DiCGAN. 

\begin{figure}[!t]
    \centering
 \begin{minipage}[b]{0.24\textwidth}
  \centering
    \includegraphics[width=.8\linewidth]{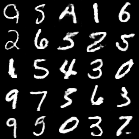}
    \label{fig:mnist_limit_fbgan}
    \centerline{(a) FBGAN}
  \end{minipage}
  \begin{minipage}[b]{0.24\textwidth}
  \centering
    \includegraphics[width=.8\linewidth]{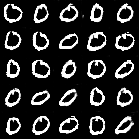}
    \label{fig:mnist_limit_dicgan}
   \centerline{(b) DiCGAN}
  \end{minipage}
  \caption{\label{fig:mnist_limit} The generated results of {(a) FBGAN and (b) DiCGAN} on MNIST given limited supervision.}\vskip-0.1in
\end{figure}

Table~\ref{tb:mnist_limit_pdd} shows that DiCGAN can learn the desired data distribution, generating $99.7\%$ zero digits, while FBGAN fails, generating $10.3\%$ digit zero, which is consistent with the visual results in Fig.~\ref{fig:mnist_limit}a and Fig.~\ref{fig:mnist_limit}b.

\begin{table}[!hb]
\arrayrulecolor{black}
\centering
\caption{\label{tb:mnist_limit_pdd}\mbox{PDD on MNIST given limited supervision.}}
  \renewcommand{\arraystretch}{1.}
	\setlength{\tabcolsep}{1.2mm}{	
		\scalebox{1.}{
\begin{tabular}{ccc}
\toprule[1.3pt]
Method                  &  Top~1         & Top~5        \\ \hline
FBGAN                   &  10.3          & 52.6         \\
DiCGAN                  &  \textbf{99.7} & \textbf{99.9} \\
\toprule[1.3pt]
\end{tabular}}}\vskip-.1in
\end{table}

\begin{figure*}[!ht]
    \centering
    \begin{subfigure}[t]{.15\linewidth}
  \centering
    \includegraphics[width=\linewidth]{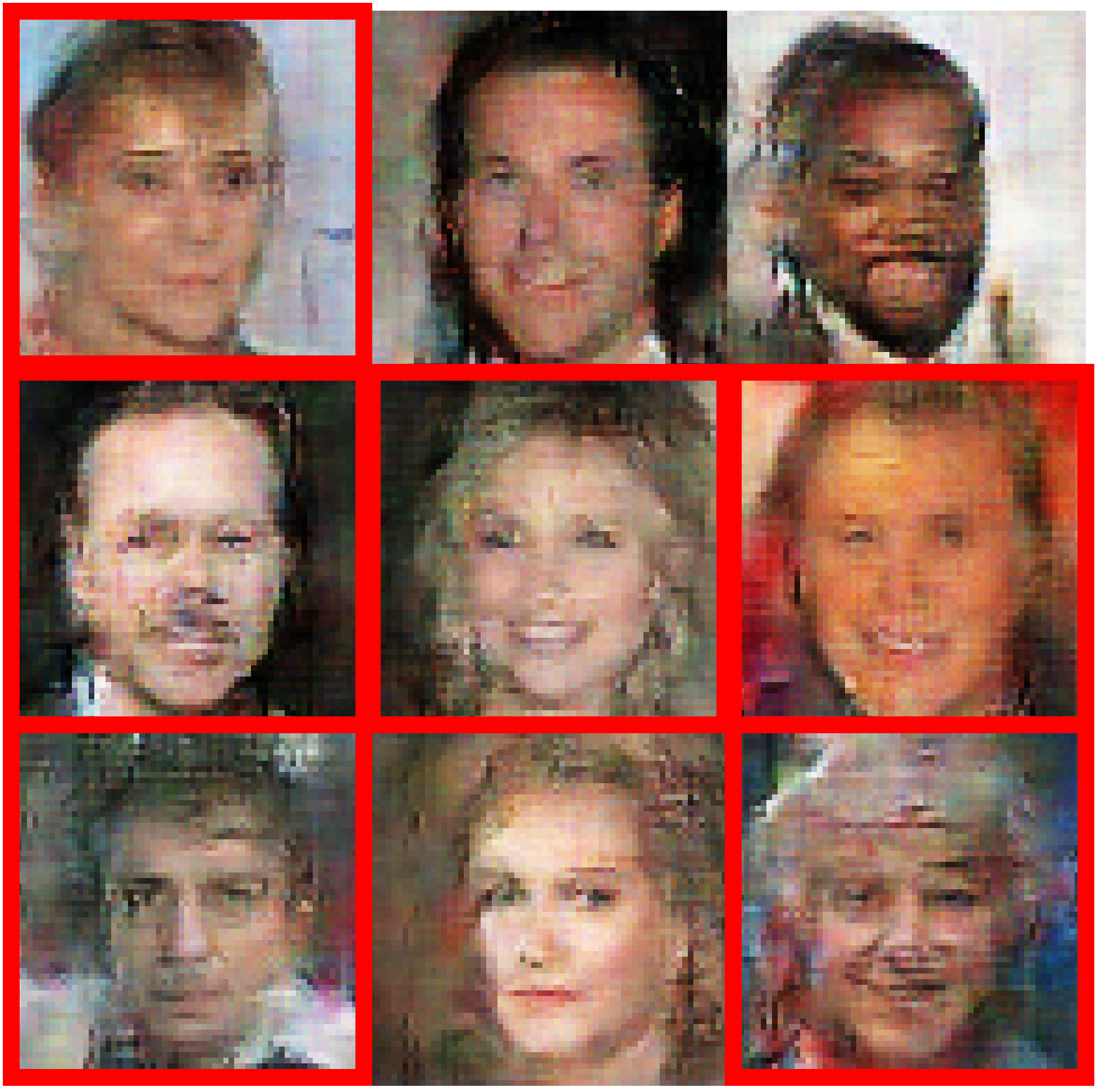}
    \caption{WGAN ($6/9$)\label{fig:wI}}
  \end{subfigure}
  \begin{subfigure}[t]{.16\linewidth}
  \centering
    \includegraphics[width=.925\linewidth]{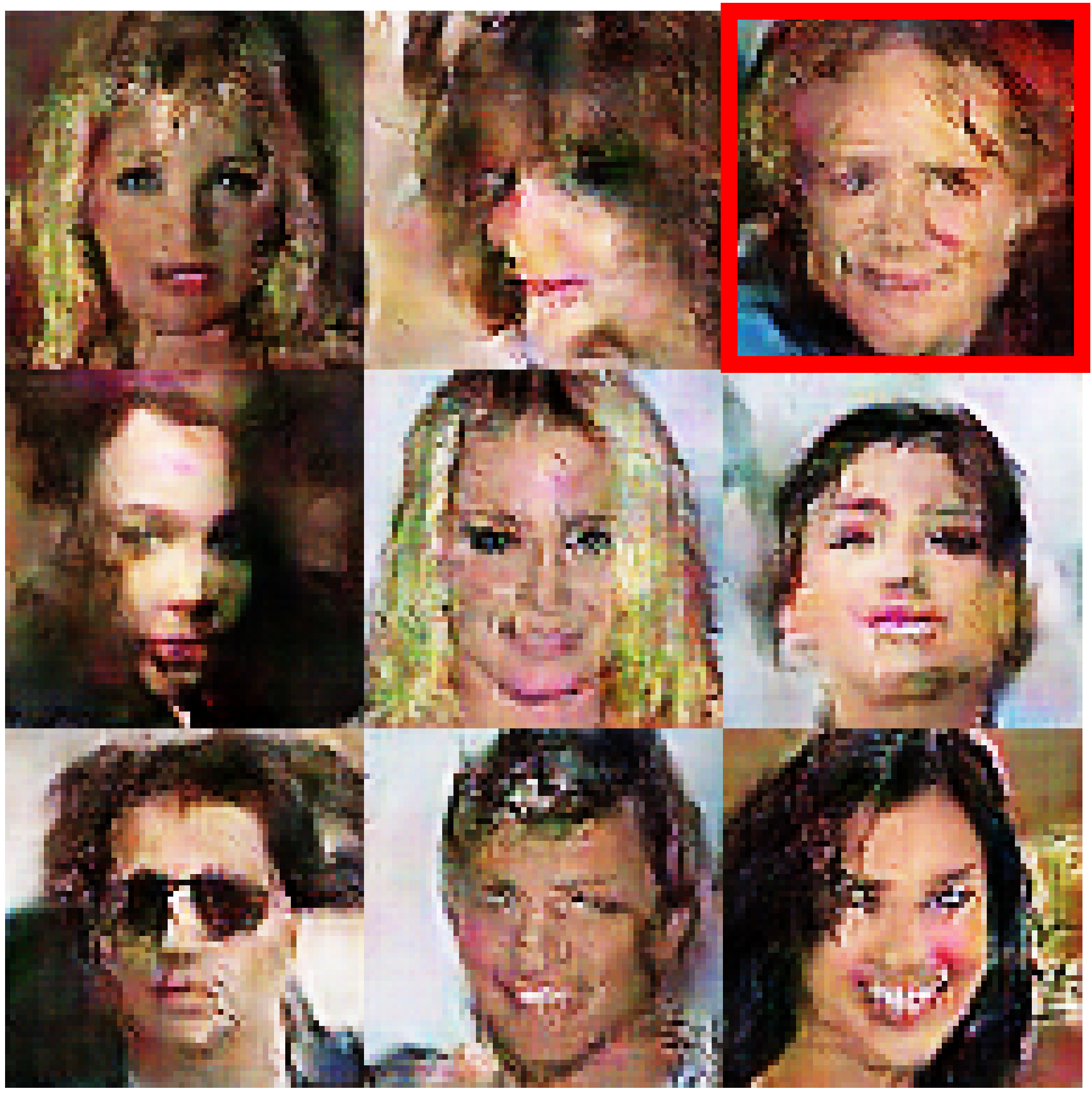}
    \caption{CWGAN ($1/9$)}
  \end{subfigure}
  \begin{subfigure}[t]{.15\linewidth}
  \centering
    \includegraphics[width=.96\linewidth]{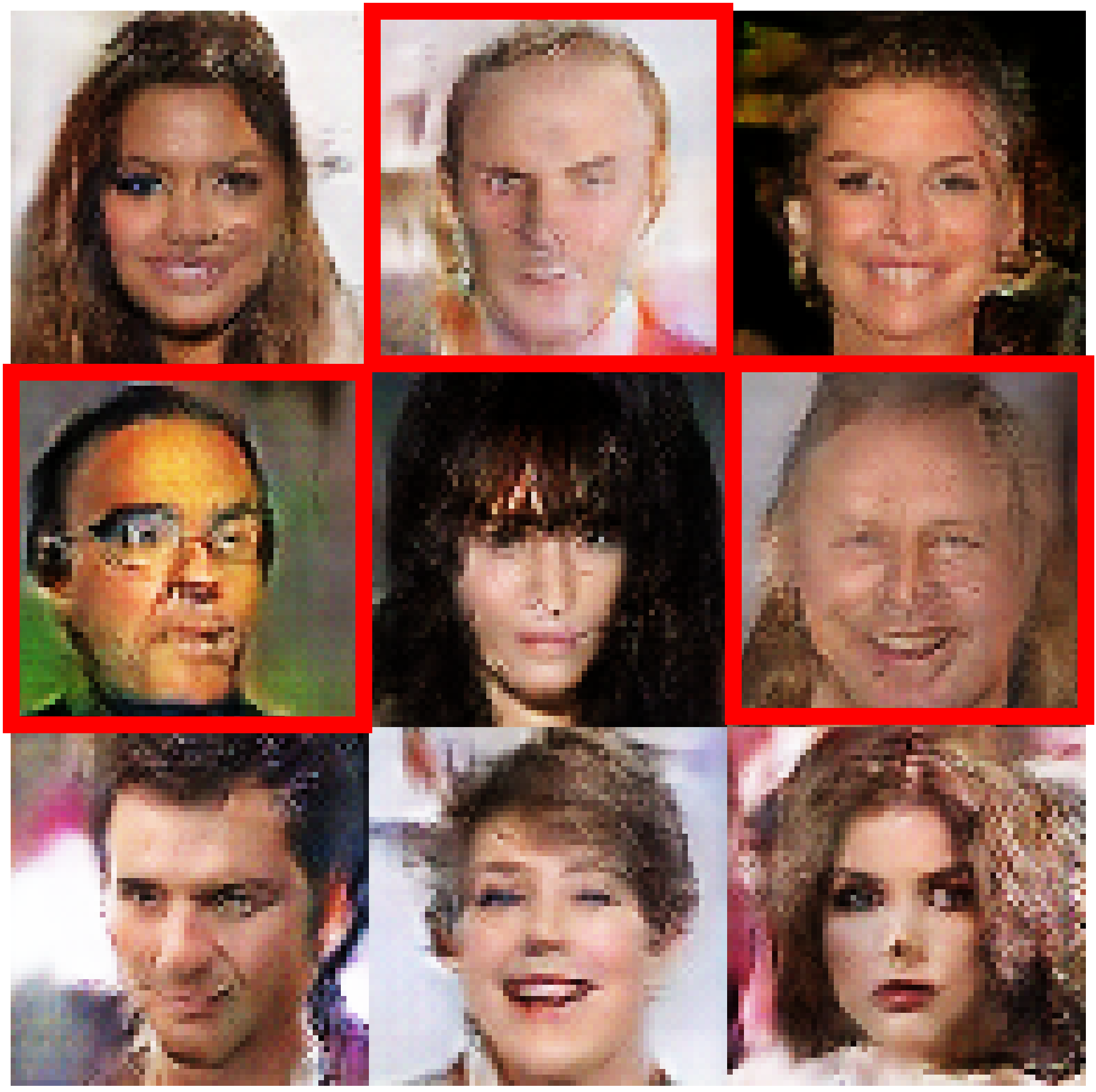}
    \caption{FBGAN ($3/9$)}
  \end{subfigure}
  \begin{subfigure}[t]{.16\linewidth}
  \centering
    \includegraphics[width=.92\linewidth]{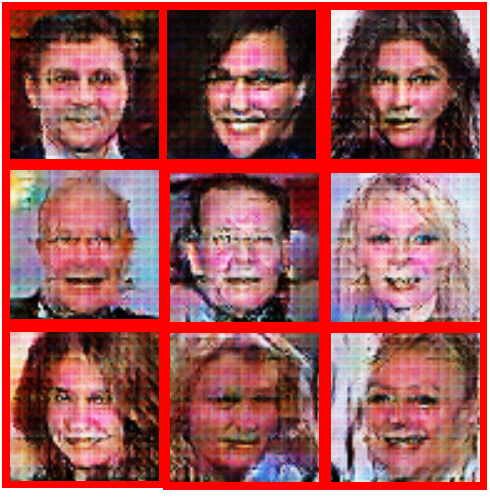}
    \caption{\revision{GAN-FT ($9/9$)}\label{fig:gan-ft_celebahq}}
  \end{subfigure}
  \begin{subfigure}[t]{.165\linewidth}
  \centering
    \includegraphics[width=.9\linewidth]{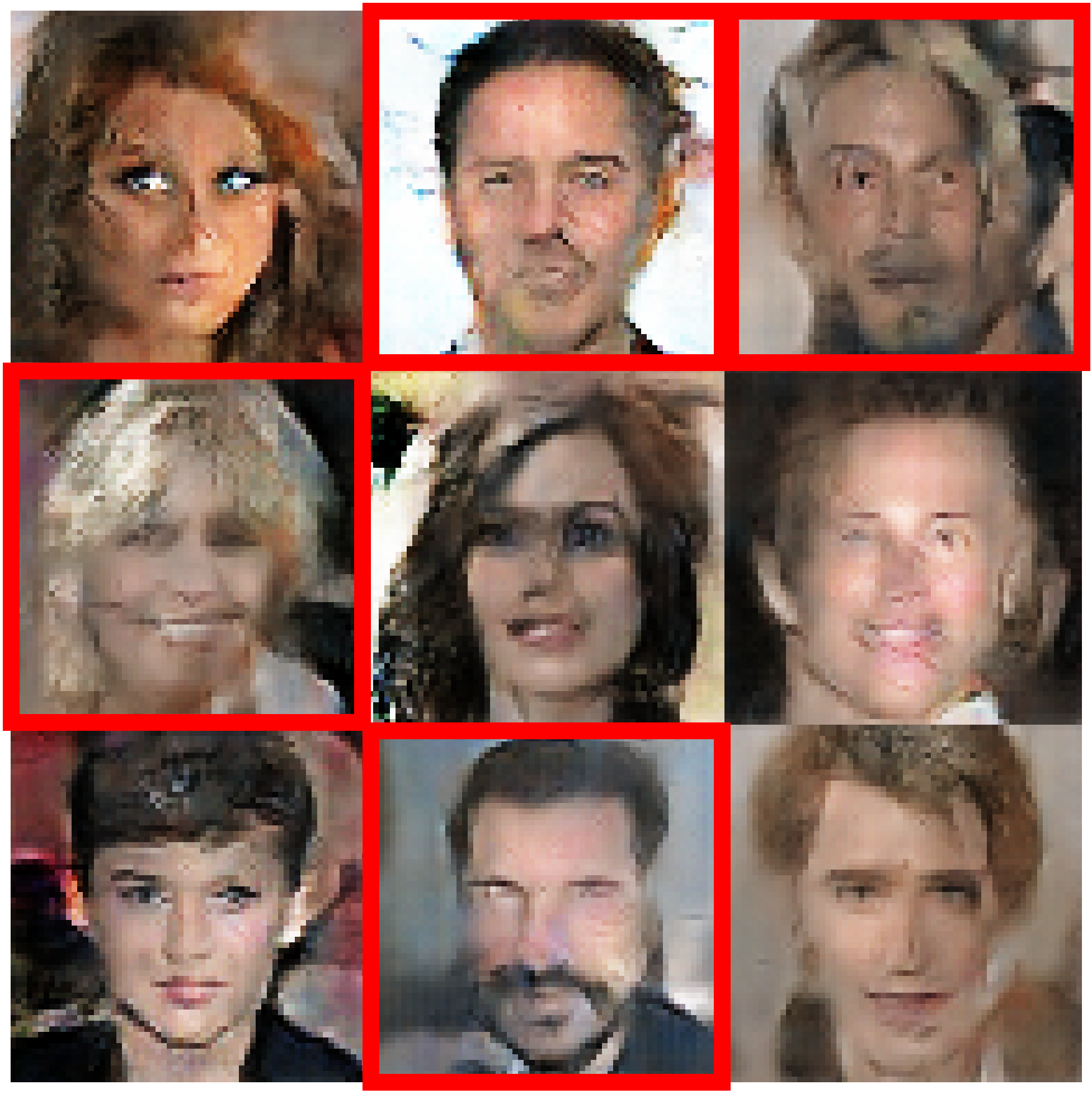}
    \caption{DiCGAN ($4/9$)}
  \end{subfigure}
  \begin{subfigure}[t]{.18\linewidth}
  \centering
    \includegraphics[width=.84\linewidth]{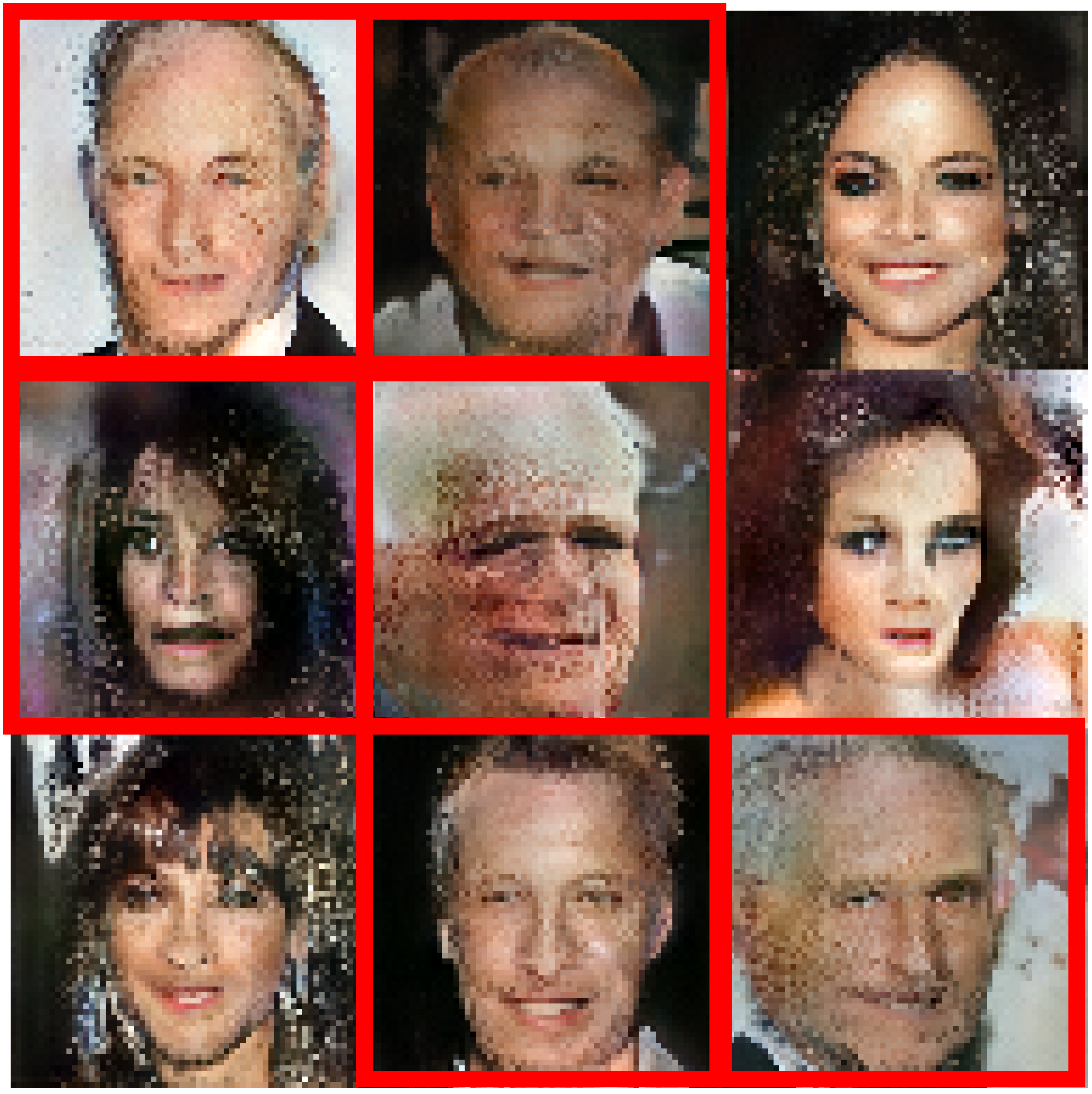}
    \caption{$\text{DiCGAN}_{\text{style}}$ ($6/9$)}
  \end{subfigure}
    \caption{Generated images on CelebA-HQ by {(a) WGAN, (b) CWGAN, (c) FBGAN, \revision{(d) GAN-FT}, (e) DiCGAN and (e) $\text{DiCGAN}_{\text{style}}$}. The red boxes refer to the images which are classified as old images.} 
    \label{fig:celebahq_baselines} \vskip-0.1in
\end{figure*}

\begin{figure*}[!t]
    \centering
  \begin{subfigure}[t]{.31\linewidth}
  \centering
    \includegraphics[width=.85\linewidth]{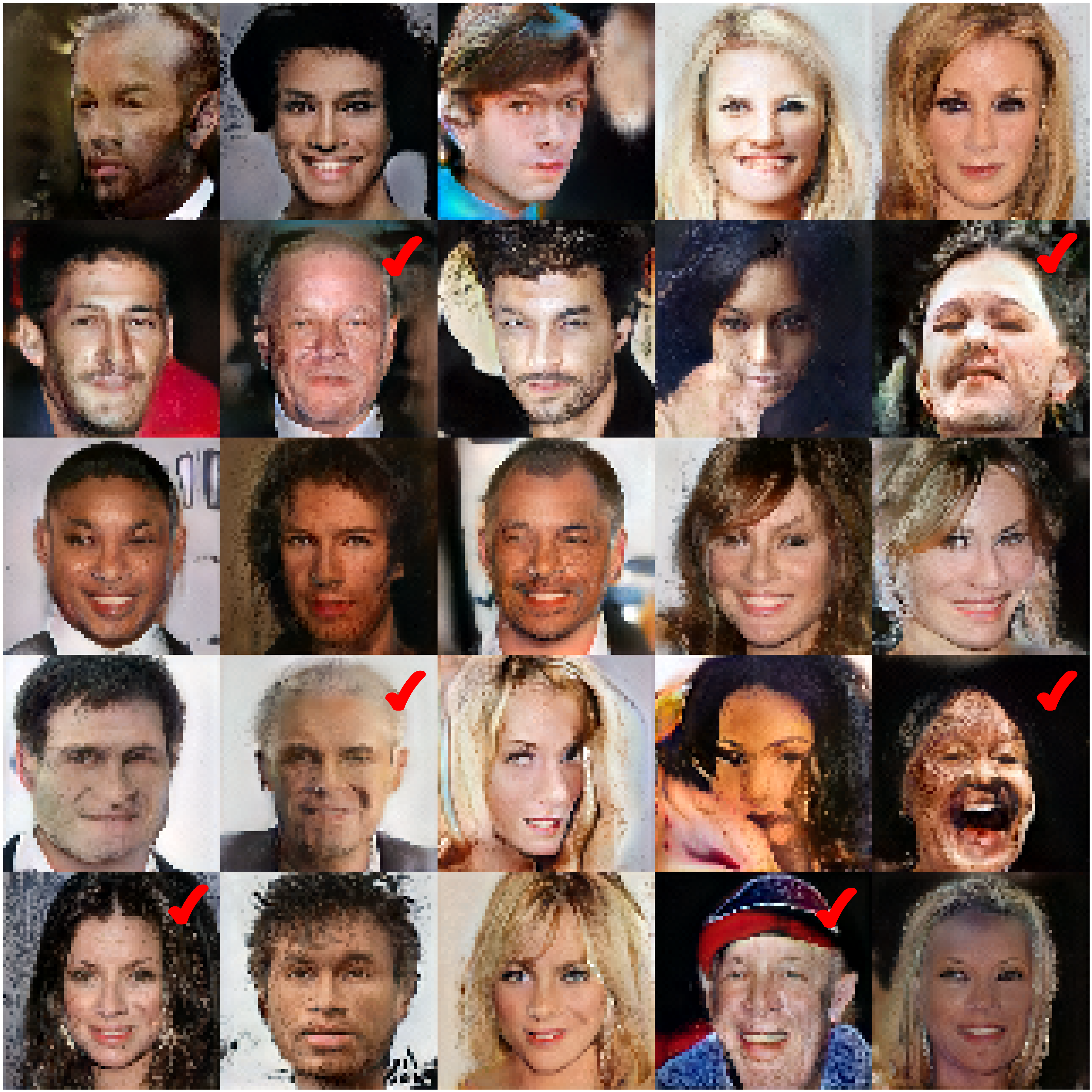}
    \caption*{Iter $4K$ ($6/25, 27.7\%$)}
  \end{subfigure}
  \begin{subfigure}[t]{.31\linewidth}
  \centering
    \includegraphics[width=.85\linewidth]{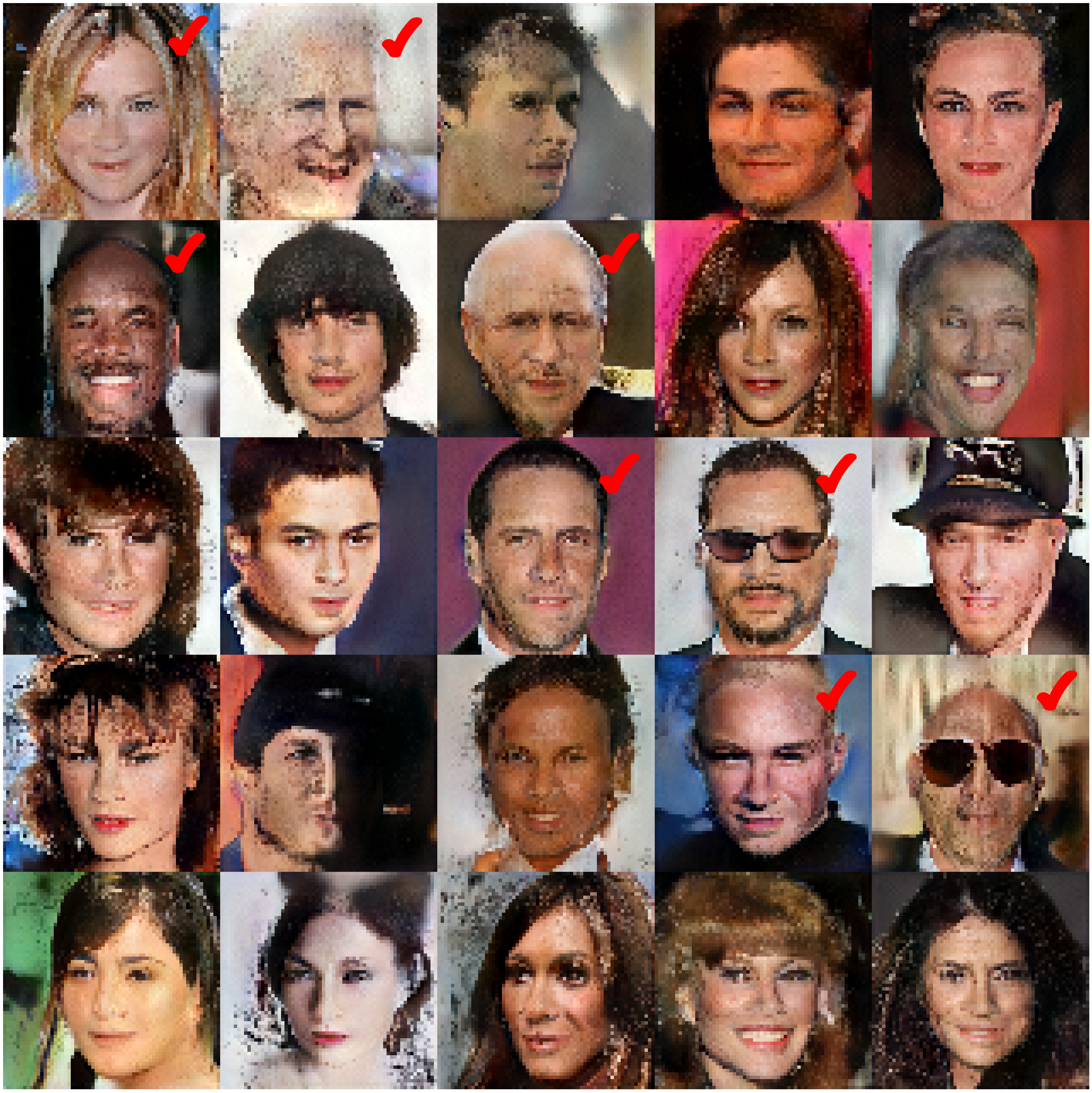}
    \caption*{Iter $6K$ ($8/25, 31.4\%$)}
  \end{subfigure}
  \begin{subfigure}[t]{.31\linewidth}
  \centering
    \includegraphics[width=.85\linewidth]{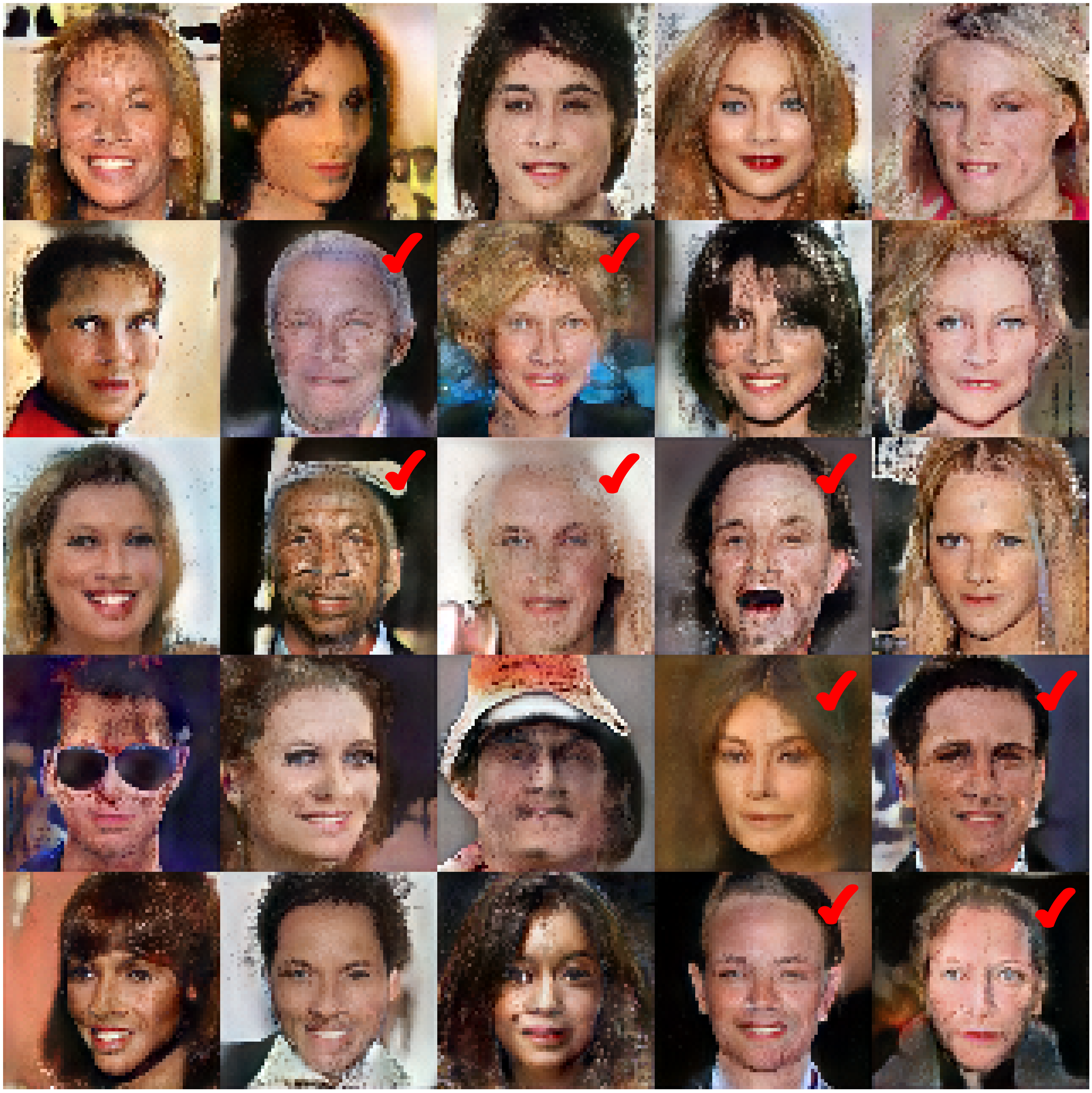}
    \caption*{Iter $8K$ ($9/25, 40.3\%$)}
  \end{subfigure}
  \begin{subfigure}[t]{.31\linewidth}
  \centering
    \includegraphics[width=.85\linewidth]{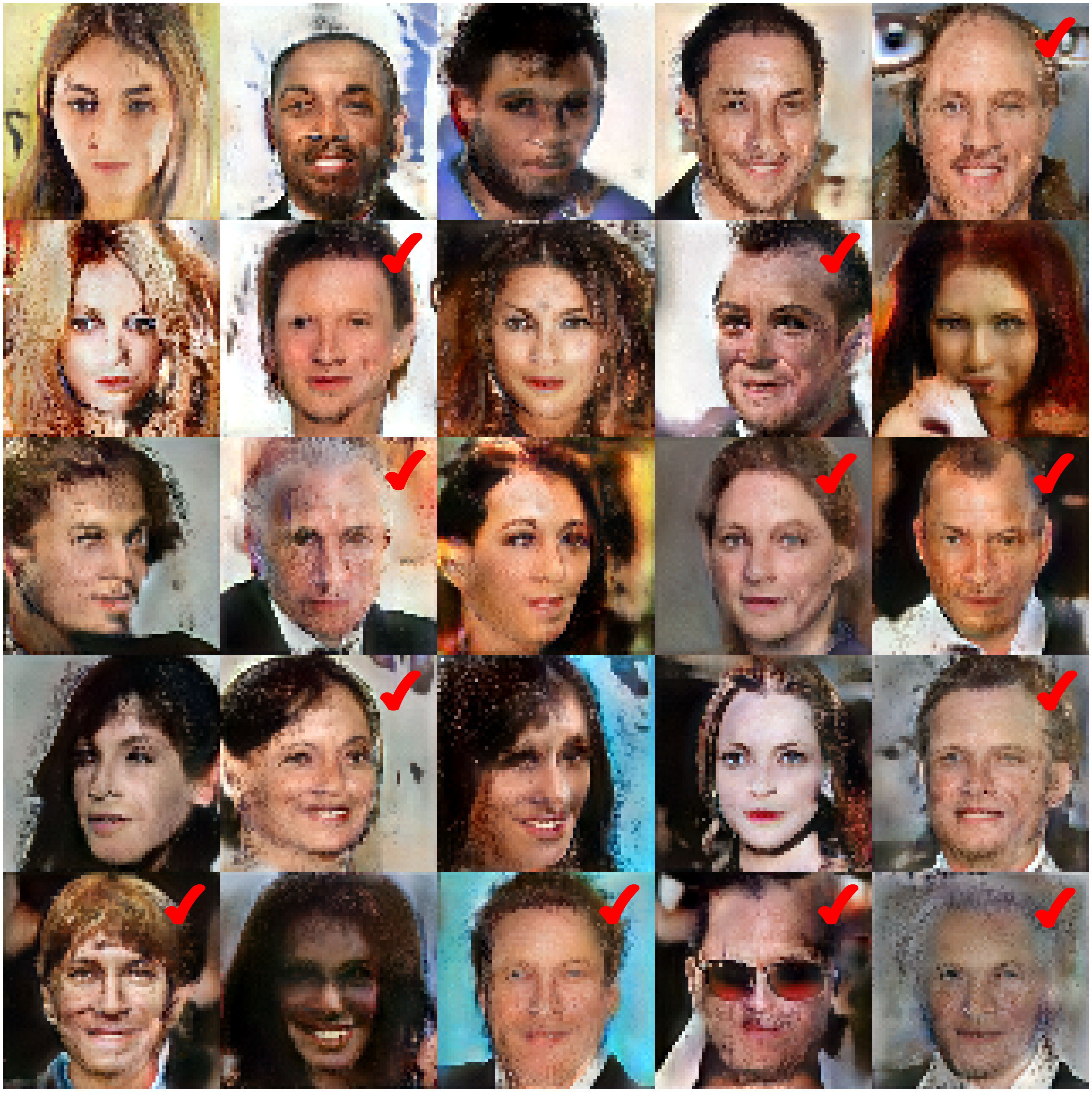}
    \caption*{Iter $10K$ ($12/25, 46.0\%$)}
  \end{subfigure}
  \begin{subfigure}[t]{.31\linewidth}
  \centering
    \includegraphics[width=.85\linewidth]{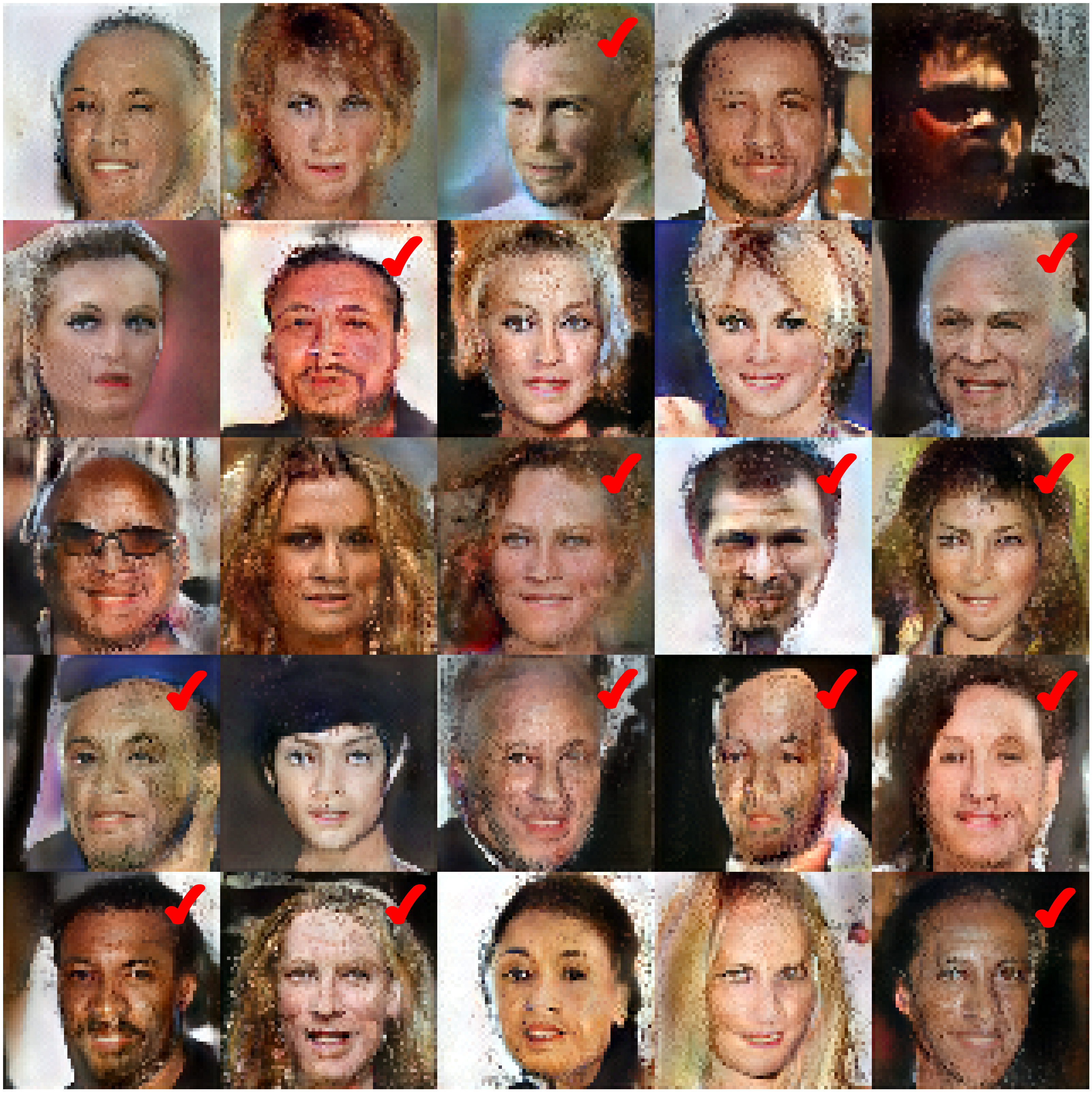}
    \caption*{Iter $12K$ ($13/25, 51.2\%$)}
  \end{subfigure}
  \begin{subfigure}[t]{.31\linewidth}
  \centering
    \includegraphics[width=.85\linewidth]{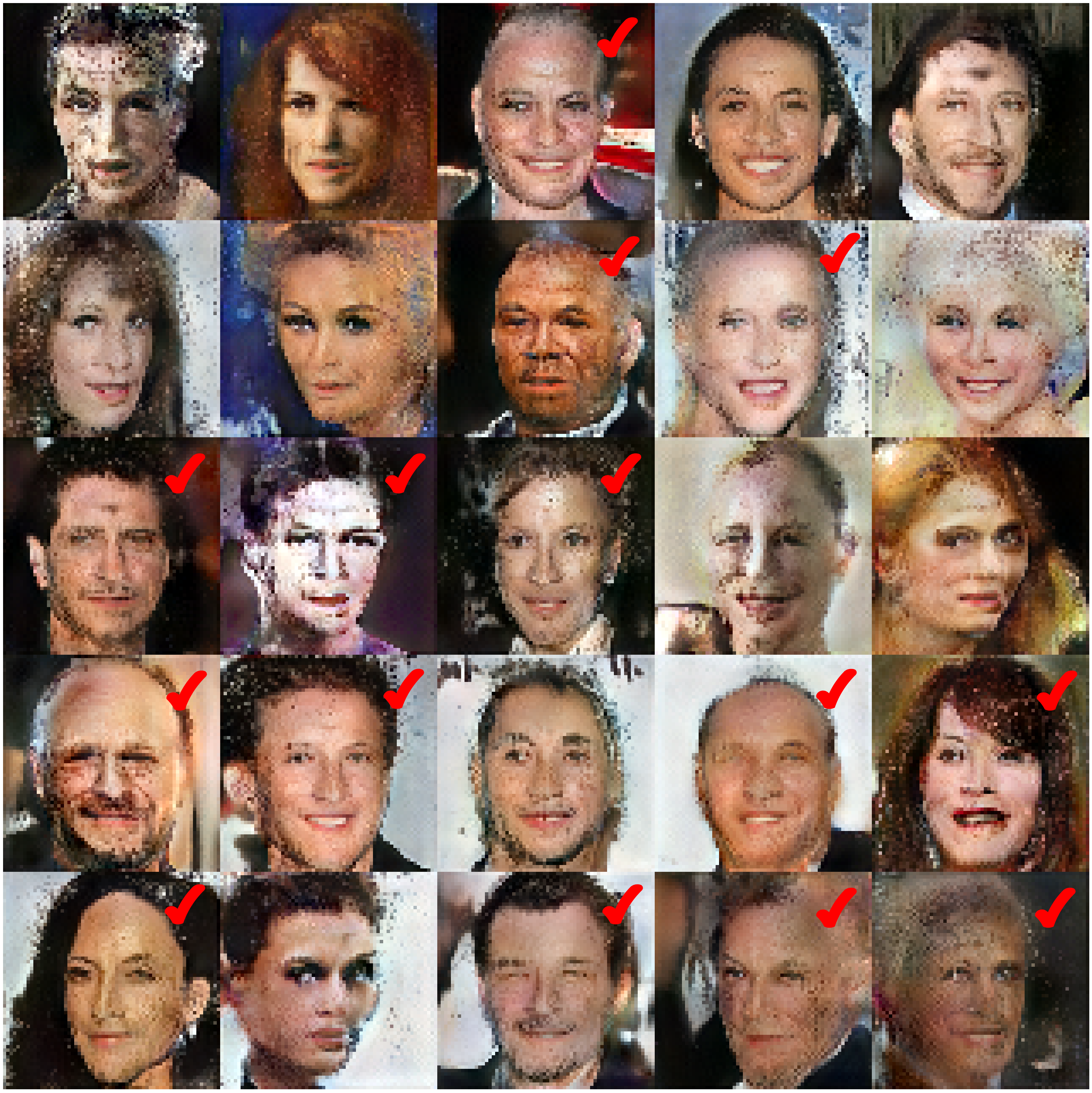}
    \caption*{Iter $14K$ ($14/25, 57.0\%$)}
  \end{subfigure}
    \caption{\label{fig:celebahq_samples}  Generated images of $\text{DiCGAN}_{\text{style}}$ on CelebA-HQ during the training process. $\text{DiCGAN}_{\text{style}}$ learns the distribution of old faces, which gradually generates more old face images. The red ticks refer to the images which are classified as old images. The \% denotes the percentage of old faces in $50K$ generated samples.}\vskip-0.1in
\end{figure*}
We explore the gap in the performance between DiCGAN and FBGAN evolves as the number of supervision increases on MNIST. Specifically,the query amount of resorting to the pre-trained classifier to obtain the prediction of the generated samples is restricted to $5K,50K,100K,150K,200K,500K$ for both FBGAN and DiCGAN. 

Fig.~\ref{fig:dw_ns} plots PDD versus the number of supervision for FBGAN and DiCGAN, respectively. It shows that (1) DiCGAN always learns the desired data distribution even given the limited supervision; (2) when given the limited supervision, FBGAN fails to learn the desired data distribution, i.e., achieving a small PDD; (3) FBGAN performs better and achieves a higher PDD,
narrowing the performance gap with DiCGAN as the number of supervision increases.

\subsection{Capturing Old Face Images on CelebA-HQ}
\label{sec:celebahq}
{Suppose the user is interested in learning the distribution of old face images on CelebA-HQ.} 

\textbf{Networks \& Hyperparameters} The balance factor $\lambda$ and the ranking margin $m$ is set to $1$. The batch size $b$ is set to $64$. {The network architecture of the critic and generator in our DiCGAN are based on WGAN-GP~\cite{gulrajani2017improved}. See Appendix for details. The baselines share the same architecture for a fair comparison.} The optimizer is Adam with a learning rate of $2e\textrm{-}4$ and $\beta_1=0.5, \beta_2=0.999$. $n_{critic}$ is set to 5. Further, we use an advanced GAN architecture, StyleGAN (\textit{https://github.com/NVlabs/stylegan2})~\cite{karras2020analyzing} to implement our DiCGAN, denoted as $\text{DiCGAN}_{\text{style}}$. The networks are optimized with Adam with $\beta_1=0, \beta_2=0.9$. The generator $G$'s learning rate is $1e\textrm{-}4$ while the critic $D$'s is $3e\textrm{-}4$~\cite{heusel2017gans}. \mbox{$n_{critic}$ is set to 1.}

\textbf{Training} There are $6,632$ old face images, labeled as desired, and $23,368$ young face images, labeled as undesired, in the training data. WGAN is only trained with the constructed desired dataset. CWGAN conditions on $c$ to model a conditional data distribution $p(x|c)$. A classifier, pre-trained for classifying young faces and old faces, is adopted for predicting the labels for the generated face images. Particularly, the query amount of resorting to the classifier is restricted to $30K$. As for FBGAN, at every training epoch, FBGAN generates $5K$ images, and those classified as old faces are selected by the selector to replace the old training data. \revision{As for GAN-FT, we first pre-trained WGAN-GP with all images including young faces and old faces. Then we fine-tuned the pre-trained generator with the classifier loss that makes the generated samples classified as old faces.} As for DiCGAN, the generated face image classified as an old face is preferred over the face image classified with the young attribute. At each iteration, $n_s$ is set to $64$. $n_{\mathrm{i}}$ is set to $1K$. $n_{\mathrm{g}}$ is set to $1K$. As for $\text{DiCGAN}_{\text{style}}$, at each iteration, $n_s$ is set to $64$. $n_{\mathrm{i}}$ is set to $500$. $n_{\mathrm{g}}$ is set to $30K$.

\begin{table}[!t]
\centering 
\caption{\label{tb:celebahq_percentage} Percentage of desired data in the generation (PDD) and image quality (FID) of various GANs on CelebA-HQ. The best results are highlighted in bold. The second best results are underlined. The strikethrough on PDD of WGAN and GAN-FT  denotes that they suffer from severe low-quality issues (large FID), generating very blur face images (Fig.~\ref{fig:celebahq_baselines}) and thus its PDD is not really meaningful.}
	\renewcommand{\arraystretch}{1.}
	\setlength{\tabcolsep}{0.5mm}{	
		\scalebox{1.}{
\begin{tabular}{ccccccccc}
\toprule[1.3pt]
Method & Original & WGAN  & CWGAN & FBGAN & \revision{GAN-FT} & DiCGAN & $\text{DiCGAN}_{\text{style}}$ \\ \hline
PDD    & 22.1      & \sout{\textbf{76.0}}  & 8.3   & 24.7  & \revision{\sout{\textbf{99.7}}} & \underline{33.4}  & \textbf{57.0}   \\
FID    & -        & 115.4 & 79.7  & 51.6 & \revision{107.1} & \underline{49.7}  & \textbf{36.5}   \\
\toprule[1.3pt]
\end{tabular}}}\vskip-0.1in
\end{table}

We visualize the generated face images randomly sampling from the generator of each model in Fig.~\ref{fig:celebahq_baselines}. For each model, we sample $50K$ samples from the generator and then calculate the percentage of old face images (PDD) and the image quality score, i.e., Frechet Inception Distance (FID) among the generated samples for quantitative evaluation in Table~\ref{tb:celebahq_percentage}. From Fig.~\ref{fig:celebahq_baselines} and Table~\ref{tb:celebahq_percentage},
(1) though WGAN mainly generates desired data, it has poor generation since its training data only consists of the desired subset, and thus is insufficient, which has only $6,632$ face images. The generated face images are blurred. Meanwhile, WGAN's FID score is the highest among all methods, i.e., $115.4$, quantitatively showing the poorest generation quality. (2) CWGAN has better generation quality than WGAN as it is trained with sufficient training data, $30K$ samples, but fails to shift towards the desired data distribution. There is only one old face image out of $9$ randomly sampled images in the visualization result. Its PDD ($8.3\%$) is smaller than the training data (Original, $22.1\%$). This is because the undesired data, i.e, the majority in the training data, dominates the generation of CWGAN. (3) FBGAN can achieve relatively good quality, with relatively small FID, but only slightly shift towards the distribution of desired data (PDD=$24.7\%$) due to limited supervision. \revision{(4) GAN-FT almost generates old faces, but the quality is poor, verified by a large FID quality score and low-quality visual results in Fig.~\ref{fig:gan-ft_celebahq}.} (5) DiCGAN achieves the best image quality among all the methods. Its FID score is the lowest. In addition, DiCGAN shifts more towards the desired data distribution than CWGAN and FBGAN, proven by a larger~PDD. 

On the other hand, WGAN-GP architecture is limited to approximating the complex distribution of CelebA-HQ data and thus cannot generate images with very high quality. Then, the introduction of generated samples into the training data will degrade the quality of generation. Therefore, DiCGAN implemented with WGAN-GP architecture is restricted with certain amounts of minor corrections and data replacement in order to obtain a good quality, achieving relatively low PDD. This problem can be improved by introducing a more advanced GAN architecture, StyleGAN. $\text{DiCGAN}_{\text{style}}$ can conduct more minor corrections and use more generated data to replace the training data, finally making the training data distribution shift very close to the desired data distribution. Thus, $\text{DiCGAN}_{\text{style}}$'s generation contains more desired samples than DiCGAN, i.e., larger PDD in Table~\ref{tb:celebahq_percentage}. Meanwhile, the generation has good quality with the best FID. We present generated images of $\text{DiCGAN}_{\text{style}}$ during the training process in Fig.~\ref{fig:celebahq_samples}. There gradually appears more desired face images, i.e., old face images in $\text{DiCGAN}_{\text{style}}$'s generation.

{Fig.~\ref{fig:gen_neighbors} shows the nearest neighbors of generated old images in the training dataset (given old face images), which demonstrates that our DiCGAN is not simply memorizing training images, but generates novel desired images. Thus, DiCGAN can perform data augmentation for desired samples.}

\begin{figure}[!t]
    \centering
    \includegraphics[width=.7\linewidth]{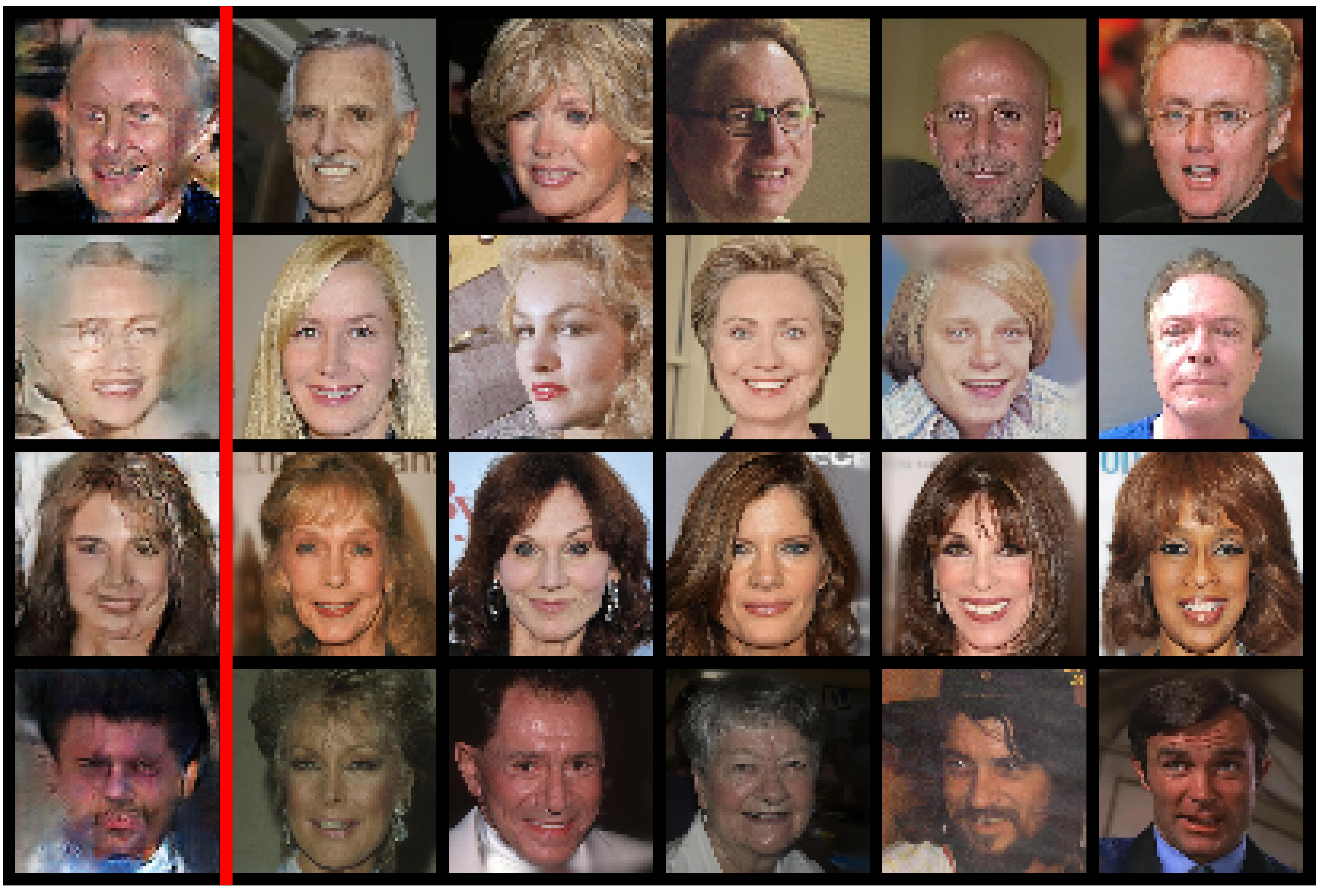}
    \caption{\label{fig:gen_neighbors}Nearest neighbors of generated desired images in the training dataset. The distance is measured by the $\ell_2$ distance between images. Images on the left of the red vertical line are samples generated by our DiCGAN. Images on the right are top~5 nearest neighbors in the training dataset.}
    \vskip-0.1in
\end{figure}

\subsection{{Simulating Synthetic Genes with Antimicrobial Properties}}
{Consider the biologist is interested in designing genes coding for antimicrobial peptides (AMPs), which are peptides with broad antimicrobial activity against bacteria, viruses, and fungi~\cite{izadpanah2005antimicrobial}. We can apply our DiCGAN to help optimize the gene coding for AMPs from an existing gene sequence dataset~\cite{gupta2019feedback}. Namely, our target is to learn the distribution of genes coding for AMPs on the gene sequence dataset.}

\textbf{Networks \& Hyperparameters} Both the balance factor $\lambda$ and the ranking margin $m$ are set to $1$. The batch size $b$ is set to $64$. All methods are implemented with the networks as FBGAN~\cite{gupta2019feedback}. The code of the network architecture can be found on FBGAN's official implementation (\textit{https://github.com/av1659/fbgan}). The networks are optimized with Adam with a learning rate of $1e\text{-}4$ and $\beta_1=0.5, \beta_2=0.9$. $n_{critic}$ is set to 10.

\textbf{Training} There is no labeling about whether the genes have the desired property in the training data. Therefore, WGAN, CWGAN, and GAN-FT cannot be applied.
Following FBGAN~\cite{gupta2019feedback}, we resort to an analyzer that can evaluate the property for genes. We pertain FBGAN and DiCGAN as vanilla WGAN using $3,655$ gene sequences. Then, we train FBGAN and DiCGAN with $2K$ gene sequences and collect the results for each method. Here we limit the amount of querying the analyzer to $6K$. $n_{\mathrm{i}}$ is set $31$. $n_{\mathrm{g}}$ is set to $500$.

\textbf{Sample selection in FBGAN} The selector in FBGAN selects desired samples based on the evaluation of the analyzer, which is able to predict the probability of a gene coding for AMPs. Specifically, the analyzer first estimates the probability of generated genes coding for AMPs. Then, the generated genes with the estimated probability over $0.8$, considered as the desired genes, are selected by the selector to replace the old training data.

\textbf{Pairwise preferences construction in DiCGAN} \revision{We consider the analyzer~\footnote{\revision{Or we can ask biological experts to compare pairs of samples in terms of the desired property if the desired properties cannot be expressed objectively.}} as the user, where a larger predicted value denotes that the gene is preferred for coding AMPs. Then, the pairwise comparison can be obtained for a pair of samples $x_1$ and $x_2$ according to their predicated values $p(x_1)$ and $p(x_2)$, i.e., $x_1 \succ x_2$ if $p(x_1)>p(x_2)$, and vice versa. At each iteration, $n_s$ is set to $64$. $n_{\mathrm{i}}$ is set to 31.} 

\begin{figure}[!tb]
    \centering
    \begin{subfigure}[t]{.49\linewidth}
  \centering
    \includegraphics[width=\linewidth]{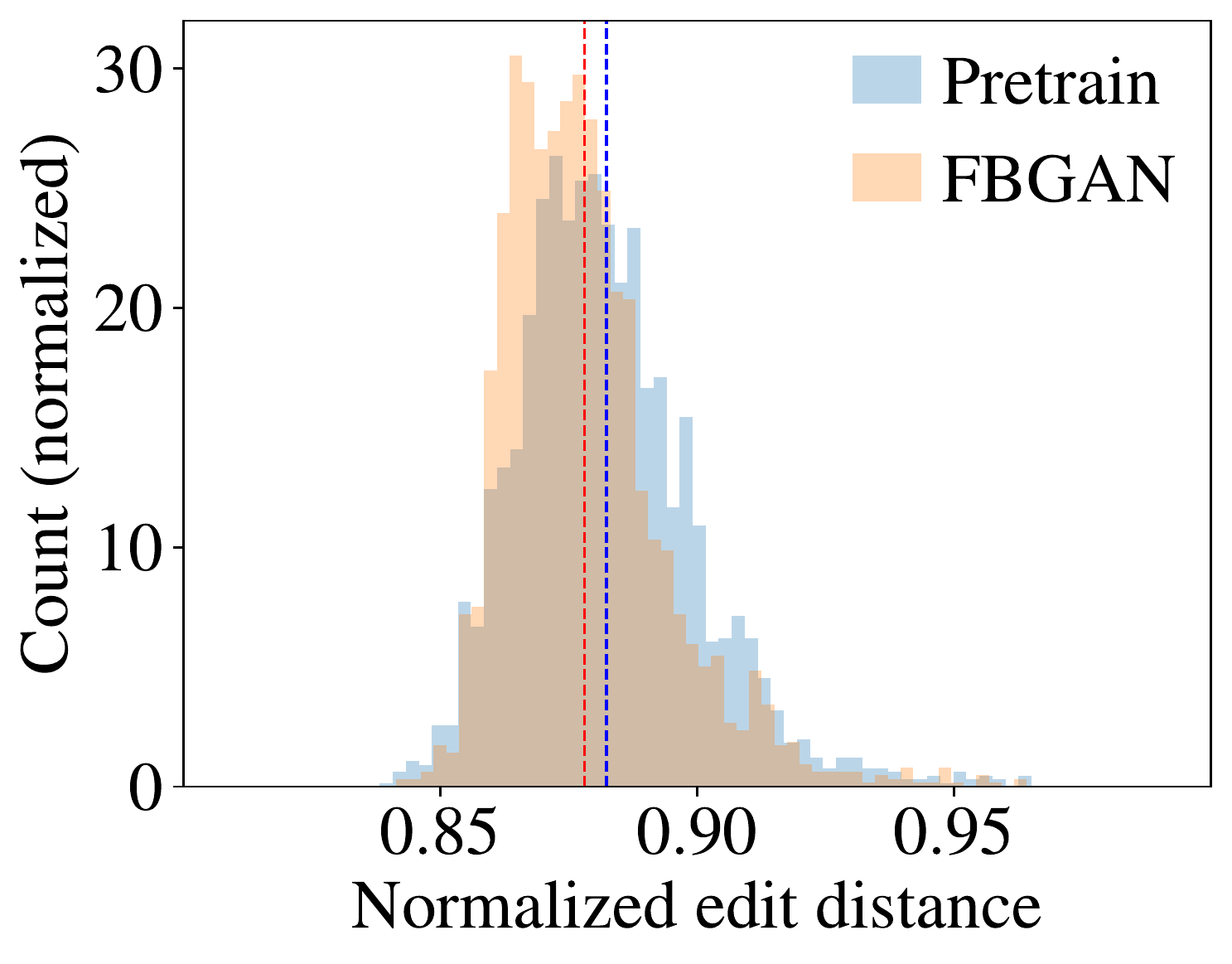}
    \caption{FBGAN}
  \end{subfigure}
    \begin{subfigure}[t]{.49\linewidth}
  \centering
    \includegraphics[width=\linewidth]{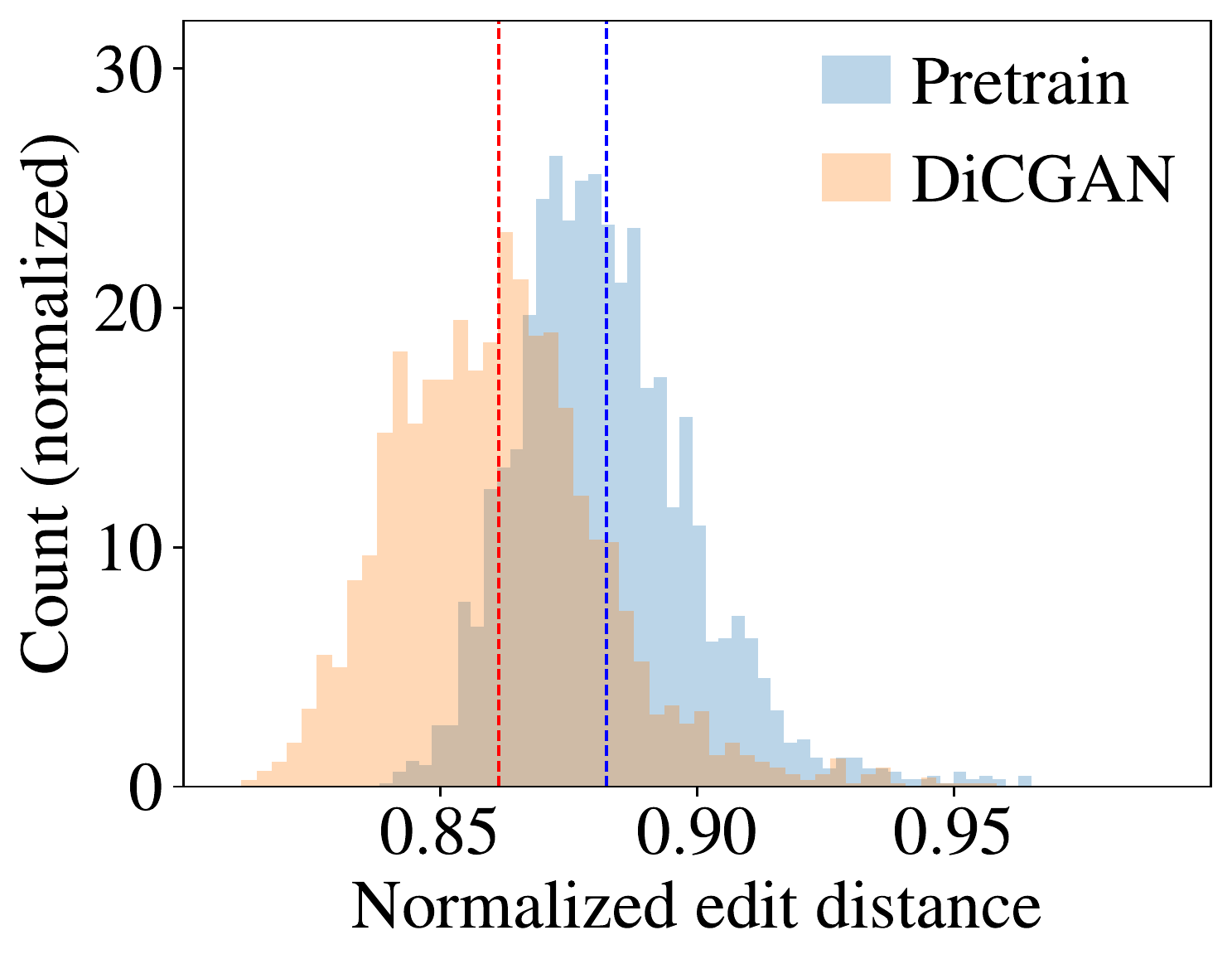}
    \caption{DiCGAN}
  \end{subfigure}
    \caption{\label{fig:gene} Comparison of {(a) FBGAN and (b) DiCGAN} on the gene sequence dataset. The dashed line denotes the mean value. The normalized edit distance is calculated between synthetic proteins and real desired proteins. A smaller distance denotes the generated genes are more similar to the desired genes.}\vskip-0.1in 
\end{figure}

The difference between genes is evaluated via the Normalized Edit Distance (normalized Levenstein distance, NED) between the proteins coded by the corresponding genes~\cite{gupta2019feedback}. The similarity of a generated gene to the desired gene can be evaluated using the averaged NED between its corresponding protein and all real desired proteins (i.e., AMPs). Particularly, a smaller NED w.r.t. AMPs denotes the generated genes more similar to genes coded for AMPs.

In Fig.~\ref{fig:gene}, we compare DiCGAN and WGAN w.r.t. the NED between AMPs and the synthetic proteins.  It shows that the synthetic proteins of DiCGAN shift toward a lower edit distance from AMPs, compared to the pretraining stage, i.e., WGAN. It means more genes coded for AMPs are generated by DiCGAN. However, FBGAN fails to shift its distribution towards the distribution of genes coding for AMPs.

\begin{table}[!hb]
\arrayrulecolor{black}
\centering
\caption{\label{tb:gene_metrics} Percentage of desired data in the generation (PDD) and gene quality (\%VG) of various GANs on the gene sequence dataset. The best results are highlighted in bold.}
	\renewcommand{\arraystretch}{1.}
	\setlength{\tabcolsep}{1.1mm}{	
		\scalebox{1.}{
\begin{tabular}{ccc}
\toprule[1.3pt]
Method & FBGAN & DiCGAN \\ \hline
PDD    & 29.0  & \textbf{98.8}   \\
\%VG   & 57.8  & \textbf{67.0}   \\
\toprule[1.3pt]
\end{tabular}}}\vskip-0.1in
\end{table}
Further,  we sample $50K$ genes from the generator for quantitative evaluation, which is collected in Table~\ref{tb:gene_metrics}. The generated genes with the probability of coding for AMPs over $0.8$ is considered as the desired genes. Then, the percentage of the desired genes among all $50K$ generated genes (PDD) is calculated. Particularly,  almost all genes generated by DiCGAN can be classified as the desired genes, i.e., $98.8\%$. In contrast, FBGAN generates $29.0\%$ desired genes. DiCGAN can learn the distribution of desired genes. However, FBGAN fails to derive the desired data distribution due to limited supervision.

On the other hand, we calculate the percentage of valid genes (\%VG)~\footnote{Correct gene structure is defined as a string starting with the canonical
start codon ``ATG'', followed by an integer number of codons of length 3, and ending with one of three canonical stop codons (``TAA'', ``TGA'', ``TAG'')~\cite{gupta2019feedback}.}.
The \%VG of DiCGAN is $67.0\%$ while that of FBGAN is $57.8\%$, which clarifies our methods achieve better quality than FBGAN. Their quality degradation compared to pre-trained WGAN is due to the introduction of generated genes as training data. 

\subsection{Study on critic values versus user preferences}
We apply DiCGAN to evaluate the critic values for all undesired data and all desired data on MNIST, CelebA-HQ, and the gene sequence dataset, respectively. Then we calculate the mean critic values for desired data and undesired data, respectively, with $95\%$ confidence interval. Meanwhile, we conduct the two-sample one-sided t-Test~\cite{welch1947generalization} for their mean critic values under the null hypothesis of equal means and the alternative hypothesis that the mean of the desired data is greater than that of the undesired data.

\begin{table}[!ht]
\centering
\caption{\label{tb:critic value} The mean (with $95\%$ confidence interval) and the two-sample one-sided t-Test results of  critic values for desired data and undesired data on MNIST, CelebA-HQ, and the gene sequence dataset.}
\renewcommand{\arraystretch}{1.}
\setlength{\tabcolsep}{1.1mm}{	
\scalebox{1.}{
\begin{tabular}{c|c|c|c}
\toprule[1.3pt]
\multirow{2}{*}{Dataset} & \multicolumn{2}{c|}{mean critic value} & \multirow{2}{*}{\begin{tabular}[c]{@{}c@{}}p value\\ (desired VS. undesired)\end{tabular}} \\ \cline{2-3}
          & \multicolumn{1}{c|}{desired} & undesired     &           \\ \midrule
MNIST     & $1.5\pm0.01$                 & $0.2\pm0.01$                   & $0.00$      \\
CelebA-HQ & $0.9 \pm 0.02$               & $0.4\pm 0.01$                  & $2.58e\text{-}285$ \\
gene      & $8.9\pm0.14$                 & $8.1\pm0.05$                   & $3.54e\text{-}24$  \\
\bottomrule[1.3pt]
\end{tabular}}}
\end{table}

The results in Table~\ref{tb:critic value} show that the average critic value for desired data is significantly larger than that of undesired data with a very small p-value. Therefore, it verifies the claim (mentioned in Sect.~\ref{sect:dicgan}) that the ranking loss can encourage high critic values to be assigned to the real desired data while low critic values are assigned to real undesired data.

\subsection{Ablation Study}
The objective in our DiCGAN (Eq.~\eqref{eq:objective}) consists of two components, i.e., the WGAN loss, which serves as the cornerstone of DiCGAN, and the ranking loss, which serves as the correction to WGAN. Meanwhile, we introduce the operation of replacement (Eq.~\eqref{eq:x}) during the model training. 

\begin{figure}[!t]
    \centering
 \begin{minipage}{0.24\textwidth}
  \centering
    \includegraphics[width=\linewidth]{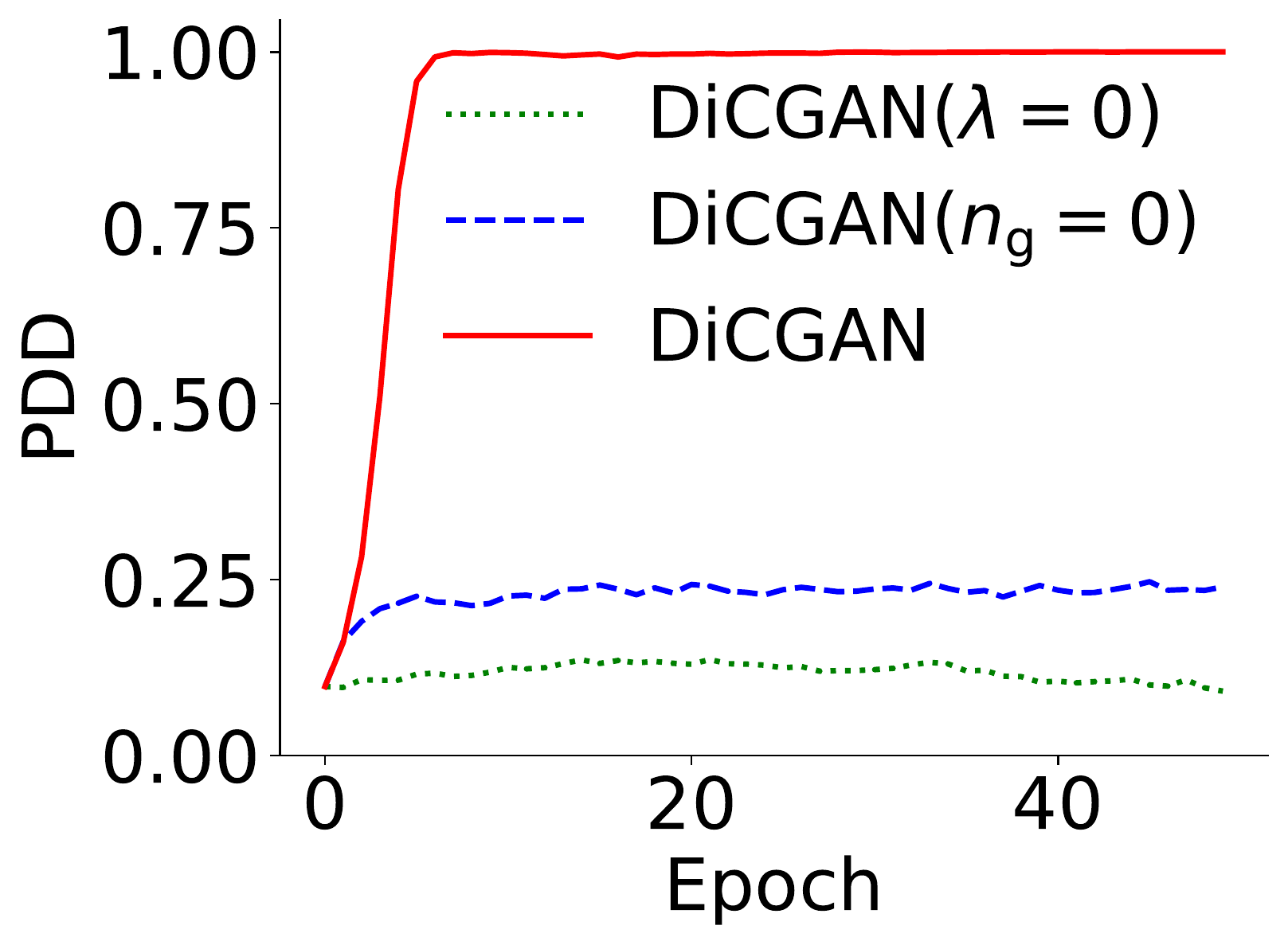}
    \centerline{(a) top~1}
    \label{fig:mnist_1}
  \end{minipage}
  \begin{minipage}{0.24\textwidth}
  \centering
    \includegraphics[width=\linewidth]{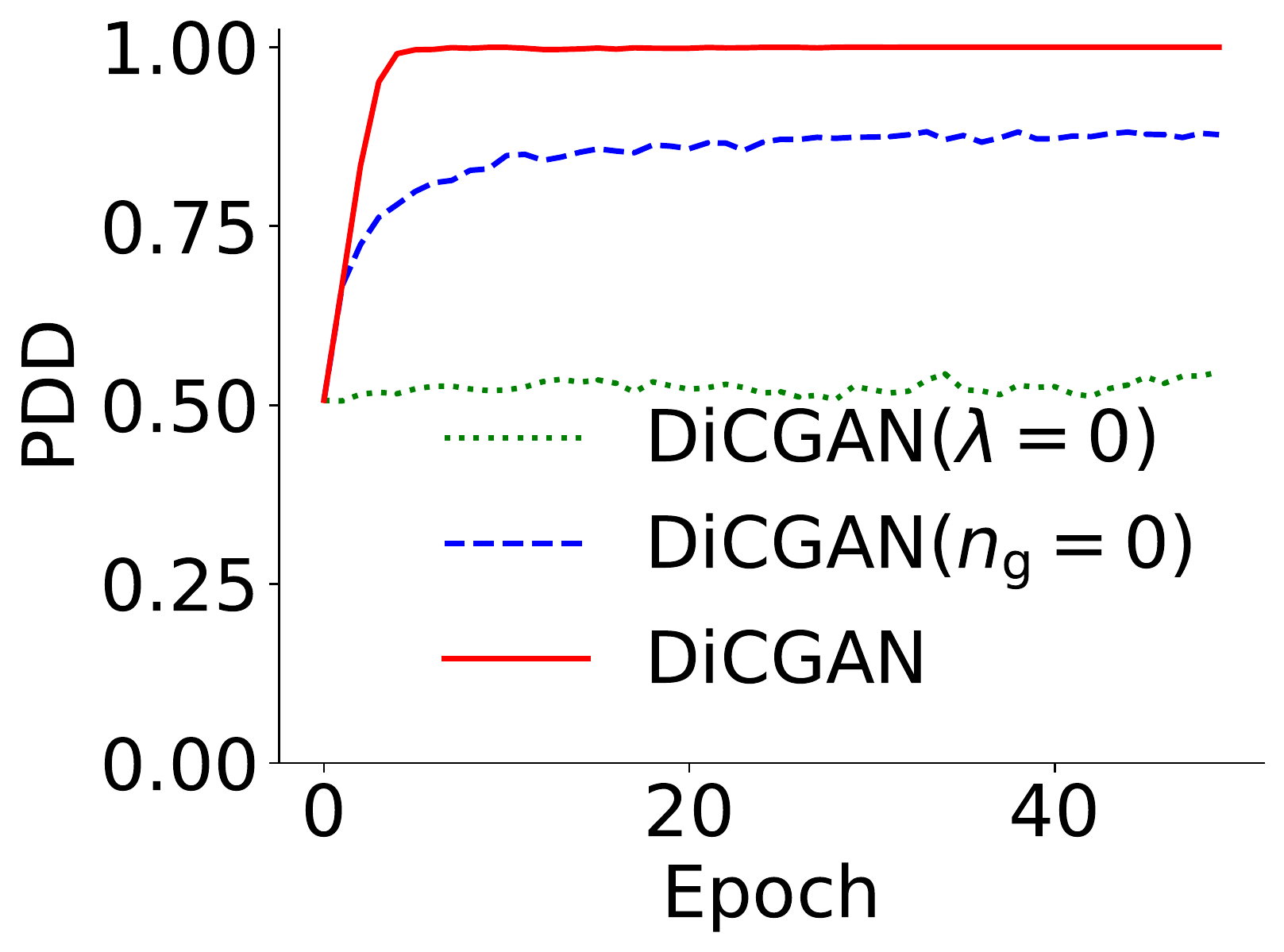}
    \centerline{(b) top~5}
    \label{fig:mnist_5}
  \end{minipage}
  \begin{minipage}{0.335\textwidth}
  \arrayrulecolor{black}
  \centering
  \renewcommand{\arraystretch}{1.1}
	\setlength{\tabcolsep}{1mm}{	
		\scalebox{0.85}{
\begin{tabular}{ccc}
\toprule[1.3pt]
Method                  &  top~1         & top~5        \\ \hline
Original                & 9.9   & 51.1      \\ \hline
DiCGAN ($\lambda=0$)               & 9.1   & \multicolumn{1}{c}{54.6}  \\ \hline
DiCGAN ($n_\mathrm{g}=0$)               & 24.0  & \multicolumn{1}{c}{87.8}   \\ \hline
DiCGAN                  & $\mathbf{100.0}$       & $\mathbf{100.0}$       \\
\toprule[1.3pt]
\end{tabular}}}
  \centerline{(c) PDD} 
  \end{minipage}
  \caption{\label{fig:ablation_PDD}Ablation study on MNIST. (a-b) PDD vs. epoch in the generation of DiCGAN ($\lambda=0$), DiCGAN ($n_\mathrm{g}=0$) and DiCGAN. (c) PDD in the data from the original dataset, DiCGAN ($\lambda=0$), DiCGAN ($n_\mathrm{g}=0$) and DiCGAN.}\vskip-0.1in
\end{figure}

To analyze the effects of the correction for WGAN (the third term in Eq.~\eqref{eq:objective}) and the replacement operation, we plot the percentage of desired data in the generation (PDD) versus the training epoch for DiCGAN ($\lambda=0$), DiCGAN ($n_\mathrm{g}=0$) and DiCGAN in Fig.~\ref{fig:ablation_PDD}a, ~\ref{fig:ablation_PDD}b. Meanwhile, the converged percentage of desired samples (PDD) is reported in Fig.~\ref{fig:ablation_PDD}c.
\begin{itemize}[leftmargin=.2in]
    \item [1)] \textbf{Without the correction term ($\lambda=0$), DiCGAN cannot learn the desired data distribution.} The PDD of DiCGAN ($\lambda=0$) remains constant during training on MNIST (Fig.~\ref{fig:ablation_PDD}a, ~\ref{fig:ablation_PDD}b) compared with that of the original dataset (Fig.~\ref{fig:ablation_PDD}c). This is because the WGAN term in DiCGAN ($\lambda=0$) focuses on learning the training data distribution. 
    \item [2)] \textbf{Without the replacement ($n_\mathrm{g}=0$), DiCGAN makes a minor correction to the generated distribution.} In Fig.~\ref{fig:ablation_PDD}a, ~\ref{fig:ablation_PDD}b, the PDD of DiCGAN ($n_\mathrm{g}=0$) slightly increases compared with the original dataset. This is consistent with our analysis that the correction term would drive the generation towards the desired data distribution. 
    \item [3)]\textbf{DiCGAN learns the desired data distribution with a sequential minor correction.} The PDD of DiCGAN grows with training and reaches almost $100\%$ when convergence.  The correction term drives DiCGAN's generation towards the desired data slightly at each epoch. With the iterative replacement, the minor correction sequentially accumulates, and finally the generated distribution shifts to the desired data distribution.
\end{itemize}

\subsection{Pairwise regularization on the Generator}
As discussed in Sect.~\ref{sec:prg}, the pairwise regularization is possibly added to the generator. We consider two cases of adding the regularization to the generator. 
First, we only add the pairwise regularization to the generator (PRG-1). Second, we add the regularization to the generator together with the regularization on the critic (PRG-2). 

\begin{figure}[!tb]
    \centering 
 \begin{minipage}[b]{0.24\textwidth}
  \centering
    \includegraphics[width=.8\linewidth]{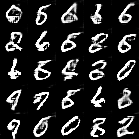}
    \label{fig:mnist_prg-1}
    \centerline{(a) PRG-1 (PDD=13.9)}
  \end{minipage}
  \begin{minipage}[b]{0.24\textwidth}
  \centering
    \includegraphics[width=.8\linewidth]{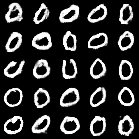}
    \label{fig:mnist_prg-2}
   \centerline{(b) PRG-2 (PDD=99.4)}
  \end{minipage}
  \caption{\label{fig:mnist_prg} \mbox{Generated digits and PDD of {(a) PRG-1 \& (b) PRG-2}.}} \vskip-0.1in
\end{figure}

We conducted experiments on MNIST to show the effectiveness of these two methods. $\lambda$ and $\lambda_g$ are both set to $1$. As shown in Fig.~\ref{fig:mnist_prg}, PRG-1 failed to learn the desired data distribution. PRG-2 can learn the desired data distribution. The quantitative results are consistent with the visual results, with $13.9\%$ and $99.4\%$ PDD, respectively.

\section{Conclusions and Discussions}
\label{sect:conclusion}
\revision{This paper proposes DiCGAN to learn the distribution of the user-desired data from the entire dataset using the pairwise preferences. This is the first work to promote the ratio of the desired data by incorporating user preferences directly into the data generation.
We empirically demonstrate the efficacy of DiCGAN in two real-world applications – generating images that meet the user’s interest for a given dataset and optimizing biological products with desired properties. Especially, our DiCGAN outperforms baselines in the cases of insufficient desired data and limited supervision.} 

\revision{Though it is superior to existing methods in terms of desired data generation when there is insufficient desired data, our DiCGAN cannot handle the case when there are extremely limited desired data, e.g., few-shot even one shot. 
Furthermore, as shown in our experimental study, high-resolution high-quality desired image generations require an advanced GAN architecture, which incurs heavy computational costs. There is an ongoing research direction of GAN that aims to generate high-resolution high-quality data with light architecture designs, which can mitigate such a limitation.}

\section*{Acknowledgment}
YP and IWT are supported by A*STAR CFAR. IWT is also supported by Australian Research Council under grants DP200101328. YY and XY are supported by the Program for Guangdong Introducing Innovative and Enterpreneurial Teams (Grant No. 2017ZT07X386), Shenzhen Science and Technology Program (Grant No. KQTD2016112514355531) and the Program for Guangdong Provincial Key Laboratory (Grant No. 2020B121201001).

\appendix
\section*{Hyperparameter setting and network structures.}
For the setup of the balance factor $\lambda$, we set $\lambda$ as $0.1, 0.5, 1, 5, 10$, respectively and found that $\lambda=1$ consistently performs well on all datasets. Thus we set $\lambda=1$ for all datasets.

\begin{table}[H]
\centering 
\caption{\label{tb:mnist D} The architecture of our critic for MNIST.} \vskip-0.05in
	\renewcommand{\arraystretch}{1.}
	\setlength{\tabcolsep}{1.1mm}{	
		\scalebox{1.}{
\begin{tabular}{c}
\toprule[1.3pt]
$x \in \mathbb{R}^{1\times28\times28}$  \\ \hline
Conv2d $5\times5$, stride 2, pad 2, $1 \rightarrow 64$; ReLU  \\ \hline
Conv2d $5\times5$, stride 2, pad 2, $64 \rightarrow 128$; ReLU  \\ \hline
Conv2d $5\times5$, stride 2, pad 2, $128 \rightarrow 256$; ReLU  \\ \hline
linear, $256\times4\times4 \rightarrow 1$   \\ 
\toprule[1.3pt]
\end{tabular}}}\vskip-0.15in
\end{table}

\begin{table}[ht]
\centering
\caption{\label{tb:mnist G} The architecture of our generator for MNIST.} \vskip-0.05in
	\renewcommand{\arraystretch}{1.}
	\setlength{\tabcolsep}{1.1mm}{	
		\scalebox{1.}{
\begin{tabular}{c}
\toprule[1.3pt]
$z \in \mathbb{R}^{128} \sim \mathcal{N}(0, I)$  \\ \hline
linear, $128\rightarrow 256\times4\times4$   \\ \hline
ConvTranspose2d $5\times5$, stride 1, pad 0, $256 \rightarrow 128$; ReLU  \\ \hline
ConvTranspose2d $5\times5$, stride 1, pad 0, $128 \rightarrow 64$; ReLU  \\ \hline
ConvTranspose2d $8\times8$, stride 2, pad 0, $64 \rightarrow 1$; Sigmoid  \\
\toprule[1.3pt]
\end{tabular}}}\vskip-0.15in
\end{table}

\begin{table}[htb]
\centering
\caption{\label{tb:celebahq D} The architecture of our critic for CelebA-HQ.} \vskip-0.05in
	\renewcommand{\arraystretch}{1.}
	\setlength{\tabcolsep}{1.1mm}{	
		\scalebox{1.}{
\begin{tabular}{c}
\toprule[1.3pt]
$x \in \mathbb{R}^{3\times64\times64}$  \\ \hline
Conv2d $5\times5$, stride 2, pad 2, $1 \rightarrow 64$; LeakyReLU  \\ \hline
Conv2d $5\times5$, stride 2, pad 2, $64 \rightarrow 128$; InstanceNorm2d; LeakyReLU  \\ \hline
Conv2d $5\times5$, stride 2, pad 2, $128 \rightarrow 256$; InstanceNorm2d; LeakyReLU  \\ \hline
Linear, $256\times8\times8 \rightarrow 1$   \\ 
\toprule[1.3pt]
\end{tabular}}}\vskip-0.15in
\end{table}

\begin{table}[H]
\centering
\caption{\label{tb:celebahq G} The architecture of our generator for CelebA-HQ.} \vskip-0.05in
	\renewcommand{\arraystretch}{1.}
	\setlength{\tabcolsep}{1.1mm}{	
		\scalebox{.85}{
\begin{tabular}{c}
\toprule[1.3pt]
$z \in \mathbb{R}^{100} \sim \mathcal{N}(0, I)$  \\ \hline
Linear, $100\rightarrow 512\times4\times4$   \\ \hline
ConvTranspose2d $5\times5$, stride 2, pad 2, output pad 1, $512 \rightarrow 256$; BatchNorm2d; ReLU  \\ \hline
ConvTranspose2d $5\times5$, stride 2, pad 2, output pad 1, $256 \rightarrow 128$; BatchNorm2d; ReLU  \\ \hline
ConvTranspose2d $5\times5$, stride 2, pad 2, output pad 1, $128 \rightarrow 64$; BatchNorm2d; ReLU  \\ \hline
ConvTranspose2d $8\times8$, stride 2, pad 2, output pad 1, $64 \rightarrow 3$; Tanh  \\
\toprule[1.3pt]
\end{tabular}}}\vskip-0.1in
\end{table}


\ifCLASSOPTIONcaptionsoff
  \newpage
\fi


\bibliographystyle{IEEEtran}
\bibliography{mybibfile}

\end{document}